\documentclass{fundamx}
\usepackage{url} 
\usepackage[ruled,lined]{algorithm2e}
\usepackage{graphicx}

\usepackage{mymacros}

\begin{document}
\setcounter{page}{1001}
\issue{XXX~(2019)}

\title{Symbolic Tensor Neural Networks for Digital Media
\\{\small  -- from Tensor Processing via BNF Graph Rules to CREAMS Applications\thanks{This is a tutorial  submitted to the {\it "Special Issue on Deep Neural Networks for Digital Media Algorithms" of Fundamenta Informaticae}.}}}

\address{w.skarbek{@}ire.pw.edu.pl}

\author{W\l adys\l aw Skarbek\\
Institute of Radioelectronics and Multimedia Technology \\
Faculty of Electronics and Information Technology \\ 
Warsaw University of Technology \\ 
Nowowiejska 15/19, 00-665 Warszawa, Poland
} 
\maketitle

\runninghead{W. Skarbek}{Symbolic Tensor Neural Networks for Digital Media}

\vspace*{-20mm}
\begin{abstract} This tutorial material on Convolutional Neural Networks (CNN) and its applications in digital media research is based on  the concept of Symbolic Tensor Neural Networks. The set of STNN expressions is specified in Backus-Naur Form (BNF) which is annotated by constraints typical for labeled acyclic directed graphs (DAG). The BNF induction begins from a collection of neural unit symbols with extra (up to five) decoration fields (including tensor depth and sharing fields). The inductive rules provide not only the general graph structure but also the specific  shortcuts for residual blocks of units. A syntactic mechanism for network fragments modularization is introduced via user defined units and their instances. Moreover, the dual BNF rules are specified in order to generate the Dual Symbolic Tensor Neural Network (DSTNN). The joined interpretation of STNN and DSTNN provides the correct flow of gradient tensors, back propagated at the training stage. The proposed symbolic representation of CNNs is illustrated for six generic digital media applications (CREAMS): Compression, Recognition, Embedding, Annotation, 3D Modeling for human-computer interfacing, and data Security based on digital media objects. In order to make the CNN description and its gradient flow complete, for all presented applications, the symbolic representations of mathematically defined loss/gain functions  and gradient flow equations for all used core units, are given. The tutorial is to convince the reader that STNN is not only a convenient symbolic notation for public presentations of CNN based solutions for CREAMS problems but also  that it is a design blueprint with a potential for automatic generation of application source code.  
\end{abstract}


\tableofcontents

\section{Introduction}

Artificial Neural Networks (ANN) are present in science and engineering since late 1950s (Rosenblatt \cite{Rosenblatt57}. Initially, as a computational tool for computer-aided decision making (perceptron), then as a function approximation universal mechanism (multilayer perceptron) -- since 1980s when the error backpropagation was published (cf. Werbos' pioneer paper \cite{WerbosP82a} and Rumelhart et.al. 
\cite{RumelhartD88a}) using a low dimensional data 
for their classification and regression, and nowadays, since  about 2010, as Deep Neural Networks (DNN) equipped with  specialized operations (e.g. convolutions), operating on large multidimensional signals (aka\footnote{The acronyms used: aka -- "also known as",
wrt -- "with respect to", iff -- "if and only if".} tensors) and dozens of processing layers to extract their implicit tensor features. The comprehensive survey of ANN history with a large collection of references can be found in Schmidhuber's article \cite{Schmidhuber14a}.

Deep neural networks algorithms embrace a broad class of  ANN algorithms for data model building in order to solve the hard data classification and regression problems. However, now due to unprecedented progress in computing technology, both tasks can accept digital media signals approaching the complexity, the live organisms can deal with. 
Moreover, due to the high quality of DNN models, nowadays digital media systems and application significantly improved their performance, comparing to the research status at the turn of the century when the existing media standards (JPEG-x, MPEG-x) were being established. 

The paper is focused on Convolution Neural Network (CNN) having the structure of  directed acyclic graphs (DAG) with edges transferring and  nodes processing tensors -- the multidimensional views of (usually large) data buffers.
Initially, the author intended to make the presentation of the existing CNN solutions for representative application areas of digital media. Since the presentation for Fundamenta Informaticae should be more mathematically rigorous than just manually handcrafted figures for CNN graphs or just tables with parameters for all CNN layers, the author was looking for a notation which joins both forms of the presentation for CNN structure and parameter definitions. 

The trial and error process has led the author to the concept of the decorated symbol which represents a data processing unit -- the symbol represents the operation type (like $\bb{C}$ stands for convolution) while decorators include the essential information to perform the operation (like the filter size or the number of filters). Then the author could show students of {\it Adaptive Image Recognition} course, the simplified facial landmarks detector\footnote{The architecture was developed by Rafa\l\ Pilarczyk as the lightweight solution for mobile applications [Pilarczyk].} as the following CNN formula\footnote{Here, the concept of {\it formula} is semantically more close to the chemical formula like $\bb{C}_6\bb{H}_{12}\bb{O}_{6}$ (the molecular formula for glucose) than to the mathematical formula $A_6B_{12}C_6$ (the product of three numbers located in arrays $A,B,C$).}:\\[5pt]
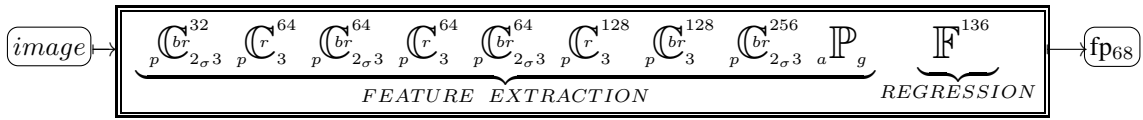
\noindent\hspace*{-2mm}
\begin{small}
\xdblock{image}{}{
$\underbrace{
\text{
\xconv{2_{\sigma}3}{32}{p}{}{br}
\xconv{3}{64}{p}{}{r}
\xconv{2_{\sigma}3}{64}{p}{}{br}
\xconv{3}{64}{p}{}{r}
\xconv{2_{\sigma}3}{64}{p}{}{br}
\xconv{3}{128}{p}{}{r}
\xconv{3}{128}{p}{}{br}  
\xconv{2_{\sigma}3}{256}{p}{}{br}
\xpool{g}{}{a}{}{}
}}_{FEATURE\ \ EXTRACTION}
\underbrace{\text{
\xdense{}{136}{}{}{}
}}_{REGRESSION}$}{\text{fp$_{68}$}}\\[2pt]
\end{small}

Yet another motivations for CNN formulas definition follow from the daily software experience: 
\begin{enumerate}
  \item CNN models are archived and transferred from one application to another one in portable way for instance as JSON documents (JavaScript Object Notation).\\
{\it   The CNN formula as the structured string of \LaTeX\  commands is portable in the same sense as any \LaTeX\ document is portable between computer platforms. However, the symbols for all inputs should be complete regarding the input tensor shapes. Binary buffers including model parameters are encoded in a portable floating point representation of real numbers, for instance in IEEE 754 standard. Moreover, the unit interconnection structure of the directed acyclic graph should be incorporated into CNN formulas. To this goal the concept of unit labeling and their outputs merging is added to the grammar of CNN formulas.
}  
\item APIs (Application Programming Interfaces) to define and use CNNs have evolved from  one stage approach where the input data, the processing operation, and the data flow were defined together, into two stage approach where the data definition and its flow are separated from the operation definitions and their graph interconnections. The latter stage is attributed by API designers as the symbolic model definition.\\
{\it The CNN formula defines the symbolic model of CNN as it includes all information on CNN units hyper parameters and on CNN units interconnection. The necessary  information to perform the tensor flow is clearly separated from CNN formula. }
\end{enumerate}

In this introduction the preliminary topics are discussed:
\begin{enumerate}
  \item The CREAMS categorization  as it is widely accepted in the digital media research.
  \item The generalized perceptron as the primary concept for tensor neural networks and its duals.
\end{enumerate}

The main body of the paper is divided into four parts: 
\begin{itemize}
\item The dual generalized perceptron for gradient computations (section \ref{sec:dgp}).
  \item The tensor concept as it is used in the CNN research (section \ref{sec:grads}). 
  \item The symbolic tensor neural networks (STNN -- section \ref{sec:stnn}). 
  \item The CNN formulas for digital media applications -- a review of the representative cases for CREAMS categorization (section \ref{sec:creams}).
\end{itemize}

\subsection{CREAMS -- digital media as research area}

The CNN formulas developed in this paper for digital media applications refer to the fundamental digital media research issues represented together by the acronym CREAMS:
\begin{itemize}
\item {\bf C}ompression of digital image, video, and audio  (efficient generative modeling).
\item {\bf R}ecognition of semantic objects in media objects (including object detection, segmentation, classification, and verification).
\item {\bf E}mbedding of one media object into another one (stego-analysis aspects, watermarking, etc.).
\item {\bf A}nnotation, aka media indexing, captioning, and summarization.
\item 3D {\bf M}odeling for human computer interfacing including gaze and pose identification.
\item Data {\bf S}ecurity schemes and algorithms based on digital media.
\end{itemize}

\subsection{Generalized perceptron-- directed graph of data processing units}

Artificial Neural Networks (ANN) can be described as a special class of  computer algorithms where the computation is performed by a collection $U$ of interconnected data processing units. 

\begin{definition}[data processing unit]
The data processing unit $u$ performs an operation\footnote{We postpone the specification of operation domain and range till the concept of tensor view attached to data buffers will be defined. Instead of range and domain compatibility for operations of connected processing units, we use the concepts of data buffer reference and data request from a referenced data buffer. The "reference to" induces here the "connection from".} $f_u$ on its input data $x_1',\dots,x_{k'}'$ and stores its results $x_1,\dots,x_k$ into its private data buffers (memory blocks) $X_1,\dots,X_k$.
The input data is propagated from other data buffers (memory blocks) $X_1',\dots,X_{k'}'$. Therefore, in a formal way {\it data processing unit} $u$ is defined by three components {\it (input buffers referenced, processing function, output buffers filled)}:
\begin{equation}
\left[
X_u^{in}\eqd(X_1',\dots,X_{k'}'),\ 
X_u^{out}\eqd(X_1,\dots,X_{k})\right] \lra
u \eqd \left(X_u^{in}; f_u; X_u^{out}\right)
\end{equation}
\end{definition}

The data processing unit requests the input data $x_i'$ from its data sources $X_i'$ which are memory blocks. The unit data source $X_i'$ is either a data buffer of another unit $u'\in U$ or $X_i'$ is a data buffer filled by an external action independent of any unit activity in the network ANN. The latter buffer can be considered as the {\it input unit}, i.e. the unit with the sole action of data input without any modification of received data. For the input unit $u$ with the buffer $X$ we apply  occasionally the following notation: $u^{in}=(;;X).$ 

In general there is no need for a definition, by symmetry, of {\it output units} as the units with outputs not requested by other units can play such a role. However, in practice each data buffer in the network can be used directly or indirectly for generating the network's nominal output or error output (while training) or observed output (while training or testing). Then the definition of the "pure" output unit which requests the data $X$ only is useful: $u^{out}=(X;;).$ 

We can consider the input request by reference as the input link between a data buffer and the processing unit where the data is propagated to. There are another implicit links: between the unit itself and its private buffers being filled by its output data. The units and the buffers extended by those links create the directed bipartite graph.
\begin{definition}[neural network bipartite graph] Let $\cl{U}$ be the set of neural network units and let $\cl{X}$ be the set of data buffers (memory blocks). Then the neural network bipartite graph $\cl{N}\eqd(V,E)$ where $V\eqd\cl{U}\cup\cl{X}$ is the set of graph vertexes and $E\eqd E_{ux}\cup E_{xu}$ is the set of directed  edges describing the input/output data connections. Hence the {\it neural network bipartite graph} $\cl{N}$ is formally defined as follows:
\begin{equation}
\begin{array}{c}
E_{ux}\eqd\left\{(u,X): u\in\cl{U},\ X\in X_u^{out}\subset\cl{X}\right\},\ 
E_{xu}\eqd\left\{(X,u): u\in\cl{U},\ X\in X_u^{in}\subset\cl{X}\right\}
\lra\\ \cl{N}\eqd\left(\cl{U}\cup\cl{X},E_{ux}\cup E_{xu}\right)
\end{array}
\end{equation} 
\end{definition}

\begin{definition}[generalized perceptron]
If the neural network's graph $\cl{N}$ has no loops, i.e. if $\cl{N}$ is the directed acyclic graph (DAG)  then the neural networks belongs to the class of the {\it Generalized Perceptrons} (GP). The data flow of GP, i.e. its {\it forward signal propagation} can be realized in the sequence of topologically sorted units $u\in\cl{U}$ by the rule:
{\it if $u$ is the input unit then $X_u^{out}$ is filled from an external data source else}
\[
\begin{array}{c}
\left[
u \eqd \left(X_u^{in}; f_u; X_u^{out}\right),\ 
X_u^{in}\eqd\left(X_1',\dots,X_{k'}'\right),\ 
X_u^{out}\eqd(X_1,\dots,X_{k})\right]
\lra\\
X_1,\dots,X_{k} \ass f_u\left(X_1',\dots,X_{k'}'\right)
\end{array}
\]
\end{definition}


The digital media applications, discussed in this paper, use GP solutions with convolutions as the main processing functions and with the tensor view of data buffers.

\subsubsection{Networks with loops -- avoiding deadlocks}
If the neural network bipartite  graph has cycles (loops) then the signal flow of GP cannot be applied in the topological sequence. Moreover, the requests for data could lead to a deadlock in the asynchronous computational model of data processing. 

We can resolve this problem if the data buffers located in a cycle are doubled. It means the units on cycles have the buffers for the output being currently computed and the buffers for the recently computed output. The requests for data refer always to the recently computed output. Hence the request links and the output links cannot create any loop. Therefore, the computation of all data outputs can be performed in the asynchronous way if only we admit the synchronous switch between the buffer labels: {\it recent output} and {\it current output}. 
The synchronized signal is issued if all units complete filling up of their current data buffers.

The above discussion leads to the conclusion that the forward signal propagation starting from inputs is possible for networks with loops if only the units in the loops read from the recent buffers and write to the current buffers. Such artificial neural networks include the class {\it Recurrent Neural Networks}  (RNN) as the {\it recent buffers} represent the state of the network. In this tutorial the neural networks with loops are not further investigated.

\subsubsection{Merging and splitting units}

In practice the computing unit function $f_u$ of GP operates on the single input data buffer ($k'=1$) and the single output data buffer ($k=1$) using their tensor views. However, in theory the unit in the given network could be implemented by another network having multiple input and output data buffers. In order to keep the property "single input -- single output" for all data processing units, in case the unit uses the output data from more than one data source, a data merging is required. To this goal we introduce the {\it merging unit} which is responsible for the task. A symmetric (dual) case when the data buffer is split into parts which are requested by another units separately is theoretically possible\footnote{Typical use of splitter is for displaying of data tensor slices.} and to this goal we introduce the {\it splitting unit.}

The merging and the splitting of any data buffer can be made on many ways. In practice tensor slices are used if tensor view is attached to the buffers. For further discussion we denote by $\mu$ any merging unit and by $\sigma$ any splitting unit. 

The merging and the splitting units can be used to normalize the units with multiple input buffers and multiple output buffers in order to get the single input and the single output for each computing unit. After such normalization we get clear subdivision of data processing units into computing units with single input/output buffers and data permuting (manipulating) units which merge and split data buffers\footnote{The special cases are pure input nodes (no input from other unit, single output) and pure output units (single input, no output buffer).}.

\begin{definition}[normalization of multi-input, multi-output unit by merging and splitting]
Let the processing unit $u=\left(X_u^{in}; f_u; X_u^{out}\right)$ with  $X_u^{in} = (X_1',\dots,X_{k'}')$  and $X_u^{out}\eqd(X_1,\dots,X_{k})$. If $k'>1$ then we define the merger unit $\mu'$ to integrate input buffers into the single buffer $X_u'=f_{\mu'}(X_1',\dots,X_{k'}')$. If $k>1$ then the merger $\mu$ integrates output buffers into the single buffer $X_u =f_{\mu}(X_1,\dots,X_{k})$. Let $\sigma$ denotes the splitter which performs the inverse action for the merger $\mu.$ Then the {\it  normalized unit} with the single input buffer $X_u'$ and the single output buffer $X_u$ is defined by the unit $u'\eqd (X_u';f_u';X_u),$ where $f_u'$ preserves the unit $u$ functionality:
$
f_{\sigma}f_u'(X_u')=f_u\left(X_1',\dots,X_{k'}'\right)\ .
$
\end{definition}

Further we assume that our GP networks are normalized in the above sense before optimizing their free parameters.

\section{Dual generalized perceptron}\label{sec:dgp}

\subsection{Abstract conjugation of neural network graph}

\subsubsection{Duals of merging and splitting operations}

Note that in the above definition we implicitly assumed the existence of inversion for merging and splitting operations.

\begin{definition}[duality of merging and splitting]
Let the inverse action for merging unit $\mu$ is performed by the the {\it dual (conjugated)} splitting unit $\sigma\eqd\ov{\mu}$. If the merging unit $\mu$ makes the inversion of splitting unit $\sigma$ then we also say that $\mu$ is dual (conjugated) to $\mu.$  By the definition the operations of conjugated splitters/mergers are related by the function inversions:
\begin{equation}
f_{\ov{\mu}} \eqd \ov{f}_{\mu} \eqd f_{\mu}^{-1},\ \ f_{\ov{\sigma}} \eqd \ov{f}_{\sigma} \eqd f_{\sigma}^{-1}
\end{equation}
\end{definition}

\subsubsection{Dual of multi-casting operation}

Let us observe that the above normalization exploiting duality of merging and splitting units is not applied for the case when the output buffer $X$ of any unit is requested by many other units $u$. The data flow from one unit to two or more units to be processed as their input data is not considered as a splitting -- it is rather a multi-casting. The dual concept to multi-casting is in-casting which aggregates (in some sense) the data coming back from the duals of  units being addressed in multi-casting.
The in-casting concept can be defined formally via the neural network bipartite graph $\cl{N}$ since the multi-casting group $mcg(X)$ of processing units $u$ requesting the data from the data buffer $X$ in terms of edges in $\cl{N}$ equals to
\begin{equation}
mcg(X) \eqd \{u\in\cl{U}: (X,u)\in E_{xu}\}
\end{equation}

\begin{definition}[in-casting unit as dual of multi-casting unit]
Let the data from the buffer $X$ be requested by the units  $u\in mcg(X)$ then the {\it in-casting} unit $\alpha$ is defined by an aggregation\footnote{When the dual units transfer gradients of error function then the aggregation $f_{\alpha}$ is the usual vector (generally tensor) summation.} operation $f_{\alpha}:$
\begin{equation}
\left[\cl{N}\eqd\left(\cl{U}\cup\cl{X},E_{ux}\cup E_{xu}\right)  \right]\lra\left[
\ov{X} \ass f_{\alpha}(\ov{X}_{u},\dots),\ \ (X,u)\in E_{xu}
\right]
\end{equation}
where $\ov{X}_u$ is the output buffer created for the dual of unit $u\in mcg(X)$. The data from all such buffers is aggregated into the buffer $\ov{X}$. The bar notation denotes that a new memory block is assigned for all dual buffers. However, all dual buffers $\ov{X}_u$ and $\ov{X}$ have the same size as the original buffer $X$.
\end{definition}

\subsubsection{Conjugation of computing unit}

Having the duality concept for unit operations and dual (conjugated) concepts for data merging, splitting and multi-casting we can define the concept of the dual  generalized perceptron\footnote{The motivation to do it is the automatic generation of the training algorithm as the dual network of the original generalized perceptron.}. We assume that the perceptron is normalized.

\begin{definition}[dual of computing unit]
Let $u\in\cl{U}$ be the computing unit processing the requested input $X_u'$ by the function $f_u$ and putting the result into the data buffer $X_u.$ Then the dual (conjugated) unit\footnote{The definition of duality assumes that for any computing unit $f_u$ used in the perceptron there exists its conjugated form referring to an input data buffer corresponding to the output buffer of the original unit $u$. The results of the dual unit are stored in a data buffer which is corresponding to the input buffer of $u.$ The duality of data buffers implicitly assumes that the buffer and its dual can store values from the same data domains.} $\ov{u}\in\ov{\cl{U}}$ is defined as follows $\ov{u}\eqd\left(\ov{X}_u;\ov{f}_u;\ov{X'}_u\right)$, where:
\begin{itemize}
  \item $\ov{f}_u$ is the dual for function $f_u$,
  \item the dual function stores the results into the dual buffer $\ov{X'}_u$,
  \item the dual unit requests the input from the dual buffer $\ov{X}_u$ which is one of two possible buffers:
\begin{itemize}  
  \item it is the buffer of in-casting unit if $|mcg(X_u)|>1,$
  \item it is the dual buffer of the unit $v$ if $mcg(X_u)=\{v\}$,
  \item it is the dual buffer of the unit $u$, i.e. $\ov{X}_u$ if $|mcg(X_u)|=0$  which becomes the buffer of the input unit in the dual network.
\end{itemize}
\end{itemize}
\end{definition}

\begin{definition}[dual of multi-casting operation]
If the unit $u\in\cl{U}$ has the output buffer $X$ and $|mcg(X)|>1$ then the multi-casting operation defines the in-casting unit $\alpha=\left(\ov{mcg(X)};f_{\alpha};\ov{X}\right),$ where $f_{\alpha}$ is the aggregation operation and $\ov{X}$ is requested by the dual of the unit $u.$
\end{definition}

\begin{definition}[dual of input unit]
If $u=(;;X)$ is the input unit then its dual (conjugated)
$\ov{u}$ is empty unit (none) since its $\ov{X}$ buffer is created either for in-casting unit $\alpha$ if $|mcg(X)|>1$ or for unit $v$ if $mcg(X_u)=\{v\}.$
\end{definition}

\begin{definition}[dual of merging and splitting units]
If $\mu\in\cl{U}$ is the unit merging the requested data $X'_1,\dots,X_{k'}'$ into the buffer $X$ then its dual (conjugated) $\ov{\mu}$ is the unit $\sigma$ which splits the requested data from the dual buffer $\ov{X}$ into the dual buffers $\ov{X'}_1,\dots,\ov{X'}_{k'}.$

On the other hand if $\sigma\in\cl{U}$ is the unit splitting the requested data $X$ into the buffers $X_1,\dots,X_{k}$ then its dual (conjugated) $\ov{\sigma}$ is the unit $\mu$ merging the requested data from the dual buffers $\ov{X}_1,\dots,\ov{X}_{k}$ into the dual buffer $\ov{X}.$
\end{definition}

\begin{definition}[dual of generalized perceptron]
The dual generalized perceptron is defined wrt the output units, i.e. the processing units with\\ $|mcg(X_u)|=0$. If we want to define as the output the data buffer with $|mcg(X_u)|>0$ then we should add to $\cl{U}$ the pure output unit $u^{out}\eqd(X_u;;)$.
Then any output unit adds the single input unit in the dual perceptron.

Let the original perceptron defines the neural network bipartite graph $\cl{N}\eqd\left(\cl{U}\cup\cl{X},E_{ux}\cup E_{xu}\right).$ Then we build the dual bipartite graph 
$\ov{\cl{N}}\eqd\left(\ov{\cl{U}}\cup\ov{\cl{X}},\ov{E}_{\ov{u}\ov{x}}\cup \ov{E}_{\ov{x}\ov{u}}\right)$ incrementally in the inverse topological order of the processing units $u\in\cl{U}:$ 
join the dual $\ov{u}$ together with dual data buffer(s) as defined for the given type of $u$ and the in-casting unit if $|mcg(X_u)|>1.$
\end{definition}

\subsection{Gradient based dual generalized perceptron}

\subsubsection{Sensitivity of modeling error on input data}

The generalized perceptron GP is used for modeling of classification and regression functions wrt the training (learning) input data. The sensitivity of model on input data $X_u\inv{n_u}$ stored in the buffer with $n_u$ real numbers is usually measured by the gradient $\od{\cl{E}}{X_u}$ of the modeling error function $\cl{E}$, $u\in\cl{U}_{in}$. In order to compute the gradients wrt input units, the existence of finite gradient $\od{\cl{E}}{X_u}$ for any unit $u\in\cl{U}$ in GP must be ensured. By the chain rule we can prove it if only there exists Jacobian matrix  $J_{f_u}\inm{n'_u}{n_u}$ for any\footnote{By $[n]$ we denote any set of $n$ integer indexes. For instance $[n]=\{0,\dots,n-1\}$ or $[n]=\{1,\dots,n\}$. The actual index set $[n]$ is the same as it is implicitly used in the notation $X\inv{n}\equiv X\in\bb{R}^{[n]}.$} $u\in\cl{U}:$
\begin{equation}
\begin{array}{ll}
\left[X=f_u(X'),\ X'\inv{n_u'},\ X\inv{n_u}\right] & \lra J_{f_u}(X')\inm{n_u'}{n_u},\\[10pt] 
& \left(J_{f_u}(X')\right)_{ij} \eqd \sp{\left(f_u(X')\right)_j}{x_i'},\ i\in[n_u'], j\in[n_u]
\end{array}
\end{equation}
Namely having $\od{\cl{E}}{X}$ we can compute $\od{\cl{E}}{X'}$ by the chain rule:
\begin{equation}\label{jacoby}
\od{\cl{E}}{X'} = J_{f_u}(X')\od{\cl{E}}{X}
\end{equation}

The error function $\cl{E}$ requests data $X$ from certain units $u\in\cl{U}_{out}$ directly or indirectly through pure output units, makes the processing of them, and finally compares the results with the training data\footnote{The  given input, and the desired output can be also processed, for instance to get implicit features for further comparisons.}. The computed measure of model inaccuracy is the value of the function $\cl{E}$.  

\begin{definition}[error gradient for output units to set their dual buffers]
Given the generalized perceptron GP with output units $\cl{U}_{out}.$ We assume that the gradient $\od{\cl{E}}{X_u}$ is somehow computed for each $u\in\cl{U}_{out}$ for the model error function $\cl{E}$. We assume that error depends on GP input data exclusively via data buffers of the output units. Those gradients are stored in the data buffers $\ov{X}_u$ which are by the definition duals (conjugated) to data buffers $X_u$, $u\in\cl{U}_{out}$.
\end{definition}

\begin{definition}[gradient dual for processing unit]
If $u=(X_u';f_u;X_u)$ then the gradient dual $\ov{u}=(\ov{X}_u,\ov{f}_u,\ov{X}'_u)$ where the gradient dual function $\ov{f}_u$ and the gradient dual buffer of the gradient dual processing unit is computed using the Jacobian equation \eqref{jacoby}:
\begin{equation}
\ov{f}_u(\ov{X}_u) \eqd J_{f_u}(X'_u)\ov{X}_u  \lra
\ov{X}'_u \eqd \od{\cl{E}}{X'} = J_{f_u}(X'_u)\od{\cl{E}}{X_u}
= \ov{f}_u(\ov{X}_u)
\end{equation}
\end{definition}

\begin{proposition}[error gradient via multi-casting group]
If model error $\cl{E}$ depends on $X_u$ only via units $v\in mcg(X_u)$ using arguments $X_u^v$ of processing functions $f_v$ then by chain rule:
\begin{equation}
\od{\cl{E}}{X_u} = \sum_{v\in mcg(X_u)}\od{\cl{E}}{X_u^v}
\end{equation}
\end{proposition}

\begin{definition}[aggregation function of in-casting unit in gradient dual network]
The gradient dual for multi-casting operation is in-casting unit with the aggregation function $f_{\alpha}:$
\begin{equation}
f_{\alpha}(\dots) = \sum_{v\in mcg(X_u)}\ov{X'}_v
\end{equation}
where $\ov{X'}_{v}$ is the gradient dual buffer of the unit $\bar{v}.$
\end{definition}

\begin{proposition}[merging and splitting units in dual GP network]
The gradient dual for merging and splitting units are the same as merging and splitting units in any dual GP network - they permute data from input(s) to output(s) only.
\end{proposition}

\begin{proposition}[sensitivity of GP model wrt to all data buffers]
Gradient dual network $\ov{\text{GP}}$ with units $\cl{U}$ and model error function $\cl{E}$ of generalized perceptron GP computes  gradients $\od{\cl{E}}{X_u}$ for all units $u\in\cl{U}$. They can be found in the dual data buffers $\ov{X}_u$ after data propagation in GP and next data propagation\footnote{The propagation in GP is called the forward propagation while the propagation in $\ov{\text{GP}}$ is called the backward error propagation as in the traditional error back-propagation algorithms the dual buffers are assigned to the original units not to their dual units.} in $\ov{\text{GP}}$.
\end{proposition}

\subsubsection{Sensitivity of modeling error on network parameters}

In order to train the generalized perceptron some unit operations $f_u$ should depend on parameters $W$ which are modified during the training process. It means that we expect the basic action $f_u$ in the form $X_u\ass f_u(X_u',W).$ For parameter optimization the stochastic gradient techniques are used as and they  expect that for any training inputs the error gradient  $\od{\cl{E}}{W}$ is computed.

\begin{theorem}[error gradient wrt parameters from unit data sensitivity]
If the computing unit $u\in\cl{U}$ depends on parameters $W$, i.e. $X_u\ass  f_u(X_u',W)$ and the sensitivity of the modeling error $\cl{E}$ wrt $X_u$  is measured by $\od{\cl{E}}{X_u}$ then
\begin{equation}
\od{\cl{E}}{W} = J_{f_u}^w(X_u',W)\od{\cl{E}}{X_u}
\end{equation}
where the Jacobian matrix $J_{f_u}^w$ computed wrt the parameters $W$ is defined as follows:
\begin{equation}
\begin{array}{l}
W\inv{n_w},\ X_u\inv{n_u},\  X_u\ass f_u(X'_u,W) \lra\\[10pt]
\left(J_{f_u}^w(X'_u,W)\right)_{ij} \eqd \sp{\left(f_u(X'_u,W)\right)_j}{w_i},\ i\in[n_w], j\in[n_u]
\end{array}
\end{equation}
If $W$ is shared by units in the set $\cl{U}_w$ the gradient of the error function $\cl{E}$ wrt to parameters $W$ equals to:
\begin{equation}
\od{\cl{E}}{W} = \sum_{v\in\cl{U}_w}J_{f_v}^w(X_v',W)\od{\cl{E}}{X_v}
\end{equation}
\end{theorem}

\begin{proof}
Without the change of dual generalized perceptron which can compute the computing gradients for data buffers of all units we can compute the gradients wrt parameters $W$ by the following gedankenexperiment:
\begin{enumerate}
  \item Assume that  $W$ is the data buffer of some extra input unit $u_w$ with the input buffer $W$, i.e. $u_w\eqd(;;W).$ 
  \item Now the parameter dependent units get at least two arguments.
  \item During normalization any parameterized unit will get its single argument from the merger.
  \item Each merger before the parametrized unit has the input from extra input unit of type $u_w$.
  \item In the dual network this input of merger becomes the output of splitter which splits this part of gradient which corresponds to $W$. This part can be deduced form Jacobian equation \eqref{jacoby} separated into parameter and non-parameter parts:
  \[
  \od{\cl{E}}{X',W} = J_{f_u}(X',W)\od{\cl{E}}{X} =
  \left[
  \begin{array}{l}
  J_{f_u}^w(X',W)\od{\cl{E}}{X}\\[5pt]
  J_{f_u}^x(X',W)\od{\cl{E}}{X}
  \end{array}
  \right]
\]
where
\[
\begin{array}{l}
W\inv{n_w},\ X_u\inv{n_u},\  X_u\ass f_u(X'_u,W) \lra\\[10pt]
\left(J_{f_u}^w(X'_u,W)\right)_{ij} \eqd \sp{\left(f_u(X'_u,W)\right)_j}{w_i},\ i\in[n_w], j\in[n_u]
\end{array}
\] 
\item Finally, the gradient of error function wrt to parameters $W$ equals to:
\[
\od{\cl{E}}{W} = J_{f_u}^w(X_u',W)\od{\cl{E}}{X_u}
\]
\item If $W$ is shared by units in the set $\cl{U}_w$ then actually $\cl{U}_w$ is the multi-casting group for the unit $u_w$. Hence the gradient of the error function wrt to the parameters $W$ equals to:
\[
\od{\cl{E}}{W} = \sum_{v\in\cl{U}_w}J_{f_v}^w(X_v',W)\od{\cl{E}}{X_v}
\]
\end{enumerate}
\end{proof}


\section{Tensor data for digital media applications\label{sec:grads}}

Tensor $T$ of $d$ dimensions is a generalization of vector and matrix to $d$ dimensions. Formally the tensor of real values is a mapping defined on a set of $d$-dimensional indexes $I_T$ (aka multi-indexes or index vectors) to reals $\bb{R}$, i.e. $T\inv{I_T}:$
\begin{equation}
T:I_T\ra \bb{R},\ I_T\subseteq[n_1]\times\dots\times[n_d]
\end{equation}
where $n_i$ is the resolution of $i$-th axis. 

\subsection{Tensor axes in digital media applications}

In digital media applications the {\it index axes} beside the sequential id -- axis $1$, $\dots$, axis $d$ -- they get the symbolic id, e.g.: axis $x$, axis $y$, (attribute) axis $a$ , (batch) axis $b$. The ordering of axes  symbols (called as {\it axes signature}) implies the sequential id. For instance, the  batch of $n_b=10$ color images of pixel resolution $640\times 480$ ($n_x=640,$ $n_y=480$) with RGB color attribute $n_a=3$ is 4D tensor $T$ with the signature $byxa,$ is the mapping 
$T:[10]\times[480]\times[640]\times[3]\ra\bb{R}$ with $T[b,y,x,a]$ identifying its arbitrary element if $b\in[10],$ $y\in[480],$ $x\in[640],$ $a\in[3].$ The interpretation of tensor's content in this case is the level of the given spectrum channel, i.e. the level of color component.

In signal neural networks there are observed trends for generalizations of signal dimensions (axes). Nowadays,  we can specify the following types of input and activation tensors $T\inv{I}$ where $I$ is the set of index vectors of the fixed length:
\begin{enumerate}
  \item $T\inv{N_a}$ -- feature tensor with $N_a$ attributes,
 \item $T\inv{N_a\times I_s}$ -- feature tensor of $N_a$ attributes for signal samples indexed by $I_s,$ where $I_s$ is the signal domain indexes -- in particular:
 \begin{enumerate} 
  \item $I_s=[N_x]\lra T\inv{N_a\times N_x}$ -- feature tensor of $N_a$ attributes  for $N_x$ samples of 1D signal,
  \item $I_s=[N_x]\times[N_y]\lra T\inv{N_a\times N_y\times N_x}$ -- feature tensor of $N_a$ attributes  for $N_x\cdot N_y$ pixels in 2D image,
  \item $I_s=[N_z]\times[N_y]\times[N_x]\lra T\inv{N_a\times N_z\times N_y\times N_x}$ -- feature tensor of $N_a$ attributes  for $N_z\cdot N_y\cdot N_x$ voxels in 3D image,
  \item $T\inv{N_a\times N_t\times \dots}$ -- temporal sequence of $N_t$ tensors in one of the above types of tensors,
  \end{enumerate}
  \item $T\inv{N_b\times N_a\times I_s}$ -- batch of $N_b$  feature tensors of $N_a$ attributes for signal samples indexed by $I_s$.
\end{enumerate}

In the tensor approach, the computing  unit transforms its input data buffer representing a tensor $T_0\inv{I_0}$ into its output data buffer representing tensor $T_1\inv{I_1}.$ The output index $I_1=[N_a']\times I_s'$ is determined from $I_0=[N_a]\times I_s$, in general, for each type of processing in a specific way. In case of batch processing the batch size is never changed: 
\[I_0=[N_b]\times [N_a]\times I_s\lra I_1=[N_b]\times [N_a']\times I_s'\]

\subsection{Parameter tensors and linear operations}

\subsubsection{Multiplication (composition) of tensors}

The tensor multiplication  (composition) $C\eqd A\cdot B \eqd AB$ is defined by generalization of matrix by matrix multiplication where the scalar indexes are replaced by vector indexes (multi-indexes). If tensor $A$ is defined on the set of multi-indexes of the form $I_A\eqd I_0\times I_1$ while the tensor $B$ on the multi-indexes $I_B\eqd I_1\times I_2$ then the tensor $C$ by the definition is specified for the indexes in the set $I_C\eqd I_0\times I_2$ by the formula: 
\begin{equation}\label{eq:ten-dot}
C[p,r] = \sum_{q\in I_1}A[p,q]B[q,r],\ p\in I_0,\ r\in I_2
\end{equation}
If we have a translation tensor $T\inv{I_0\times I_2}$ then the affine operation is defined as follows:
\begin{equation}\label{eq:ten-dott}
C[p,r] = \sum_{q\in I_1}A[p,q]B[q,r] + T[p,r],\ p\in I_0,\ r\in I_2
\end{equation}

Let us observe that the tensor composition of tensor $A\inv{I_A}$ and $B\inv{I_B}$ is defined wrt $I_1$ -- the common subindex of $A$ and $B$: $I_A\eqd I_0\times I_1,$ $I_B\eqd I_1\times I_2.$ Therefore for the given $I_A$ and $I_B$ we can define many composition operations of different output tensor represented by index decomposition $I_0\times I_2$. For instance for $I_A=I_B=[n]^d$ we get $d$ possible composition operations, including the dot product if $I_0=\text{\O}=I_2.$ 
However, not for any trailing index $q$ in $I_A$ we can always find the same leading index in $I_B$ and opposite. Depending on this "trailing-leading index compatibility" we can get from $0$ to $\min(d_A,d_B)$ different composition operations, where $d_X$ is the index length $|i|$, the same for any $i\in I_X.$ 

For the digital media applications discussed in this tutorial the composition operation \eqref{eq:ten-dot} is used only by so called {\it full connection layer} (aka dense layer)  when $I_A=[n_{out}]\times I_B$. Then $I_0=[n_{out}],$ $I_1=I_B,$ $I_2=\text{\O},$ and $A$ is the parameter tensor  matching input tensor $B$ to produce the output vector $C\inv{n_{out}}$ in the composition operation \eqref{eq:ten-dot}.

\subsubsection{Parameter tensor for convolution}

The convolution is the generalization of the single kernel (mask) linear filter, implementing the dot product of the kernel coefficients with relevant slices of the input signal. 

Referring to different tensor types we can distinguish and suitably interpret,  the following types of $W$ tensors for convolutions\footnote{The translation tensor $B\inv{N_a'}$ is actually the vector with the single scalar value for each output attribute.}. The general shape of kernel tensor $W$ has the form $W\inv{N_a'\times N_a\times J_s},$ where $J_s$ is the signal domain of the kernel:
\begin{enumerate}
  \item $J_s=[n_x]\lra W\inv{N_a'\times N_a\times n_x}$ -- the tensor of $N_a'$ kernels for 1D signal convolution implemented in the window of $n_x$ samples, together for all $N_a$ input attributes\footnote{As the result $N_a$ input features of any sample are linearly transformed to its $N_a'$ output features.},
  \item $J_s=[n_y]\times[n_x]\lra W\inv{N_a'\times N_a\times n_y\times n_x}$ -- the tensor of $N_a'$ kernels for 2D signal convolution implemented in the window of $n_y\times n_x$ samples, together for all $N_a$ input attributes,
  \item $J_s=[n_z]\times[n_y]\times[n_x]\lra W\inv{N_a'\times N_a\times n_z\times n_y\times n_x}$ -- the tensor of $N_a'$ kernels for 3D signal convolution implemented in the window of $[n_z]\times [n_y]\times [n_x]$ samples, together for all $N_a$ input attributes\footnote{In our review of applications the 3D convolutions are used only in image compression when the entropy of quantized features is estimated. However, then the $z$ axis is identified with $a$ axis at entropy optimization. Note that like for $n_x,n_y$, $n_z$ is a small fraction of $N_z.$},
\item $J_s=[n_t]\times .\lra W\inv{N_a'\times N_a\times n_t\times .}$ -- the tensor of $N_a'$ kernels for in time 1D/2D/3D signal convolution implemented in the window of the given shape together for all $N_a$ input attributes.
\end{enumerate}

If $I_Y\eqd I_Y^a \times I_Y^s,$ denotes the set of multi-indexes of the output tensor decomposed into the signal domain $I_X^s$ and in the attribute domain $I_X^a$ then for the convolution operation, the indexes $I_W$ of the kernel window (e,g. $I_W=[n_y]\times [n_x]$) are used to locate input tensor elements for the local dot products \footnote{In the formula we assume: (a) the domain of the output signal tensor $Y$ fits to the domain of the input tensor $X$ and to the padding mode at the nD signal boundaries, (b) indexes begin from zero, (c) addition of multi-indexes is vectorial.}:
\begin{equation}\label{eq:conv-xd}
\begin{array}{l}
\ds
Y[f',p] \ass \sum_{f\in[N_a]}\sum_{q\in J_s} W[f',f,q]\cdot X[f,p+q]+B[f']\\[5pt] 
f'\in [N_a'], \forall q\in J_s\left[ p\in I_Y^s \lra (p+q)\in I_X^s\right]
\end{array}
\end{equation}
In case of batch tensors, the formula for the convolution operation has the form:
\begin{equation}\label{eq:conv-bxd}
\begin{array}{l}
\ds
Y[b,f',p] \ass \sum_{f\in[N_a]}\sum_{q\in J_s} W[f',f,q]\cdot X[b,f,p+q]+B[f']\\[5pt] b\in[N_b], f'\in [N_a'], \forall q\in J_s\left[ p\in I_Y^s \lra (p+q)\in I_X^s\right]
\end{array}
\end{equation}

\subsection{Tensor representation of data buffer}
The tensor $T\inv{I_T}$ to be processed is represented (stored) as data memory block (data buffer) $X\inv{|I_T|}.$ There are many ways to store elements of $T$ into $X$ like there is many one-to-one mapping from the set of multi-dimensional indexes $I_T$ to the set of scalar indexes $\left[|I_T|\right].$ Further we call such mapping as {\it the addressing scheme}.

\begin{definition}[strided addressing]
In practice of tensor processing packages the {\it affine addressing scheme} is used (aka {\it strided addressing}):
\begin{equation}\label{astrided}
a\eqd address(i_1,\dots,i_d) \eqd o+\tp{s}i = o +\sum_{j=1}^ds_ji_j \lra X[a] = T[i_1,\dots,i_d]
\end{equation}
where the stride vector $s\in\bb{N}^d$ and the index vector (multi-index) $i\in\bb{N}^d.$ The offset $o$ means that the data of tensor $T$ begins from the element $X[o]$ in the data memory block $X$. The $address$ function returns the index in the data buffer $X$ -- it is not the byte offset from the beginning of $X$. 
\end{definition}

\begin{definition}[monotonic addressing  wrt lexicographic order of indexes]
Let $I_T\subseteq[n_1]\times\dots[n_d]$ be the index domain of the tensor $T:I_T\ra\bb{R}.$ The {\it lexicographic order} $\prec$ of $I_T$ is defined in the standard way:
\[
(i_1,\dots,i_d)\prec(j_1,\dots,j_d)\ \text{ iff }\ i_1<j_1\ \text{ or } \exists k\in\langle 1,d), i_1=j_1,\dots,i_k=j_k, i_{k+1}<j_{k+1}
\]
We say that {\it addressing scheme is monotonic wrt  the lexicographic order} 
\[
(i_1,\dots,i_d)\prec(i'_1,\dots,i'_d) \lra  address(i_1,\dots,i_d)<address(i'_1,\dots,i'_d)
\] 
Briefly, the {\it addressing scheme is lexicographic} iff it is monotonic wrt  the lexicographic order.
\end{definition}

It is interesting that there is only one strided addressing scheme which is lexicographic, i.e. it preserves the lexicographic order.
\begin{theorem}[addressing scheme vs. lexicographic order]
Let $I_T=[n_1]\times\dots[n_d]$. The the strided addressing given by the formula \eqref{astrided} is monotonic wrt the lexicographic ordering of $I_T$ iff 
\begin{equation}
s_{d}=1, s_{j-1}=n_{j}s_j, j=d,\dots,2
\end{equation}
The explicit form of the above strides is given by the equations:
\begin{equation}\label{lexstrides}
s_d=1,\ s_j=\prod_{k=j+1}^dn_k,\ j=1,\dots,d-1
\end{equation}
The above strides are the weights of the {\it mixed radix representation} for integers $a\in\langle 0,n_1\cdot\dots\cdot n_d).$ They define the unique lexicographic strided addressing scheme. 
\end{theorem}

\begin{proposition}[from address to index]
Let $I_T=[n_1]\times\dots[n_d]$ and  $n\eqd n_1\cdot\dots\cdot n_d$\ . We assume the strided lexicographic addressing  of tensor indexes. Then for any address $a\in\langle 0,n)$ of an element in tensor $T$ its indexes $[i_1,\dots,i_d]$ can be computed by {\it  the mixed radix algorithm}: 
$
\left\{
\begin{array}{l}
b\ass a\\
for\ k=d,\dots,1:\\
\hspace*{7mm}(b,i_k)\ \ass\ (b\div n_k,\ b\bmod n_k)
\end{array}
\right.
$
\end{proposition}

\begin{definition}[tensor view for a data buffer]
Let $X$ be the data buffer with $n$ elements. The tensor $T\inv{n_1\times\dots\times n_d}$ with $n = n_1\cdot\dots\cdot n_d$ is {\it the tensor view of the data buffer} $X$ iff
there exists the strides vector $s\in\bb{N}^d$: 
\[
T[i] = X[\tp{s}i],\ \text{ for any }\ i\in [n_1]\times\dots\times[n_d]
\]
\end{definition}
The strided addressing schemes which are not lexicographic can define interesting tensor views $T$. For instance:
\begin{enumerate}
  \item {\it RGB image planes.}\\
  If the data buffer $X\inv{3n_xn_y}$ represents in the lexicographic order RGB image $I_{RGB}\in\langle 0,255\rangle^{n_y\times n_x\times 3}$ with the strides $s_1=3n_x,s_2=3,s_3=1$ then the same image with tensor view $I'_{RGB}\in\langle 0,255\rangle^{3\times n_y\times n_x}$ is defined as the view of $X$ by strides $s'$: 
  \[s_1'=s_3=1,\ s_2'=s_1=3n_x,\ s_3'=s_2=3\]
  In $I'_{RGB}$ the pixel color components are virtually unpacked to image color components.
\item  {\it Image transposition.}\\If the data buffer $X\inv{n_xn_y}$ represents in the lexicographic order the luminance image $Y\in\langle 0,255\rangle^{n_y\times n_x}$ with the strides $s_1=n_x,s_2=1$ then the same image with tensor view $Y'\in\langle 0,255\rangle^{n_x\times n_y}$ is defined as the view of $X$ by strides $s'$: 
 $s_1'=s_2=1,\ s_2'=n_x.$ This view actually implements the transposition of image matrix $Y$ without moving its elements:$Y'=\tp{Y}.$
 \item{\it  Luminance to RGB buffer}.\\Certain graphics system to display the luminance image $Y$ require the conversion of $Y\in\langle 0,255\rangle^{n_y\times n_x}$ into $RGB\in\langle 0,255\rangle^{n_y\times n_x\times 3}$ image with components $R=G=B=Y.$ The tensor view $Y_{RGB}$ with strides $s_1=n_x,s_2=1,s_3=0$ provides the necessary sequence of values to fill the graphics buffer without creation of intermediate RGB image in the application memory.
 \item {\it Temporal signal windowing}.\\If $X\inv{n}$ is the data buffer which represents the $n$ samples of temporal signal and $n=k\cdot\tau,$ i.e. the signal can be divided into $k$ temporal windows of length $\tau$ then the tensor view $X_w\inm{k}{\tau}$ representing the windows and the samples in the windows together is defined by strides: $s_1=\tau, s_2=1.$
 \item {\it Sliding window for 1D signal and 1D convolution.}\\
Let $X\inv{n_x}$ be the discrete signal of length $n_x$ and $W\inv{n_w}$ be the convolution kernel of length $n_w<n_x,$ to be applied with offset $d_x$. Then the output signal is computed in the positions $i\cdot d_x$ such that $i\cdot d_x+n_w\leq n_x$. Hence the filtered signal $Y\inv{n_{out}}$, where $n_{out}= 1+\lfloor(n_x-n_w)/d_x\rfloor.$
The sliding window (without padding) is the tensor view $X_s\inm{n_{out}}{n_w}$ with the strides $s_1=d_x, s_2=1$. Then the convoluted signal can be computed as $Y=X_s\cdot W$, i.e. the convolution operation is implemented by the matrix $X_s$ composed with the vector $W$.
\item {\it Sliding window for 2D signal and 2D convolution.}\\
Sliding window with offsets $d_x, d_y$ for the luminance image $X\inm{n_y}{n_x}$ and 2D convolution kernel $W\inm{n_h}{n_w}$ can be also arranged to implement the 2D convolution by the multiplication of a 4D tensor view by a 2D kernel (if the full index matching is used for tensor multiplication) or 3D tensor by 1D vector (if 1d index matching is allowed for multiplication). Namely, like for 1D convolution, the filtered image $Y\inm{n_y'}{n_x'}$ has the shape 
\[n_y' = 1+\lfloor(n_y-n_h)/d_y\rfloor,\ \ n_x' = 1+\lfloor(n_x-n_w)/d_x\rfloor\ .\]
The sliding window now results in 4D tensor view $X_s\in\bb{R}^{n_y'\times n_x'\times n_h\times n_w}$ of $X$, defined by strides $s_1=d_yn_x,$ $s_2=d_x,$ $s_3=n_x,$ $s_4=1.$  In the full index matching we get the filtered luminance image $Y=X_sW$. If only 1D indexes are matched then we have to change the sliding window view of $X$  from 4D tensor $X_s$ to 3D tensor $X_s'\in\bb{R}^{n_y'\times n_x'\times n_h\cdot n_w}$ with strides $s_1=d_yn_x,s_2=d_x,s_3=1$ and the 2D kernel $W$ to 1D vector view $W'$ with strides $s_1'=n_w,$ $s_2'=1.$ Then the filtered luminance $Y=X_s'W'.$
\end{enumerate}

\begin{definition}[index range]
Let $a,b,r\in\bb{Z},$ $r\neq 0$, $(b-a)\cdot r>0$. Then the notation $a:b:r$ denotes  the range of integers, aka the arithmetic sequence of integers $x_i$. It begins from $a$, moves by $r$ and is strictly bounded by $b$:
\[
\begin{array}{l}
x_0 = a,\ x_i = x_{i-1}+r,\ i>0\\
r>0 \lra x_i<b,\ \ r<0 \lra x_i>b, i>0
\end{array}
\]
\end{definition}

\begin{definition}[tensor slice]
If $I_T= [n_1]\times\dots\times[n_d]$ be the index of the tensor view $T$ for the data buffer $X\inv{n}$ with the strides $s_1,\dots,s_d$. Let also $I_T^j\eqd (a_j:b_j:r_j)\subseteq[n_j],$ $j=1,\dots,d,$ be ranges. Then $T[I_T^1,\dots,I_T^d]$ is the {\it tensor slice} which is also the tensor view $T'$ of $X$ with the following address function wrt to the buffer $X$:
\begin{equation}
address(i_1,\dots,i_d) = \tp{s}(a+r\odot i)
\end{equation}
where $i_j\in[k_j],$ $k_j=1+\left(b_j-a_j-\text{sign}(r_j)\right)\div r_j$, $j=1,\dots,d,$ and $\odot$ is component-wise multiplication. Hence the tensor view $T'\in[k_1]\times\dots\times[k_d].$
\end{definition}

\begin{proposition}[tensor slice is always nonempty]
The conditions for the index range $r_j\neq 0$, $(b_j-a_j)\cdot r_j>0$ and the range inclusion $I_T^j\subseteq[n_j]$ imply:
\begin{itemize}
\item $1\leq k_j\leq n_j$
\item $k_j=1$ iff $|b_j-a_j|\leq |r_j|$ 
\item $k_j=n_j$ iff $\left[|r_j|=1\ \text{ and }\ r_j(b_j-a_j)=n_j\right]$
\end{itemize}
Hence, the tensor slice is always nonempty. However, the tensor could be equal to the original tensor view or even it could be $d$-dimensional tensor with the single element $T[a_1,\dots,a_d].$
\end{proposition}

\subsection{Gradient duals for tensor computing units}

The dual network  is defined via dual operations for computing units and dual DAG supported by  duals for data multi-casting, merging, and splitting operations (cf. section \ref{sec:dgp}). While the latter three are common for all generalized perceptrons, the gradient dual of computing unit is specific for the function the unit computes. 

Assuming our data buffer model for processing units which represent tensor views, we can observe that in daily practice of neural computing  two types of functions of form $f:\bb{R}^{n^{in}}\ra\bb{R}^{n^{out}}$ prevail: differentiable functions\footnote{There are few exceptions like ReLU and Trunc functions where the gradient does not exist in few points of their domain. However, the sub-gradient sets are bounded.} which compute the single output vector component from the single input vector component (component-wise functions) and functions which change any output tensor component proportionally to the change of any input tensor component. In principle for both types of units their gradient duals computed as the  gradient $\ov{x}^{in}\eqd \od{\cl{E}}{x^{in}}$ of the network error function $\cl{E}$ wrt to the input data $x^{in}\inv{m}$ can be easily determined via the chain rule \eqref{jacoby}: $\ov{x}^{in} = J_{f}(x^{in})\ov{x}^{out}$ where $x^{out}=f(x^{in}).$
\begin{itemize}
  \item {\it Component-wise computing units and their duals: }
\[\left[n^{out}=n^{in},\ x^{out}\eqd f(x^{in})\right] \lra x_i^{out} = f_i(x^{in}_i),\ \ \forall i\in[n]\ .\]
Then the Jacobian matrix is diagonal and we can compute the gradient in the component-wise manner, as well:
\begin{equation}
\ov{x}^{in}_i = \frac{d\left[f_i(x_i^{in})\right]}{dx_i^{in}}\ov{x}_i^{out},\ \ \ i\in[n^{out}]
\end{equation}
For instance for ReLU ({\bf Re}ctified {\bf L}inear {\bf U}nit) unit and $i\in[n^{out}]:$
\begin{equation}\label{eq:relu}
x_i^{out} = ReLU(x_i^{in})\eqd\left\{
\begin{array}{ll}
0 & \text{if }\ x_i^{in}\leq 0,\\
x_i^{in} & \text{otherwise.}
\end{array}
\right. 
\ \ \lra\ \ \ \
\ov{x}^{in}_i = \left\{
\begin{array}{ll}
0 & \text{if }\ x_i^{in}< 0,\\
\frac{1}{2}\ov{x}_i^{out} & \text{if }\ x_i^{in}= 0,\\
\ov{x}_i^{out} & \text{if }\ x_i^{in}> 0.
\end{array}
\right.
\end{equation}

The rationale behind the value $\frac{1}{2}$ taken from the sub-gradient of ReLU at the point $0$ is the batch computation of the gradient used in the stochastic optimization of TNN. Namely, let us slightly modify the current input of the network to get small positive $x_i^{in}(1)>0$ and small negative $x_i^{in}(2)<0$ values for any ReLU units with  $x_i^{in}=0$,  and let us replace in the batch the original input by such pair of new inputs. To keep the same number of elements in batch computation of gradient we replace the undefined gradient by the average of defined gradients: 
$\frac{1}{2}\left(1\cdot\ov{x}_i^{in}(1)+0\cdot\ov{x}_i^{in}(2)\right)\approx\frac{1}{2}\ov{x}_i^{in}.$ Hence the value $\frac{1}{2}$ arises. 
  \item {\it Affine computing units and their duals}: 
\[
\begin{array}{ll} 
\left[x^{out}\eqd f(x^{in}), \Delta x^{out} \eqd (f(x^{in}+\Delta x^{in}) - f(x^{in}))\right]\ \lra 
& \Delta x_i^{out}= a_{ij}\Delta x^{in}_j\\[5pt]
&\forall i\in[n_u^{out}], j\in[n_u^{in}]\ .
\end{array}
\]
Note that from the above definition easily follows that the affine function is the linear function determined by the the matrix $A\inm{n^{out}}{n^{in}}$ with the bias $b\eqd f(0):$ $x^{out}=f(x^{in})=Ax^{in}+b, x^{in}\inv{n^{in}},$ $A\inm{n^{out}}{n^{in}}, b\inv{n^{out}}.$ 
The Jacobian matrix $J_f$ for the affine function $f$ equals to the transposed matrix $\tp{A}$ as:
\[
\begin{array}{l}
x_i^{out} = \dots +a_{ij}x_j^{in}+\dots+b_i,\ i\in[n^{out}], j\in[n^{in}] \lra \\[10pt]
\left(J_{f}(x^{in})\right)_{ji} \eqd \sp{\left(f_i(x^{in})\right)}{x^{in}_j}=a_{ij},\ i\in[n^{out}], j\in[n^{in}]
\ \ \lra\ \ J_f(x^{in}) = \tp{A}
\end{array}
\]
Hence the gradient flow equation for the affine unit has the form:
\begin{equation}\label{eq:grad-flow}
\od{\cl{E}}{x^{in}} = J_f(x^{in})\od{\cl{E}}{x^{out}} \lra \ov{x}^{in} = \tp{A}\ov{x}^{out}
\end{equation}
Except the full connection function, the matrix $A$ is sparse  and usually components of $b$ are not independent -- the convolution is good example for both cases. In applications we never represent explicitly the sparse and redundant matrix $A$ and the redundant vector $b$. In case of convolution we use the kernel and bias parameters for filter value computation. 
\end{itemize}

There is another implication of chain rule important for programming  of the gradient flow through the linear units. 
\[
y_i \eqd \sum_j a_{ij}x_j \lra
\ov{x}_j\eqd\sp{\cl{E}}{x_j} = \sum_i \sp{\cl{E}}{y_i}\sp{y_i}{x_j} = 
\sum_i \ov{y}_i a_{ij}
\]

It leads to the conclusion  that any programming code designed to compute the forward flow through a linear unit can be changed to the code computing the gradient flow just by modification the single line!

\begin{proposition}[single line modification to get the gradient flow]
\begin{equation}\label{eq:single-line}
\boxed{\big[y_i \xleftarrow{+} a_{ij}x_j\big] \lra \big[\ov{x}_j \xleftarrow{+} a_{ij}\ov{y}_i\big]}
\end{equation}
\end{proposition}

\subsubsection{Full connection unit and its gradient flow}

The full connection (dense) unit is implemented by the composition of the parameter tensor\\ $W:I^{out}\times I^{in}\rightarrow \bb{R}$ with the input tensor $X^{in}: I^{in} \rightarrow \bb{R}$. The composition is followed by the vector translation  $B$. Hence the output tensor of the dense unit is the vector $X^{out}: I^{out}\rightarrow \bb{R}.$ 
\begin{equation}\label{eq:fc-gen}
\ds \left[X^{out} = WX^{in}+B,\ i\in I^{out}\right] \lra  X^{out}[i] = B[i]+\sum_{j\in I^{in}} W[i,j]X^{in}[j]
\end{equation}
Let $n^{in}\eqd|I^{in}|$, $n^{out}\eqd|I^{out}|$, ${_vW}\inv{n^{out}\cdot n^{in}},$ ${_vX^{in}}\inv{n^{in}}$ be the data buffers (vectors) -- the tensors $W, X^{in}$ are monotonic lexicographic views at. Then the  above formula can be written in the vectorial form:
\begin{equation}
i\in[n^{out}] \lra {_vX}^{out}[i] = B[i] + \sum_{j\in[n^{in}]} {_vW}[\underbrace{n^{in}\cdot i+j}_{k\in[n^{in}\cdot n^{out}]}]\cdot{_vX^{in}}[j]
\end{equation}
Therefore the full connection unit can be implemented without using explicit strides for computing addresses of data elements in the buffers\footnote{In the pseudo-code below, we wssume the zero-based indexing of data arrays.}:
\begin{equation}\label{eq:fc-pcode}
\left\{
\begin{array}{l}
 {_vX}^{out}\ass \bm{0};\  i\ass 0;\ j\ass 0\\
for\ \ k\in[n^{in}\cdot n^{out}]:
\left\{
\begin{array}{l}
{_vX}^{out}[i] \xleftarrow{+} {_vW}[k]\cdot {_vX^{in}}[j];\ j\xleftarrow{+}1\\[2pt]
if\ \ j=n_{in}:
\left\{
\begin{array}{l}
{_vX}^{out}[i] \xleftarrow{+}B[i];\ i\xleftarrow{+}1\\
j \ass 0
\end{array}
\right.
\end{array}
\right.
\end{array}
\right.
\end{equation}

The above code will get even less addressing operations if an extra variable $s$ is used for row dot product computations:
\[
\left\{
\begin{array}{l}
s \ass 0;\  i\ass 0;\ j\ass 0\\
for\ \ k\in[n^{in}\cdot n^{out}]:
\left\{
\begin{array}{l}
s \xleftarrow{+} {_vW}[k]\cdot {_vX^{in}}[j];\ j\xleftarrow{+}1\\[2pt]
if\ \ j=n_{in}:
\left\{
\begin{array}{l}
X[i] \ass B[i]+s;\ i\xleftarrow{+}1\\
s\ass 0;\ j \ass 0
\end{array}
\right.
\end{array}
\right.
\end{array}
\right.
\]
Let now ${_mW}\inv{n^{out}\times n^{in}}$ be the matrix view of the parameter buffer ${_vW}$  then using \eqref{eq:grad-flow} we can conclude  the gradient flow formula for the full connection unit:
\begin{equation}
\begin{array}{c}\label{eq:fc-grad-x}
\ {_mW}[i,j] \eqd {_vW}[i\cdot n^{in}+j] \eqd W[i,index(j)]\ \lra\\[3pt]
X = B+({_mW})(_vX^{in})\ \lra\ \ov{_vX}^{in} = \tp{({_mW})}\ov{X}^{out}\ \lra\\[3pt]
\od{\cl{E}}{_vX^{in}} = \tp{({_mW})}\od{\cl{E}}{{_vX}^{out}}\ \lra\ \od{\cl{E}}{X^{in}} = \tp{W}\od{\cl{E}}{X^{out}}
\end{array}
\end{equation}
where the multi-index $j'\eqd index(j)\in I^{in}$, corresponding to the buffer address $j$, is appended to the index $i\in[n_{out}].$  The last tensor transposition $\tp{W}$ is performed with respect the pair of multi-indexes $(i,j)\mapsto (j,i)$:
\[
i\in I^{out},\ j\in I^{in} \lra \tp{W}[j,i]\eqd W[i,j]\ .
\]

The algorithm to fill the dual buffer $\ov{X}^{in}$ with error gradient has a dual form wrt to the algorithm \eqref{eq:fc-pcode} filling the buffer $X$ by regular data signal:

\[
\left\{
\begin{array}{l}
 {_v\ov{X}}^{in}\ass \bm{0};\  i\ass 0;\ j\ass 0\\
for\ \ k\in[n^{in}\cdot n^{out}]:
\left\{
\begin{array}{l}
{_v\ov{X}}^{in}[j] \xleftarrow{+} {_vW}[k]\cdot {_v\ov{X}^{out}}[i];\ j\xleftarrow{+}1\\[2pt]
if\ \ j=n_{in}:
\left\{
\begin{array}{l}
i\xleftarrow{+}1;\ j \ass 0
\end{array}
\right.
\end{array}
\right.
\end{array}
\right.
\]


We will show now that the error gradient wrt $_mW$ is the outer product of the gradient wrt $X$ with the vector $_vX^{in}$. Namely
\[
\begin{array}{l}
\sp{\cl{E}}{_vW[k]} = \sp{\cl{E}}{W[i,j]}= \sp{\cl{E}}{X^{out}[i]}\sp{X^{out}[i]}{W[i,j]}=
\sp{\cl{E}}{X^{out}[i]}X^{in}[j]
\end{array}
\]
Hence the buffer of the error gradient wrt to the parameters $W$ of dense unit at the position\\ $k\in[n^{in}\cdot n^{out}]$ is filled by:
\begin{equation}
\left(\od{\cl{E}}{_vW}\right)_k = \left(\od{\cl{E}}{X^{out}}\right)_{i}\cdot ({_vX^{in}})[j]\ \ 
\text{where }\ \ k = i\cdot n^{in}+j
\end{equation}
Finally, the matrix view of this gradient has mathematically elegant form of the outer product:
\begin{equation}\label{eq:fc-grad-w}
\od{\cl{E}}{_mW} = \od{\cl{E}}{X^{out}}\cdot \tp{(_vX^{in})}
\end{equation}

The error gradient of dense unit wrt translation parameters $B\inv{n_{out}}$ can be derived easily:
\[
\begin{array}{l}
\sp{\cl{E}}{B[i]} = \sp{\cl{E}}{X^{out}[i]}\sp{X^{out}[i]}{B[i]}=
\sp{\cl{E}}{X^{out}[i]}\cdot 1
\end{array}
\]
Therefore for the dense unit:
\begin{equation}\label{eq:fc-grad-b}
\od{\cl{E}}{B} = \od{\cl{E}}{X^{out}}
\end{equation}

\subsubsection{Sliced full connection unit and its gradient flow}

In full connection unit all features are linearly mapped into a new feature vector. Then we mix data from different axes. If the signal domain axes are not homogeneous like frequency and time in speaker identification, for engineers is more clear if we mix features in homogeneous stages. Therefore we prefer the full connection in slices, i.e. along the axes which are considered as homogeneous.

In order to define the concept of full connecting in slices, let us divide the index of the signal domain into two indexes $(p,q)\in[N_s^{in}].$ Then the regular full connecting unit could be described as follows:
\[
X_{FC}^{out}[a] \eqd \sum_{a',p,q}A[a,a',p,q]X^{in}[a',p,q]
\]
In case of the slice indexed by $p$ the full connecting along this slice has the form:
\[
X_{FC(p)}^{out}[a,p] \eqd \sum_{a',q}A[a,a',p,q]X^{in}[a',p,q]
\]
Hence the regular full connection function can be summed up from the sliced versions:
\[
X_{FC}^{out}[a] = \sum_{p}X^{out}_{FC(p)}[a,p]
\]

For the reference we gather the forward signal flow equation with the backward gradient flow equation:
\begin{equation}
\left[X^{out}[a,p] \xleftarrow{+} A[a,a',p,q]X^{in}[a',p,q]\right]
\lra 
\left[\ov{X}^{in}[a',p,q]\xleftarrow{+}A[a,a',p,q]\ov{X}^{out}[a,p]\right]
\end{equation}

\subsubsection{Convolution unit and its gradient flow}

While for the full connection unit the order of axes in tensor views is irrelevant for efficiency of computing for both, the forward signal flow and backward error propagation, in case of convolution of the tensor $X^{in}$ with the parameter tensor $W$, the order for signal and attribute axes, e.g. $x,y,z,a^{out},a^{in}$ is important if only we want to minimize the number of index operations. Therefore we assume that attribute axis precedes the signal axes. We begin from the generalization of tensor convolution definition \eqref{eq:conv-xd} presented in the vector (buffer) form . Next we develop the procedure for error gradients wrt to signal and parameter data of the convolution unit.

\paragraph{General definition of (transposed) convolution\\}

According our notation for buffers created and buffers reference, the filter result is denoted by $X^{out}$, and the convolution input is $X^{in}.$ Moreover, according the common practice, the filter with kernel $J_s$ is generalized here by the sub-sampling rate $\sigma=k\geq 1, k\in\bb{N}$ or the over-sampling rate $\sigma=\frac{1}{k}$, and the dilation factor $\delta \geq 1, \delta\in \bb{N}$.  If the over-sampling is used then the convolution is called {\it fractionally strided convolution} or {\it transposed convolution} -- the name used in auto-encoders of image compression algorithms where over-sampling rate of the transposed convolution unit in decoder is the inverse of sub-sampling rate in the corresponding convolution unit.

The general formula (without batch leading index) is defined in the signal domain $I_s^{in}$ of the input tensor. We assume here that for optional preserving of signal resolution, the necessary padding is included in the input tensor domain.
\begin{equation}\label{eq:conv-gen}
\begin{array}{l}\ds
X^{out}[g,p] \ass B[g]+\sum_{f\in[N_a^{in}]}\sum_{q\in J_s} W[g,f,q]\cdot
\left\{ 
\begin{array}{ll}
X^{in}[f,k\cdot p+\delta\cdot q] & \text{if }\ \sigma=k,\\[2pt]
\ds X^{in}\left[f,\frac{p+\delta\cdot q}{k}\right] & \text{if }\ \sigma=\frac{1}{k},\ k|(p+\delta\cdot q),\\[2pt]
0 & \text{if }\ \sigma=\frac{1}{k},\ k\not|(p+\delta\cdot q),
\end{array}
\right.
\\[30pt] 
g\in [N_a^{out}], \forall q\in J_s\left[ p\in I_s^{out} \lra 
\left\{
\begin{array}{ll}
(k\cdot p+\delta\cdot q)\in I_s^{in} & \text{if }\ \sigma=k,\\[2pt]
(p+\delta\cdot q)/k \in I_s^{in} & \text{if }\ \sigma=\frac{1}{k},\ k|(p+\delta\cdot q).
\end{array}
\right.
\right]
\end{array}
\end{equation}
The number of parameters  $\#\bb{C}$ for any convolution layer depends on number of features $N_a^{in}, N^{out}$ and the number of kernel weights $n_s$ in the  spatial domain:
\begin{equation}
\#\bb{C} = \left( 1+n_s\cdot N_a^{in}\right)\cdot N_a^{out}
\end{equation}

\paragraph{Resolution change at (transposed) convolution -- general formulas\\}

Let generalize the parameters of (transposed) convolution by letting dependence of (up-) sub-sampling ratio $k$ and dilation factor $\delta$ on $d$ signal axes. Then $k_s,\delta_s\in\bb{N}^d$ like axes resolutions $N_s,  n_s\in\bb{N}^d$ for data and kernel tensors.

Then the basic constraints for the above parameters get the form of the following inequalities and equivalent equalities:
\begin{equation}\label{eq:conv-ineq}
\begin{array}{ll}
\bb{C}: & \left(N_s^{out}-1\right)\cdot k_s+ (n_s-1)\cdot\delta_s\leq N_s^{in}-1 \lra 
\left\{
\begin{array}{l}
\left(N_s^{out}-1\right)\cdot k_s+\ov{N}_s^{in}\bmod k_s = \ov{N}_s^{in}\\
\text{where }   \ov{N}_s^{in} \eqd N_s^{in}-(n_s-1)\cdot\delta_s-1
\end{array}
\right.
\\[15pt]
_t\bb{C}: & \left(N_s^{in}-1\right)\cdot k_s+ (n_s-1)\cdot\delta_s\leq N_s^{out}-1 \lra 
\left\{
\begin{array}{l}
\left(N_s^{in}-1\right)\cdot k_s+\ov{N}_s^{out}\bmod k_s = \ov{N}_s^{out}\\
\text{where }   \ov{N}_s^{out} \eqd N_s^{out}-(n_s-1)\cdot\delta_s-1
\end{array}
\right. 
\end{array}
\end{equation}

The above equalities can be resolved -- uniquely in case of convolution and by $k_s$ solutions in case of transposed convolution:
\begin{equation}\label{eq:conv-res}
\begin{array}{ll}
\bb{C}:  & \ds N_s^{out} =1+\frac{\ov{N}_s^{in}-\ov{N}_s^{in} \bmod k_s}{k_s}\\[15pt]
_t\bb{C}: & N_s^{out}=\left(N_s^{in}-1\right)\cdot k_s+ 1+(n_s-1)\cdot\delta_s+r^{\ast},\ \text{ where } r^{\ast}\in[0,k_s)^d
\end{array}
\end{equation}

\paragraph{AS-let block -- input/output compatibility\\}

There is interesting aspect of correspondence of convolutions and their transposed counterparts. Such convolution corresponding is natural in encoding/decoding applications, e.g. compression, steganographic embeddings, and in analysis/synthesis pairs responsible for the same level of signal resolution. The latter application type exhibits the direct chain (called as AS-let) of convolution unit $\cl{C}$ and transposed unit $\cl{C}'$ with the same corresponding down-sampling and up-sampling rate $k_s$, the same kernel size $n_s$, the same dilation factors $\delta_s$, and  with invariant I/O resolution in the signal domain:  
\begin{equation}\label{eq:as-inv}
Y\ass \cl{C}_{down}(X),\ Z \ass \cl{C}_{up}(Y,Y')  \lra N_s(Z) = N_s(X)
\end{equation}
where the optional tensor $Y'$ has the same signal resolution as the tensor $Y$ and it is stacked with $Y$ along the feature axis.

\begin{theorem}[Necessary and sufficient condition of resolution invariance for AS-let] 
Let $X,Y,Z$ be the tensors processed in the AS-let as defined in \eqref{eq:as-inv}. Let $\ov{N}_s \eqd N_s(X)- (n_s-1)\cdot\delta_s-1$ and $r^{\ast}$ be the term fixing the resulting resolution of transposed convolution as defined in \eqref{eq:conv-res}. Then $N_s(Z)=N_s(X)$ if and only if $r^{\ast}=\ov{N}_s\bmod  k_s.$
\end{theorem}
\begin{proof}
By the formulas \eqref{eq:conv-res} we get
$
N_s(Z) = \left(N_s(Y)-1\right)\cdot k_s+ 1+(n_s-1)\cdot\delta_s+r^{\ast}
$
On the other hand  by the equation \eqref{eq:conv-ineq}, $\left(N_s(Y)-1\right)\cdot k_s=\ov{N}_s(X)-\ov{N}_s\bmod k_s$. Hence:
\[
\begin{array}{rcl}
N_s(Z) &=& \ov{N}_s-\ov{N}_s\bmod k_s+1+(n_s-1)\cdot\delta_s+r^{\ast} =\\
&=& N_s(X)- (n_s-1)\cdot\delta_s-1 +1+(n_s-1)\cdot\delta_s+r^{\ast} -\ov{N}_s\bmod k_s\\
&=& N_s(X) + r^{\ast} -\ov{N}_s\bmod k_s 
\end{array}
\]
Therefore $N_s(Z)=N_s(X)$ if and only if $r^{\ast}=\ov{N}_s\bmod  k_s.$
\end{proof}

\paragraph{Resolution change at (transposed) convolution -- examples\\}

The resolution of the filtered signal depends on the resolution of the unit's input signal, the kernel resolution, and the parameters $\sigma,$ and $\delta.$ For instance for 3D signal domain, assuming the zero based indexing, we get the following formulas for the signal domain resolution $N_z^{out}\times N_y^{out}\times N_x^{out}$ of the filtered tensor wrt to the resolution of the input tensor $N_z^{in}\times N_y^{in}\times N_x^{in}$  and the kernel resolution  $[n_z]\times [n_y]\times [n_x]:$
\begin{equation}
\begin{array}{l}
I_s^{in} = [N_z^{in}]\times[N_y^{in}]\times[N_x^{in}], J_s=[n_z]\times[n_y]\times[n_z] \lra
I_s^{out} = [N_z^{out}]\times[N_y^{out}]\times[N_x^{out}] \\[5pt] 
\text{where}\ 
\left.
\begin{array}{l}
N_x^{out} = 1+\left\lfloor\frac{N_x^{in}-\delta(n_x-1)-1}{k}\right\rfloor\\[5pt]
N_y^{out} = 1+\left\lfloor\frac{N_y^{in}-\delta(n_y-1)-1}{k}\right\rfloor\\[5pt]
N_z^{out} = 1+\left\lfloor\frac{N_z^{in}-\delta(n_z-1)-1}{k}\right\rfloor
\end{array}
\right\}\ \text{if }\ \sigma=k,\ \ 
\left.
\right.
\\[35pt]
\text{and}\ 
\left.
\begin{array}{l}
N_x^{out} = k\cdot (N_x^{in}-1)+\delta(n_x-1)+1+\left(N_x^{in}-1-\delta(n_x-1)\right)\bmod  k\\
N_y^{out} = k\cdot (N_y^{in}-1)+\delta(n_y-1)+1+\left(N_y^{in}-1-\delta(n_y-1)\right)\bmod  k\\
N_z^{out} = k\cdot (N_z^{in}-1)+\delta(n_z-1)+1+\left(N_z^{in}-1-\delta(n_z-1)\right)\bmod  k
\end{array}
\right\} \text{if }\ \sigma=\frac{1}{k}\ .
\end{array}
\end{equation}
The resolutions for signal axes are stacked into the resolution vectors $N_s^{in}, N_s^{out}$ which for 3D signal domain are defined as follows:
$
N_s^{in}\eqd[N_z^{in},N_y^{in},N_x^{in}],\ \ \ N_s^{out}\eqd[N_z^{out},N_y^{out},N_x^{out}]\ .
$

\paragraph{Padding to match output resolution to sampling rate\\}

If no sampling is used ($\sigma=1$) then we frequently want to keep the same resolution of convolution input and output. To this goal the input tensor is padded in the signal domain. Without dilation the padding along the signal axis with kernel resolution $n$ equals to $p=n-1,$ otherwise the formula is less obvious. Moreover, even if make sub-sampling by $\sigma=k$ or over-sampling by $\sigma=\frac{1}{k}$ then we want to scale the resolution proportionally, i.e. reduce or enlarge $k$ times the  resolution of each signal axis. Having $p$ value as even makes the symmetrical padding possible at both ends of axis samples. All this requirements can be satisfied for $n$ odd (what is a natural choice) as it is stated in the following proposition.

\begin{proposition}[padding for (transposed) convolution to match resolution and sampling rate]
Given the resolution $N$ of signal domain axis and the convolution kernel resolution $n$ at the same axis. Let the sampling rate $\sigma=k\in\bb{N}$ for sub-sampling or $\sigma=\frac{1}{k}$ for over-sampling with the dilation factor $\delta\in\bb{N}$. Then, we are looking for $p\in\bb{N}$ satisfying one of the following conditions: 
\[
\begin{array}{ll}
\left\lfloor\frac{N-1}{k}\right\rfloor = \left\lfloor\frac{p+N-\delta(n-1)-1}{k}\right\rfloor & \text{if }\ \sigma=k\\[5pt]
kN = p+kN-\delta(n-1) & \text{if }\ \sigma=\frac{1}{k}
\end{array}
\]
Let $p_0\eqd\delta(n-1).$ Then the admissible padding always exists. Moreover, 
\begin{enumerate}
\item for no sampling and over-sampling cases: the unique admissible padding $p=p_0,$
\item for sub-sampling case: the admissible padding
$p\in[p_0',p_0'+k),$ with $p_0'\eqd p_0-(N-1)\hspace*{-2mm}\mod k,$   and there exists the even padding $p$ in this interval,
\item if $n$ is odd then for no sampling and over-sampling cases there exists the even admissible padding $p$,
\item if $n$ is even then for no sampling and over-sampling cases the even admissible padding exists iff the dilation $\delta$ is even number. 
\end{enumerate}
$\Box$
\end{proposition}
\begin{proof}
Perhaps the sub-sampling case ($k>1$) needs some guidance.
\begin{enumerate}
\item The condition to preserve the resolution   $1+\left\lfloor\frac{N-1}{k}\right\rfloor$ of sub-sampled tensor by the convolution of the original padded tensor followed by the sub-sampling:
\[
\left\lfloor\frac{N-1}{k}\right\rfloor = \left\lfloor\frac{N-1}{k}+\frac{p-p_0}{k}\right\rfloor,\quad \text{where }\ p_0\eqd (n-1)\delta\ .
\]
\item the equivalent inequalities:
$
-\frac{N-1}{k}\hspace*{-2mm}\mod 1\leq \frac{p-p_0}{k}<1-\frac{N-1}{k}\hspace*{-2mm}\mod 1
$
\item the interesting general equality:
$
\frac{N-1}{k}\hspace*{-2mm}\mod 1 = \frac{(N-1)\hspace*{-2mm}\mod k}{k}
$
\item the implications of the last two points:
\begin{enumerate}
\item
$
-\frac{(N-1)\hspace*{-2mm}\mod k}{k}\leq \frac{p-p_0}{k}<1-\frac{(N-1)\hspace*{-2mm}\mod k}{k}
$
\item 
$
p_0-(N-1)\hspace*{-2mm}\mod k\leq p<k+p_0-(N-1)\hspace*{-2mm}\mod k
$
\item 
$
p_0'\leq p<k+p_0'
$\ 
where $p_0'\eqd p_0-(N-1)\hspace*{-2mm}\mod k = \delta(n-1)-(N-1)\hspace*{-2mm}\mod k\ ,$
\end{enumerate}
\item if $p_0'\geq 0$ then either $p=p_0'$ or $p=p_0'+1$ is the even padding as $k\geq 2$,
\item if $p_0<0$ then no padding is needed as $p=0$ belongs to the interval of admissible paddings as $p_0'+k=p_0+k-(N-1)\hspace*{-2mm}\mod k\geq (n-1)\delta+1\geq 1$.
\end{enumerate}
\end{proof}
\paragraph{Efficient computing of (transposed) convolution\\}

Contrary to full connection unit the solutions for convolution unit to be efficient  are not general wrt to the signal domain. It means that we write separate code for 1D, 2D, and 3D cases. The main idea of algorithm acceleration is using of Horner's scheme to handle strides for computing addresses in the buffers. 
\begin{description}
\item[1D case:]$ $\\
For $g\in[N_a],\ x\in[N_x]$ the convolution 1D is computed as follows in tensor notation:
\[
\hspace*{-3mm}
\begin{array}{rcl}
\ds
X^{out}[g,x] &=& \ds\sum_{f\in[N_a^{in}]}\sum_{u\in[n_u]}W[g,f,u]\cdot
\left\{
\begin{array}{ll}
X^{in}[f,k\cdot x+\delta\cdot u] & \text{if }\ \sigma=k\\
X^{in}[f,(x+\delta\cdot u)/k] & \text{if }\ \sigma=\frac{1}{k},\ k|(x+\delta\cdot u)\\
0 & \text{if }\ \sigma=\frac{1}{k},\ k\not|(x+\delta\cdot u)
\end{array}
\right.
\end{array}
\]
It is implemented by changing indexes to addresses by Horner scheme. Then we refer directly to the data buffers ${_vX^{in}}$ (input signal), ${_vW}$ (kernel), ${_vX^{out}}$ (filtered signal).
\[
\begin{array}{rcl}
{_vX^{out}}[gN_x^{out}+x] &=& \ds\sum_{f\in[N_a^{in}]}\sum_{u\in[n_u]}{_vW}[(gN_a^{in}+f)n_u+u]\cdot\\[10pt]
&&
\left\{
\begin{array}{ll}
{_vX^{in}}[fN_x^{in}+k\cdot x+\delta\cdot u] & \text{if }\ \sigma=k\\
{_vX^{in}}[fN_x^{in}+(x+\delta\cdot u)/k] & \text{if }\ \sigma=\frac{1}{k},\ k|(x+\delta\cdot u)\\
0 & \text{if }\ \sigma=\frac{1}{k},\ k\not|(x+\delta\cdot u)
\end{array}
\right.
\end{array}
\]

The pseudo code for 1D convolution (sampling case) is obvious, but we need it to illustrate the construction of code for the gradient flow:
\begin{small}
\[
\begin{array}{l}
{_vX^{out}}\ass \bm{0}\\
for\ \ g\in[N_a^{out}]:\ \ for\ \ x\in[N_x^{out}]:\ \
for\ \ f\in[N_a^{in}]:\ for\ \ u\in[n_u]:\\
\hspace*{5mm}
{_vX^{out}}[g\cdot N_x^{out}+x] \xleftarrow{+}
{_vW}[(g\cdot N_a^{in}+f)\cdot n_u+u]\cdot
{_vX^{in}}[fN_x^{in}+k\cdot x+\delta\cdot u]
\end{array}
\]
\end{small}
An optimization of the above code via factoring common expressions out of the loops is obvious, too. However such code refactoring makes the code conversion to gradient flow less clear. According to the proposition \eqref{eq:single-line}, the pseudo code for gradient flow has the form for the  sub-sampling case:
\begin{small}
\[
\begin{array}{l}
{_v\ov{X}^{in}}\ass \bm{0}\\
for\ \ g\in[N_a^{out}]:\ \ for\ \ x\in[N_x^{out}]:\ \
for\ \ f\in[N_a^{in}]:\ for\ \ u\in[n_u]:\\
\hspace*{5mm}
{_v\ov{X}^{in}}[fN_x^{in}+k\cdot x+\delta\cdot u]
\xleftarrow{+}
{_vW}[(g\cdot N_a^{in}+f)\cdot n_u+u]\cdot
{_v\ov{X}^{out}}[g\cdot N_x^{out}+x]
\end{array}
\]
\end{small}

The over-sampling case add the conditional line in the innermost loop:
\begin{itemize}
\item pseudo code for 1D convolution\footnote{The function divmode$(a,b)$ returns integer quotient ant the remainder of the division of the integer $a$ by the integer $b.$}
 -- over-sampling case ($\sigma=\frac{1}{k}$):
\begin{small}
\[
\begin{array}{l}
{_vX^{out}}\ass \bm{0}\\
for\ \ g\in[N_a^{out}]:\ \ for\ \ x\in[N_x^{out}]:\ \
for\ \ f\in[N_a^{in}]:\ for\ \ u\in[n_u]:\\
\hspace*{5mm}
q,r \ass \text{divmod}(x+\delta\cdot u, k);\ if\ \ r\neq 0:\ continue\\
\hspace*{5mm}
{_vX^{out}}[g\cdot N_x^{out}+x] \xleftarrow{+}
{_vW}[(g\cdot N_a^{in}+f)\cdot n_u+u]\cdot
{_vX^{in}}[fN_x^{in}+q]
\end{array}
\]
\end{small}
\item pseudo code for the gradient flow of 1D convolution -- over-sampling case ($\sigma=\frac{1}{k}$):
\begin{small}
\begin{equation}\label{code:grad-flow-conv}
\begin{array}{l}
{_v\ov{X}^{in}}\ass \bm{0}\\
for\ \ g\in[N_a^{out}]:\ \ for\ \ x\in[N_x^{out}]:\ \
for\ \ f\in[N_a^{in}]:\ for\ \ u\in[n_u]:\\
\hspace*{5mm}
q,r \ass \text{divmod}(x+\delta\cdot u, k);\ if\ \ r\neq 0:\ continue\\
\hspace*{5mm}
{_v\ov{X}^{in}}[fN_x^{in}+q]
\xleftarrow{+}
{_vW}[(g\cdot N_a^{in}+f)\cdot n_u+u]\cdot
{_v\ov{X}^{out}}[g\cdot N_x^{out}+x]
\end{array}
\end{equation}
\end{small}

\end{itemize}

\item[2D case:]$ $\\
For $g\in[N_a^{out}],\ y\in[N_y^{out}],\ x\in[N_x^{out}]$ the (transposed) convolution 2D with sampling rate $\sigma$ and kernel dilation $\delta$ in the tensor notation has the form:
\[
\begin{array}{rcl}
\ds
X^{out}[g,y,x] &=& \sum_{f\in[N_a^{in}]}\sum_{v\in[n_v]}\sum_{u\in[n_u]}W[g,f,v,u]\cdot\\[10pt]&&
\left\{
\begin{array}{ll}
X^{in}[f, k\cdot y+\delta\cdot v, k\cdot x+\delta\cdot u] &
\text{if $\sigma=k,$}\\
X^{in}[f, (y+\delta\cdot v)/k, (x+\delta\cdot u)/k] &
\text{if $\sigma=\frac{1}{k}$ and}\\
&\quad k|(y+\delta\cdot v), k|(x+\delta\cdot u),\\
0&\text{otherwise.}
\end{array}
\right.
\end{array}
\]

In 2D case the benefit of Horner's scheme in converting indexes to addresses is more visible. 
\[
\begin{array}{l}
{_vX^{out}}[(gN_y^{out}+y)N_x^{out}+x] =  \ds
\sum_{f\in[N_a^{in}]}\sum_{v\in[n_v]}\sum_{u\in[n_u]}{_vW}[((gN_a^{in}+f)n_v+v)n_u+u]\cdot\\[15pt]
\left\{
\begin{array}{ll}  
 {_vX^{in}}[(fN_y^{in}+k\cdot y+\delta\cdot v)N_x^{in}+k\cdot x+\delta\cdot u]&\text{if $\sigma=k,$}\\
{_vX^{in}}[(fN_y^{in}+(y+\delta\cdot v)/k)N_x^{in}+(x+\delta\cdot u)/k]&
\text{if $\sigma=\frac{1}{k}$ and}\\
&\quad k|(y+\delta\cdot v), k|(x+\delta\cdot u),\\
0&\text{otherwise.} 
\end{array}
\right.
\end{array}
\]
The pseudo codes for 2D convolution and its gradient flow are built as for 1D case with adding two loops more (over $y$ and over $v$).

The 3D case is the straightforward generalization of 2D case: the loops over $z$ and $w$ are just added.

\end{description}

\paragraph{Gradient wrt to convolutional parameters\\}

We conclude this section with error $\cl{E}$ gradients formulas wrt parameters $W$ and $B$ of generalized (transposed) convolution unit. 

Looking for $\sp{\cl{E}}{W[g,f,q]}$ we should consider all equations for output signals $X^{out}[g,p]$ with contribution of this parameter. Let $P_q$ be the set of such equations at fixed $g$:
\[
P_q\eqd\{p\in I_s^{out}: \exists r\in I_s^{in}, \lfloor\sigma\cdot p+\delta\cdot q\rfloor=r\}
\]
According the definition of generalized (transposed) convolution \eqref{eq:conv-gen}, for any $q\in J_s$ all elements in $I_s^{out}$ satisfy the above condition, i.e.
\[
P_q = I_s^{out},\ \text{ for any }\ \ q\in J_s\ .
\]
Hence
\[
\begin{array}{rcl}
\sp{\cl{E}}{W[g,f,q]} &=& \ds\sum_{p\in I_s^{out}}\sp{\cl{E}}{X^{out}[g,p]}\cdot \sp{X^{out}[g,p]}{W[g,f,q]} = \sum_{p\in I_s^{out}}\ov{X}^{out}[g,p]
\cdot\\[10pt]&&
\left\{
\begin{array}{ll}
X^{in}[f,k\cdot p+\delta\cdot q] & \text{if }\ \sigma=k,\\
X^{in}[f,(p+\delta\cdot q)/k] & \text{if }\ \sigma=\frac{1}{k},\ k|(p+\delta\cdot q),\\
0 & \text{if }\ \sigma=\frac{1}{k},\ k\not|(p+\delta\cdot q).
\end{array}
\right.
\end{array}
\]
In terms of tensor operation it is a kind of scalar product of the output gradient's $g$ slice along $p$ axes with the $f$ slice of the input tensor along the mapped $r$ axes. In total we get a family of Gram's matrices indexed by $q\in J_s.$ For the gradient wrt parameters $B$ we get
\[
\sp{\cl{E}}{B[g]} = \sum_{p\in I_s^{out}}\sp{\cl{E}}{X^{out}[g,p]}\cdot \sp{X^{out}[g,p]}{B[g]} = \sum_{p\in I_s^{out}}\ov{X}^{out}[g,p]\cdot 1
\]
Now, the value is the sum of $g$ slice elements taken along $p$ axes.

\begin{proposition}[gradients of generalized (transposed) convolution wrt parameters]
Given the gradient $\od{\cl{E}}{X^{out}}$ the gradients wrt parameters $W$ and $B$ for the generalized (transposed) convolution are given by the formulas\footnote{Divisions by $k$ in the formulas are all component-wise.}:
\begin{equation}\label{eq:conv-grad-wb}
\begin{array}{rcl}
\od{\cl{E}}{W}[g,f,q] & = & \ds\sum_{p\in I_s^{out}}\ov{X}^{out}[g,p]\cdot 
\left\{
\begin{array}{ll}
X^{in}[f,k\cdot p+\delta\cdot q] & \text{if }\ \sigma=k,\\
X^{in}[f,(p+\delta\cdot q)/k] & \text{if }\ \sigma=\frac{1}{k},\ k|(p+\delta\cdot q),\\
0 & \text{if }\ \sigma=\frac{1}{k},\ k\not|(p+\delta\cdot q).
\end{array}
\right.
\\[15pt]
\od{\cl{E}}{B}[g] & = & \ds\sum_{p\in I_s^{out}}\ov{X}^{out}[g,p]
\end{array}
\end{equation}
\end{proposition}

\subsubsection{Pooling unit and its gradient flow}

Pooling unit aims to reduce the tensor resolution in its signal domain. The resolution is reduced by aggregation of tensor values in  blocks of size $k$ which are selected with stride $\sigma$. Typically $k$ and $\sigma$ are equal for each of $d$ signal axis. However, there are applications where they depend on axes\footnote{For instance spatial pyramid pooling.} . Therefore, we assume here the general case\: $k\in\bb{N}^d,$ $\sigma\in\bb{N}^d.$
Moreover, we assume that pooling is performed in full blocks and padding is not applied to make the aggregated block complete. Hence for some combination of the signal domain resolution $N_s^{in}\in\bb{N}^d$ of the input tensor $X^{in}$ and the hyper-parameters $k,\sigma$  it may happen that a small margin will be ignored and in the gradient flow it will be filled by zeros.

For the above assumptions the signal domain resolution $N_s^{out}\in\bb{N}^d$ of the output tensor $X^{out}$ can be computed by indexing of blocks $Q_p\subset[N_s^{in}]$:
\begin{equation}
Q_p\eqd \{q\in[N_s^{in}]: \exists j\in[k], q=\sigma\cdot p+j\},\ p\in[N_s^{out}]
\end{equation}
In the formula for $q$, the multiplication of index vectors is performed component-wise.

Let us consider the last block $Q_p$ in the lexicographic order. Then its last element $q=\sigma\cdot p+k-1_d.$ Since $q\in[N_s^{in}]$, we get the bounds for $p$:
\[
\sigma\cdot p+k-1_d<N_s^{in} \lra 
0_d\leq p < \left\lceil\frac{N_s^{in}-k+1_d}{\sigma}\right\rceil
\]
where the division of index vectors and ceiling function are component-wise, as well.

Hence, we conclude the resolution formula in the signal domain of the output tensor for the pooling units:
\begin{equation}\label{eq:pool-res}
\begin{array}{l}
N_s^{out} = \left\lceil\frac{N_s^{in}-k+1_d}{\sigma}\right\rceil,\quad\quad
\left[\sigma=k \lra N_s^{out} = \left\lfloor\frac{N_s^{in}}{k}\right\rfloor\right]
\end{array}
\end{equation}
The above special case follows from the arithmetic identity for ceiling and floor functions: 
\[
\lceil (a+1)/b\rceil-1= \lfloor a/b\rfloor\ \text{ for any }\  a,b\in\bb{N}.
\]

 There are two types of pooling used in digital media applications: {\it max pooling} and {\it average pooling} with options {\it sliced pooling and global pooling.}

\begin{definition}[Max pooling]
If $X^{in}\in\bb{R}^{[N_a]\times [N_s^{in}]}$ then the pooled tensor $X^{out}\in\bb{R}^{[N_a]\times [N_s^{out}]}$ with signal domain resolution specified by \eqref{eq:pool-res} with content defined for any $a\in[N_a], p\in[N_s^{out}]:$
\begin{equation}\label{eq:max-pool}
\ds X^{out}[a,p]\eqd \max_{q\in Q_p} X^{in}[a,q],\ Arg[a,p]=\arg\max_{q\in Q_p}X^{in}[a,q]
\end{equation}
\end{definition}

In order to compute the gradient flow let us observe that:
\begin{itemize}
  \item the element $X^{out}[a,p]$ depends only on the element $X^{in}[a,Arg[a,p]]$,
  \item the element $X^{in}[a,q]$ could be the source for none, one or more elements $X^{out}[a,p]$ if only $Arg[a,p]=q.$
\end{itemize}
Therefore if we initially set the gradient tensor $\ov{X}^{in}$ to zero the gradient flow is implemented  by the trivial loop:
\begin{equation}\label{eq:max-pool-grad}
for\ \ a,p\in[N_a]\times[N_s^{out}]:\ \ov{X}^{in}[a,Arg[a,p]] \xleftarrow{+} \ov{X}^{out}[a,p]
\end{equation}

\begin{definition}[Average pooling]
Let $|k|$ be the number of elements in the block $Q_p$. 
If $X^{in}\in\bb{R}^{[N_a]\times [N_s^{in}]}$ then the pooled tensor $X^{out}\in\bb{R}^{[N_a]\times [N_s^{out}]}$ with signal domain resolution specified by \eqref{eq:pool-res} with content defined for any $a\in[N_a], p\in[N_s^{out}]:$
\begin{equation}\label{eq:avg-pool}
X^{out}[a,p]\eqd \frac{1}{|k|}\sum_{q\in Q_p} X^{in}[a,q]
\end{equation}
\end{definition}

Since average pooling is the linear operation we can use the rule \eqref{eq:single-line} for the gradient flow.Then after filling the tensor $\ov{X}^{in}$ by zeros, the loop has the form:
\begin{equation}\label{eq:avg-pool-grad}
\begin{array}{l}
for\ \ a,p\in[N_a]\times[N_s^{out}]:\ 
for\ \ q\in Q_p:\ 
\ov{X}^{in}[a,q] \xleftarrow{+} \ov{X}^{out}[a,p]\\[5pt]
\ov{X}^{in}\xleftarrow{\text{\tiny$/$}}|k|
\end{array}
\end{equation}

\begin{definition}[Slice pooling and global pooling]
Pooling by slices is special case of average pooling:
\begin{enumerate}
  \item {\it pooling along $x$ axes}: $k_x=N_x^{in},$ $k_y$ and remaining axes (if any) equal to one,
  \item {\it pooling along $y$ axes}: $k_y=N_y^{in},$ $k_x$ and remaining axes (if any) equal to one,
  \item {\it pooling along both $xy$ axes}: $k_x=N_x^{in},$ $k_y=N_y^{in},$  and remaining axes (if any) equal to one,
  \item {\it global pooling} along all spatial axes: $k_x=N_x^{in}$ and any other spatial axis $i$ gets $k_i=N_i^{in}.$
\end{enumerate}
In all cases strides $\sigma\eqd k$. Moreover, all axes in the tensor $X^{out}$ with resolution one, are removed from its tensor view.
\end{definition}

\subsubsection{Decimation/interpolation unit and its gradient flow}

Despite the sub-sampling and over-sampling options for (transposed) convolution and pooling units there are applications for which the primitive over-sampling via copying or simple sub-sampling via selection by regular striding, are not satisfactory\footnote{for instance in case of compression via auto-encoder models.}. We adopt here to this goal the general decimation/interpolation scheme with a continuous kernel $K:\bb{R}\times\bb{R}\lra\bb{R}$ applied in discrete neighborhoods $\cl{N}_r$ with the fixed radius wrt a norm -- typically it is the Minkowski's norm $\|\cdot\|_m$  for $m=1,2,\infty.$ Namely for $\tilde{p}\in\bb{R}$ we have:
\[
\cl{N}_r(\tilde{p}) \eqd \{q\in[N_s^{in}]:\ \|q-\tilde{p}\|<r\}
\]
Interestingly, the neighborhood $\cl{N}_2(p)$ with $p$ in the discrete grid and for the norm $\|\cdot\|_1$ is aka $4$-neighborhood while for the norm $\|\cdot\|_{\infty}$ is aka $8$-neighborhood. 

Let $\sigma\eqd N_s^{in}/N_s^{out}$ be the vector in $\bb{R}^d$ of scalings from the input resolution in the signal domain to the desired output resolution in the signal domain, as well. Let $\Sigma\inm{d}{d}$ be the diagonal matrix with the vector $\sigma$ in its diagonal. 

If $\sigma_i>1$ then we deal with decimation while for $\sigma_i<1$, the interpolation is performed. Typically for $\sigma_i>1$ we take the radius $r\eqd\lceil\sigma\rceil.$ However, in case of interpolation $r$ is independent of $\sigma$ -- for instance $r=1,2,3.$

Having the kernel $K$ and the neighborhoods $\cl{N}_r(\cdot)$ we define the interpolation/decimation formula for any $a\in[N_a]$ and $p\in[N_s^{out}].$
\begin{equation}\label{eq:interp-nk}
X^{out}[a,p] \eqd 
\sum_{q\in\cl{N}_r(\Sigma^{-1}p)}
\overbrace{
\frac{K\left(q,\Sigma^{-1}p\right)}
{\ds\sum_{q'\in\cl{N}_r(\Sigma^{-1}p)}K\left(q',\Sigma^{-1}p\right)}}^{k_{pq}}
\cdot X^{in}[a,q]
\end{equation}

The above formula is linear wrt $X^{in}$. Therefore, the gradient flow through the dual interpolation unit has the form:
\begin{equation}\label{eq:interp-grad}
for\ \ a,p\in[N_a]\times[N_s^{out}]:\ 
for\ \ q\in\cl{N}_r(\Sigma^{-1}p):\ 
\ov{X}^{in}[a,q] \xleftarrow{+} k_{pq}\ov{X}^{out}[a,p]
\end{equation}
where
\[
k_{pq} \eqd
\frac{K\left(q,\Sigma^{-1}p\right)}
{\ds\sum_{q'\in\cl{N}_r(\Sigma^{-1}p)}K\left(q',\Sigma^{-1}p\right)}
\]

Let us consider two special cases of the interpolation formula \eqref{eq:interp-nk}.
\begin{description}
\item[Multi-linear interpolation:]$ $\\
The kernel $K(q,\tilde{q})$ is defined as the product of components similarities:
\[
K(q,\tilde{q}) \eqd \prod_{i\in[d]} (1-|q_i-\tilde{q}_i|).
\]
The neighborhoods $\cl{N}_1(\tilde{q})$ are defined wrt the infinity norm:
$
\cl{N}_1(\tilde{q}) \eqd \{q:\|q-\tilde{q}\|_{\infty}<1\}.
$

Note that the neighborhood cardinality is not constant: $1\leq|\cl{N}_1(\tilde{q})|\leq 2^d.$ Moreover, in the neighborhoods, the  kernel values play the role of weights summing up to one:
\[\sum_{q\in\cl{N}_1(\tilde{q})}K(q,\tilde{q})=1\ \text{ and }\ k_{pq} = K(q,\Sigma^{-1}p)\]

\item [Gaussian decimation/interpolation:]$ $\\
The kernel $K(q,\tilde{q})$ is defined now as the product of (non normalized) Gaussian values:
\[
K(q,\tilde{q}) \eqd \prod_{i\in[d]} e^{-(q_i-\tilde{q}_i)^2} = e^{-\|q-\tilde{q}\|^2_2}.
\]
The neighborhoods $\cl{N}_r(\tilde{q}),$ $r>1,$ are defined wrt the infinity norm:
\[
\cl{N}_r(\tilde{q}) \eqd \{q:\|q-\tilde{q}\|_{\infty}<r\}.
\]

Note that the upper bound for the neighborhood cardinality is 
$|\cl{N}_r(\tilde{q})|\leq (2r)^d$. The lower bound for points $\tilde{q}$ with distance from the the signal domain boundary greater than $r$ is achieved on the discrete grid of points and equals to $(2r-1)^d.$ The bounds for neighborhoods cardinalities give an insight on time complexity of decimation/interpolation units as the function of signal domain dimensionality.
\end{description}

\subsubsection{Batch normalization and its gradient flow}

The batch normalization unit performs a kind of data post-processing   for convolution and full connection units. Mathematically it works independently for each feature component $f\in[N_a]$ of their output tensors across the $N_b\cdot|I_s|$ elements of batch and signal domain changing the statistics of such sequences to the zero mean and the unit variance by the standard formula $\tilde{x}\eqd (x-\bar{x})/\sigma_x$, and next restoring  the second order statistics by the trained affine combination: $\beta\tilde{x}+\gamma$.

There are several solid premises\footnote{The batch normalization is a heuristic procedure and the enumerated premises are a kind of justification for using this heuristic.} to do batch normalization:
\begin{enumerate}
  \item The outputs of convolution and full connection as summations of features tend to be drawn from a Gaussian distribution -- therefore preserving the second order statistics means preserving the statistical distributions,
  \item The normalization is an affine transform of batch data to make them on average close to zero (almost all data of any fixed component for the batch belongs now to the interval $[-3,+3]$).
  \item After de-normalization and ReLU operation almost all  feature points in the batch belong to the ball, in the infinity norm, of radius $3\beta/2$.
  \item When $\beta$ and $\gamma$ converges then the above conclusion is true for all data points generated for training data set at this batch normalization unit. This means that the next unit gets the data from a nearly fixed domain at each following training epoch.

\end{enumerate}

\begin{definition}[Batch normalization -- phase one]
Let $b\in[N_b]$, $a\in[N_a^{in}]$, and $p\in I_s^{in}$. Then
\begin{equation}\label{eq:bn-1}
\widetilde{X}^{in}[b,a,p]\eqd \frac{X^{in}[b,a,p]-M[a]}{D[a]} 
\eqd
\frac{X^{in}[b,a,p]-\overbrace{\frac{1}{N_{bs}^{in}}\cdot\sum_{b\in[N_b]}\sum_{p\in I_s}X^{in}[b,a,p]}^{M[a]}}
{\underbrace{\ds\sqrt{\frac{1}{N_{bs}^{in}}\cdot\sum_{b\in[N_b]}\sum_{p\in I_s}\left(X^{in}[b,a,p]\right)^2-M^2[a]+\epsilon}}_{D[a]}}
\end{equation}
where $N_{bs}^{in}\eqd N_b|I_s^{in}|$ and $\epsilon$ -- a small positive number to avoid divisions by zero, e.g. $\epsilon=10^{-7}.$
\end{definition}

\begin{definition}[Batch normalization -- phase two]
Let $b\in[N_b]$, $a\in[N_a^{in}]$, $p\in I_s^{in}$ and $\widetilde{X}^{in}[b,a,p]$ be the normalization tensor after the phase one.
Let $\beta[a]$ and $\gamma[a]$ be the trainable affine transformation coefficients. Then
\begin{equation}\label{eq:bn-2}
\begin{array}{rcl}
X^{out}[b,a,p] &\eqd&  \beta[a]+\gamma[a]\widetilde{X}^{in}[b,a,p]
\end{array}
\end{equation}
where $\epsilon$ -- a small positive number to avoid divisions by zero, e.g. $\epsilon=10^{-7}.$
\end{definition}

\begin{theorem}[Gradient flow for batch normalization]
The gradients wrt to affine coefficients $\beta,\gamma$ for the batch normalization unit are given by the formulas:
\begin{equation}\label{eq:bn-grad-1}
\sp{\cl{E}}{\beta[a]} = \sum_{b\in[N_b]}\sum_{p\in I_s}\ov{X}^{out}[b,a,p],\quad\quad 
\sp{\cl{E}}{\gamma[a]} = \sum_{b\in[N_b]}\sum_{p\in I_s}\widetilde{X}^{in}[b,a,p]\ov{X}^{out}[b,a,p].
\end{equation}
Let $b\in[N_b]$, $a\in[N_a^{in}]$, $p\in I_s^{in},$ $\gamma[a]$ be the trainable scaling coefficient, $\widetilde{X}^{in}[b,a,p]$ be the normalization tensor after phase one, and $D[a]$ be the standard deviation of batch samples in the tensor's feature component $a$ (used to compute the phase one normalization). Then
the gradient flow through the batch normalization unit can be determined as follows:
\begin{equation}\label{eq:bn-grad-2}
\ov{X}^{in}[c,a,q] = \frac{\gamma[a]}{D[a]}\cdot\ov{X}^{out}[c,a,q]
-\frac{1}{N_{bs}^{in}}\cdot\frac{\gamma[a]}{D[a]}\cdot
\left[\widetilde{X}^{in}[c,a,q]\cdot\sp{\cl{E}}{\gamma[a]}+\sp{\cl{E}}{\beta[a]}\right]
\end{equation}
\end{theorem}
\begin{proof} 
We begin from gradients wrt to $\beta$ and $\gamma$ as the proofs follow directly from the chain rule and the latter gradient is used to simplify the formula for the gradient wrt $X^{in}$.
\[
\begin{array}{ll}
\ds\sp{\cl{E}}{\beta[a]} = \sum_{b\in[N_b]}\sum_{p\in I_s} \sp{\cl{E}}{X^{out}[b,a,p]}\cdot\overbrace{\sp{X^{out}[b,a,p]}{\beta[a]}}^{1} =
\sum_{b\in[N_b]}\sum_{p\in I_s}\ov{X}^{out}[b,a,p]\\
\ds\sp{\cl{E}}{\gamma[a]} = \sum_{b\in[N_b]}\sum_{p\in I_s} \sp{\cl{E}}{X^{out}[b,a,p]}\cdot\underbrace{\sp{X^{out}[b,a,p]}{\gamma[a]}}_{\widetilde{X}^{in}[b,a,p]} =
\sum_{b\in[N_b]}\sum_{p\in I_s}\widetilde{X}^{in}[b,a,p]\ov{X}^{out}[b,a,p]
\end{array}
\]
\noindent By the definitions \eqref{eq:bn-1}, \eqref{eq:bn-2}:\\
$
X^{out}[b,a,p]= \beta[a]+\gamma[a]\cdot
\frac{\ds X^{in}[b,a,p]-\overbrace{\frac{1}{N_{bs}^{in}}\cdot\sum_{b\in[N_b]}\sum_{p\in I_s}X^{in}[b,a,p]}^{M[a]}}
{\underbrace{\ds\sqrt{\frac{1}{N_{bs}^{in}}\cdot\sum_{b\in[N_b]}\sum_{p\in I_s}\left(X^{in}[b,a,p]\right)^2-M^2[i]+\epsilon}}_{D[i]}}
$

\noindent Therefore for any $c\in[N_b], q\in I_s$:
\[
\begin{array}{rcl}
\ds\sp{X^{out}[b,a,p]}{X^{in}[c,a,q]} & = &
\gamma[i]\cdot
\ds\frac{
\left(\bm{1}_{bp=cq}-\frac{1}{N_{bs}^{in}}\right)\cdot D[a]-
(X^{in}[b,a,p]-M[a])\cdot
\sp{D[a]}{X^{in}[c,a,q]}
}{D^2[a]}\\[10pt]
&=& 
\frac{\gamma[a]}{N_{bs}^{in}D[a]}\cdot
\ds\left[
N_{bs}^{in}\cdot \bm{1}_{bp=cq}-1-
\frac{\left(X^{in}[b,a,p]-M[a]\right)}{D[a]}
\cdot\frac{\left(X^{in}[c,a,q]-M[a]\right)}{D[a]}
\right]
\end{array}
\]
where $
\bm{1}_{condition} = \left\{
\begin{array}{ll}
1 & \text{if $condition$ is true}\\
0 & \text{otherwise}.
\end{array}
\right.
$

\noindent By the definition \eqref{eq:bn-1}:
\[
\ds\sp{X^{out}[b,a,p]}{X^{in}[c,a,q]}  = 
\frac{\gamma[i]}{N_{bs}^{in}\cdot D[a]}\cdot
\ds
\left[
N_{bs}^{in}\cdot\bm{1}_{bp=cq}-1-\widetilde{X}^{in}[b,a,p]\widetilde{X}^{in}[c,a,q]
\right]
\]

\noindent Hence by the chain rule:
\[
\begin{array}{rcl}
\ov{X}^{in}[c,a,q] &\eqd&
\ds \sp{\cl{E}}{X^{in}[c,a,q]} = 
\sum_{b\in[N_b]}\sum_{p\in I_s}\sp{\cl{E}}{X^{out}[b,a,p]}
\cdot\sp{X^{out}[b,a,p]}{X^{in}[c,a,q]}\\[15pt]
&=&\ds
\sum_{b\in[N_b]}\sum_{p\in I_s}\ov{X}^{out}[b,a,p]
\cdot\sp{X^{out}[b,a,p]}{X^{in}[c,a,q]}
\end{array}
\]

Substitution of the partial derivatives $\sp{X^{out}[b,a,p]}{X^{in}[c,a,q]}$ leads to:
\[
\begin{array}{rcl}
\ov{X}^{in}[c,a,q] &\eqd&
\ds\frac{\gamma[a]}{N_bD[a]}\cdot
\sum_{b\in[N_b]}\sum_{p\in I_s}\left[N_{bs}^{in}\bm{1}_{bp=cq}-1-\widetilde{X}^{in}[c,a,q]
\widetilde{X}^{in}[b,a,p]\right]
\ov{X}^{out}[b,a,p]\\[15pt]
&=&\ds\frac{\gamma[a]}{N_{bs}^{in}D[a]}\cdot
\sum_{b\in[N_b]}\sum_{p\in I_s}\left[N_{bs}^{in}\cdot \bm{1}_{bp=cq}\cdot\ov{X}^{out}[b,a,p]-\ov{X}^{out}[b,a,p]\right.\\[15pt]
&&\quad\quad\quad\quad\quad\quad\quad\quad\quad\quad
\left.-\widetilde{X}^{in}[c,a,q]\cdot
\widetilde{X}^{in}[b,a,p]\cdot\ov{X}^{out}[b,a,p]\right]
\end{array}
\]

\noindent Using the computed gradient for $\gamma$ and $\beta$, we finally get the gradient flow equation:
\[
\begin{array}{rcl}
\ov{X}^{in}[c,a,q] &\eqd& 
\ds\frac{\gamma[a]}{N_{bs}^{in}D[a]}\cdot
\left[
\sum_{b\in[N_b]}\sum_{p\in I_s}N_{bs}^{in}\bm{1}_{bp=cq}\cdot\ov{X}^{out}[b,a,p]-\sum_{b\in[N_b]}\sum_{p\in I_s}\ov{X}^{out}[b,a,p]
\right.\\[5pt]
&&\ds\quad\quad\quad\quad\quad\quad\left.
-\sum_{b\in[N_b]}\sum_{p\in I_s}\widetilde{X}^{in}[c,a,q]\cdot
\widetilde{X}^{in}[b,a,p]\cdot\ov{X}^{out}[b,a,p]
\right]
\\[20pt]
&=&\ds\frac{\gamma[a]}{D[a]}\cdot\ov{X}^{out}[c,a,q]
-\frac{1}{N_{bs}^{in}}\cdot\frac{\gamma[a]}{D[a]}\cdot
\left[\widetilde{X}^{in}[c,a,q]\cdot\sp{\cl{E}}{\gamma[a]}+\sp{\cl{E}}{\beta[a]}\right]
\end{array}
\]
\end{proof}

From the form of gradient flow equation \eqref{eq:bn-grad-2} we conclude an interesting property of the batch unit normalization in the process of optimization for the generalized perceptron: when the stochastic optimizer is close to the local minimum of perceptron's loss function, the batch unit transfers the gradient with the same attenuation or amplification for all batch elements. 

The scaling coefficient $\gamma[a]/D[a]$ depends on batch and feature index $a.$  Since $\gamma[a]$ is, independent of batch, estimation of of standard deviation for feature component $a$ then the actual scaling of gradient depends on ratio of global estimation of standard deviation to the instant standard deviation computed for the batch.

\paragraph{Batch normalization in testing stage?}

It seems that batch normalization should be performed only during the training. In practice it means that despite of compensating character of affine operation determined by $\beta$ and $\gamma$ coefficients, the performance of testing could be inferior than the validation results. On the other hand, preserving this kind of normalization during testing requires estimation of statistics $M[a], D[a]$ wrt the whole training data set. It is performed incrementally by accumulating of batch statistics in the last epoch of the training stage for the mean of batch values and for the mean of their squares. 

If the $k$-th batch (of $K$ batches) has the mean of $N_{bs}^{in}$ input values $X^{in}[b,a,p]$ equal to $M_k[a]]$ and the mean for their squares equal to $S_k[a]$ then for any $a\in[N_a]:$
\begin{equation}
\begin{array}{l}
M[a] \ass 0,\quad S[a]\ass 0\\[5pt]
\text{ for } k<K\quad\quad\left\{
\begin{array}{l}
\ds M[a] \ass \frac{k\cdot M[a]+N_{bs}^{in}\cdot M_{k+1}[a]}{k+1}\\[10pt]
\ds S[a] \ass \frac{k\cdot S[a]+N_{bs}^{in}\cdot S_{k+1}[a]}{k+1}
\end{array}\right.\\[25pt]
D[a] \ass \sqrt{S[a]-M^2[a]+\epsilon}
\end{array}
\end{equation}

Obviously, preserving the batch normalization for testing, results in the overhead of memory for vectors $\beta,\gamma,M,D$ -- together $4N_a$ float numbers.  Hence, the memory overhead for parameters is negligible. If we make the affine operations component-wise in place then there is no memory overhead for the tensor itself.  However, there is slight time overhead. The small benefits of the performance are not usually in practice worth of those overheads.

\subsubsection{Instant normalization}

It was observed that for image compression and image color transfer the better results are achieved if the batch normalization is replaced by the instance normalization. It is similar to its batch counterpart only in the aspect of using the statistical normalization concept, i.e. the data centering by its mean and next division by the data standard deviation. The basic difference is in the way how the local data is collected around the current element of the input tensor:
\begin{itemize}
  \item for the batch normalization the data is collected along the batch axis giving $N_b$ data samples,
  \item for the instance normalization the data is selected for all signal axes and for batch axis\footnote{Instant normalization frequently is used when GAN architecture is used. Then the training is performed with no batches, i.e. $N_b=1.$}.
\end{itemize}

\begin{definition}[Instant normalization]
Let $b\in[N_b]$, $a\in[N_a]$, and $p\in I_s^{in}=I_s^{out}=I_s$. Then
\begin{equation}\label{eq:in-1}
\widetilde{X}^{in}[b,a,p]\eqd \frac{X^{in}[b,a,p]-M[b,a]}{D[b,a]} 
\eqd
\frac{X^{in}[b,a,p]-\overbrace{\frac{1}{N_s}\cdot\sum_{p\in I_s}X^{in}[b,a,p]}^{M[b,a]}}
{\underbrace{\ds\sqrt{\frac{1}{N_s}\cdot\sum_{p\in I_s}\left(X^{in}[b,a,p]\right)^2-M^2[b,a]+\epsilon}}_{D[b,a]}}
\end{equation}
where $N_s\eqd |I_s|$ and $\epsilon$ -- a small positive number to avoid divisions by zero, e.g. $\epsilon=10^{-7}.$
\end{definition}

\begin{theorem}[Gradient flow for instant normalization]
Let $b\in[N_b]$, $a\in[N_a]$, $p\in I_s,$ $\widetilde{X}^{in}[b,a,p]$ be the instantly normalized input tensor, and $D[b,a]$ be the standard deviation of batch samples in the signal domain (used to compute the instant normalization). Then
the gradient flow through the instant normalization unit 
\begin{equation}\label{eq:in-grad-1}
\begin{array}{rcl}
\ov{X}^{in}[b,a,q] 
&=&\ds\frac{1}{D[b,a]}\cdot
\left[
\ov{X}^{out}[b,a,q]
-\left(\frac{1}{N_s}\cdot\sum_{p\in I_s}\ov{X}^{out}[b,a,p]\right)
\right.\\[5pt]
&&\ds\quad\quad\quad\quad\left.
-\widetilde{X}^{in}[b,a,q]\cdot\left(\frac{1}{N_s}\cdot\sum_{p\in I_s}
\widetilde{X}^{in}[b,a,p]\cdot\ov{X}^{out}[b,a,p]\right)
\right]
\end{array}
\end{equation}
\end{theorem}
\begin{proof} 
\noindent By the definition \eqref{eq:in-1}:\quad
$
X^{out}[b,a,p]= 
\frac{\ds X^{in}[b,a,p]-\overbrace{\frac{1}{N_s}\cdot\sum_{p\in[N_s]}X^{in}[b,a,p]}^{M[b,a]}}
{\underbrace{\ds\sqrt{\frac{1}{N_s}\cdot\sum_{p\in I_s}\left(X^{in}[b,a,p]\right)^2-M^2[b,a]+\epsilon}}_{D[b,a]}}
$

\noindent Therefore for any $q\in I_s$:
\[
\begin{array}{l}
\ds\sp{X^{out}[b,a,p]}{X^{in}[b,a,q]} = 
\ds\frac{
\left(\bm{1}_{p=q}-\frac{1}{N_s}\right)\cdot D[b,a]-
(X^{in}[b,a,p]-M[b,a])\cdot
\sp{D[b,a]}{X^{in}[b,a,q]}
}{D^2[b,a]}\\[10pt]
= 
\frac{1}{N_sD[b,a]}\cdot
\ds\left[
N_s\cdot \bm{1}_{p=q}-1-
\frac{\left(X^{in}[b,a,p]-M[b,a]\right)}{D[b,a]}
\cdot\frac{\left(X^{in}[b,a,q]-M[b,a]\right)}{D[b,a]}
\right]
\end{array}
\]

Substitution of the normalized forms  which are computed in the forward signal propagation leads to:
\[
\ds\sp{X^{out}[b,a,p]}{X^{in}[b,a,q]}  = 
\frac{1}{N_s\cdot D[b,a]}\cdot
\left[
N_s\cdot\bm{1}_{p=q}-1-\widetilde{X}^{in}[b,a,p]\widetilde{X}^{in}[b,a,q]
\right]
\]

\noindent Hence again by the chain rule:
\[
\begin{array}{rcl}
\ov{X}^{in}[b,a,q] &\eqd&
\ds \sp{\cl{E}}{X^{in}[b,a,q]} = 
\sum_{p\in I_s}\sp{\cl{E}}{X^{out}[b,a,p]}
\cdot\sp{X^{out}[b,a,p]}{X^{in}[b,a,q]}\\[15pt]
&=&\ds
\sum_{p\in I_s}\ov{X}^{out}[b,a,p]
\cdot\sp{X^{out}[b,a,p]}{X^{in}[b,a,q]}
\end{array}
\]

Substitution of the partial derivatives $\sp{X^{out}[b,a,p]}{X^{in}[b,a,q]}$ gives the formula:
\[
\begin{array}{rcl}
\ov{X}^{in}[b,a,q] &\eqd&
\ds\frac{1}{N_sD[b,a]}\cdot
\sum_{p\in I_s}\left[N_s\bm{1}_{p=q}-1-\widetilde{X}^{in}[b,a,q]
\widetilde{X}^{in}[b,a,p]\right]
\ov{X}^{out}[b,a,p]\\[15pt]
&=&\ds\frac{1}{N_sD[b,a]}\cdot
\sum_{p\in I_s}\left[N_s\cdot \bm{1}_{p=q}\cdot\ov{X}^{out}[b,a,p]-\ov{X}^{out}[b,a,p]\right.\\[15pt]
&&\quad\quad\quad\quad\quad\quad\quad\quad
\left.-\widetilde{X}^{in}[b,a,q]\cdot
\widetilde{X}^{in}[b,a,p]\cdot\ov{X}^{out}[b,a,p]\right]
\end{array}
\]

\noindent  Grouping the terms into separate sums gives interesting statistics of normalized input and the gradient to be transferred back through the normalization unit:
\[
\begin{array}{rcl}
\ov{X}^{in}[b,a,q] &=& 
\ds\frac{1}{N_sD[b,a]}\cdot
\left[
\sum_{p\in I_s}N_s\bm{1}_{p=q}\cdot\ov{X}^{out}[b,a,p]
-\sum_{p\in I_s}\ov{X}^{out}[b,a,p]
\right.\\[5pt]
&&\ds\quad\quad\quad\quad\left.
-\sum_{p\in I_s}\widetilde{X}^{in}[b,a,q]\cdot
\widetilde{X}^{in}[b,a,p]\cdot\ov{X}^{out}[b,a,p]
\right]
\end{array}
\]

The first summation reduces to the term $p=q$ what implies the centering of the gradient data in the signal domain. Last term is related to the correlation between the normalized input signal and the transferred gradient.
\[
\begin{array}{rcl}
\ov{X}^{in}[b,a,q] 
&=&\ds\frac{1}{D[b,a]}\cdot
\left[
\ov{X}^{out}[b,a,q]
-\left(\frac{1}{N_s}\cdot\sum_{p\in I_s}\ov{X}^{out}[b,a,p]\right)
\right.\\[5pt]
&&\ds\quad\quad\quad\quad\left.
-\widetilde{X}^{in}[b,a,q]\cdot\left(\frac{1}{N_s}\cdot\sum_{p\in I_s}
\widetilde{X}^{in}[b,a,p]\cdot\ov{X}^{out}[b,a,p]\right)
\right]
\end{array}
\]
\end{proof}

From the form of gradient flow equation \eqref{eq:in-grad-1} we conclude that the instant normalization significantly changes the gradient $\ov{X}^{out}[b,a,q]$. Firstly it is centered wrt to the signal domain average, next reduced by the normalized input $\widetilde{X}^{in}[b,a,q]$ scaled with the correlation coefficient between the normalized input and the gradient  to be transferred. Finally the result is divided by the standard deviation of the input tensor wrt the signal domain.

\paragraph{Instant normalization in testing stage?}

Contrary to the batch normalization, the presence of instant normalization in the testing stage is mandatory. The unit has no learned parameters, and the instantly normalized tensor is fully determined from the original input tensor.

\subsubsection{ReLU, leaky ReLU, sigmoid, hyperbolic tangent, and their gradient flow equations}

There are few digital media applications (e.g. steganography) using additionally to ReLU other nonlinearities in their neural modules. For completeness their definitions and gradient flow equations are joined below:
{\small
\begin{equation}\label{eq:lst}
\begin{array}{l|l|l}
name & definition & gradient\ flow\ equation\\
\hline
ReLU & 
f(x)\eqd\left\{
\begin{array}{ll}
x & \text{if }\ x>0\\
0 & \text{otherwise}
\end{array}
\right.
&
\ov{X}^{in}[b,i] = 
\ov{X}^{out}[b,i]\cdot
\left\{
\begin{array}{ll}
1 & \text{if }\ X^{in}[b,i]>0\\
0 & \text{if}\ X^{in}[b,i]<0\\
0.5 & \text{otherwise}
\end{array}
\right.
\\[20pt]
\hline
leaky\ ReLU_k & 
f(x)\eqd\left\{
\begin{array}{ll}
x & \text{if }\ x>0\\
\frac{k\cdot x}{100} & \text{otherwise}
\end{array}
\right.
&
\begin{array}{l}
\ov{X}^{in}[b,i] =\\[5pt] 
\ov{X}^{out}[b,i]\cdot
\left\{
\begin{array}{ll}
1 & \text{if }\ X^{in}[b,i]>0\\
k\cdot 0.01 & \text{if}\ X^{in}[b,i]<0\\
0.5+k\cdot 0.005 & \text{otherwise}
\end{array}
\right.
\end{array}
\\[20pt]
\hline
sigmoid & 
f(x)\eqd \frac{1}{1+e^{-x}}
&
\ov{X}^{in}[b,i] =
\ov{X}^{out}[b,i]\cdot X^{out}[b,i]\cdot(1-X^{out}[b,i])
\\[5pt]
\hline
hyperbolic\ tangent & 
f(x)\eqd \frac{e^x-e^{-x}}{e^x+e^{-x}}
&
\ov{X}^{in}[b,i] =
\ov{X}^{out}[b,i]\cdot\left[1-\left(X^{out}[b,i]\right)^2\right]
\end{array}
\end{equation}
}
\section{Symbolic tensor neural network (STNN)}\label{sec:stnn}

\subsection{Symbolic representation: decorated symbols}

\begin{description}
\item[Decorated symbol:]$ $\\
Each symbol represents a processing unit which can be parametrized by filling all or some of five parametric fields, e.g.:\\[5pt]
\framebox{\xgen{1}{2}{3}{4}{5}}
\item[Fields of symbol:]$ $\\
Field's content depends on type of processing unit the symbol represents. In the latter example the symbol $\bb{Q}$ represents an abstract unit, so we can outline the intended use of each field:\\[5pt]
\framebox{\xgen{slicing}{featuring}{optioning\!}{sharing\!}{$o$}}

The compact, and imprecise meaning of the symbol fields can be outlined:
\begin{enumerate}
  \item {\it slicing:}  specifies how the data is selected from the input tensor for the operation the unit is designed for,
  \item {\it featuring:} defines the number of features computed -- it is the depth of output data tensor,
  \item {\it optioning:} specifies options of the operation, e.g. average or max pooling is to be performed,
  \item {\it sharing:} if parameter data is to be shared with other units, a label  for shared tensor is displayed,
  \item {\it o:}  a symbol of  activation \un{o}perations joined to the unit.
\end{enumerate}
\item[Coding kernel parameters:] The convolution unit is based on   the kernel concept. Therefore encoding shape of kernel, its striding, and dilation into the featuring field is important:
\begin{enumerate}
\item Kernel shape in signal domain\footnote{The depth of kernels, i.e. the size of the attribute axis, is not encoded as it is always equal to the depth of the input tensor.}:
\begin{itemize}
  \item by literals -- examples for $width\times height$:
\begin{itemize}
  \item   $3\times 3 \mapsto$ $\boxed{3_k^x3_k^y}$ or $\boxed{3_k^y3_k^x}$ or the compact form $\boxed{3_k}$
  \item   $3\times 5 \mapsto$ $\boxed{3_k^x5_k^y}$ or $\boxed{5_k^y3_k^x},$
  \item for literals the index $k$ can be dropped, i.e. the equivalent forms for the above examples are: $\boxed{3}$, $\boxed{3^x5^y}$, and $\boxed{5^y3^x}.$
\end{itemize}
  \item by parameter $k$ for user defined units\footnote{The value of the parameter is assigned in the expressions defined within user defined unit. If the value is a list then zero based indexing is used.}:
\begin{itemize}
  \item   $a\times b \mapsto$ $\boxed{k^xk^y}$ or $\boxed{k^yk^x}$ if the expression is, e.g. $k^x = 1_{\$}; k^y = 2_{\$}$ with right hand values implied by the arguments of an instance call for the unit instance call $\bb{U}(a,b).$
  \item   $a\times a \mapsto$ $\boxed{k}$ if expression is, e.g. $k = 1_{\$}$ for the unit instance call $\bb{U}(a).$
  \item   $a\times a \mapsto$ $\boxed{k_1}$ if expression is, e.g. $k = 1_{\$}$ for the unit instance call $\bb{U}([?,a]).$
  \item   $a\times b \mapsto$ $\boxed{k_0^xk_1^y}$ if the expression is, e.g. $k = [a,b]$ or $k = 1_{\$}$ for the unit instance call $\bb{U}([a,b]).$
\end{itemize}
\end{itemize}
\item Dilation of kernel signal domain:
\begin{itemize}
  \item by literal: e.g. dilation by 2 is encoded as $\boxed{2_{\delta}},$
  \item by parameter $\delta$: e.g. dilation by 3 is encoded as $\boxed{\delta}$ for $\delta = 1_{\$}$ with $\bb{U}(3).$ 
  \end{itemize}
\item Striding of kernel signal domain:
\begin{itemize}
  \item by literal: e.g. striding by 2 is encoded as $\boxed{2_{\sigma}}$ while fractionally strided (transposed) convolution by $\frac{1}{2}$ is written as $\boxed{2_{\sigma}}$ with the symbol $t$ in the option field,
  \item by parameter $\sigma$: e.g. striding by 3 and by $1/3$ both are encoded as $\boxed{\sigma}$ for $\sigma = 2_{\$}$ with $\bb{U}(?,3)$ -- the fractionally strided (transposed) is distinguished by the symbol $t$ in the option field.
  \end{itemize}
\end{enumerate}
\item[Coding the number of convolutions]: Since the number of convolutions the unit performs equals to the number of features computed by the unit it fills the field of featuring. Therefore, the encoding is direct without using the feature axis label:
\begin{itemize}
  \item by literals: e.g. $\boxed{100}$ if $100$ output features is produced,
  \item by parameter $f$ of scalar value: e.g. $\boxed{f}$ if an expression, e.g. $f = 1_{\$}$ defines the value for $f$ via instance call $\bb{U}(100).$
  \item by parameter $f$ of list value: e.g. $\boxed{f_1}$ if an expression, e.g. $f = 1_{\$}$ defines the value for $f$ via instance call $\bb{U}([?,100]).$
\end{itemize}
\item[Coding pooling window parameters:] The pooling unit is based on pooling  window concept. Therefore encoding shape of pooling window, and its striding is important, as well:
\begin{enumerate}
\item Pooling window shape in the signal domain is encoded exactly in the same way as encoding of the kernel shape.
\item Pooling window striding is an extension of kernel striding (however, without fractional option):
\begin{itemize}
  \item by literals -- examples for $stride\ horizontally\times stride\ vertically$:
\begin{itemize}
  \item   $3\times 3 \mapsto$ $\boxed{3_{\sigma}^x3_{\sigma}^y}$ or $\boxed{3_{\sigma}^y3_k^{\sigma}}$ or the recommended form $\boxed{3_{\sigma}}$
  \item   $3\times 5 \mapsto$ $\boxed{3_{\sigma}^x5_{\sigma}^y}$ or $\boxed{5_{\sigma}^y3_{\sigma}^x}$
\end{itemize}
  \item by parameter $\sigma$ for user defined units:
\begin{itemize}
  \item   $u\times v \mapsto$ $\boxed{\sigma^x\sigma^y}$ or $\boxed{\sigma^y\sigma^x}$ if the expression is, e.g. $\sigma^x = 1_{\$}; \sigma^y = 2_{\$}$ with right hand values implied by the arguments of an instance call for the unit instance call $\bb{U}(u,v).$
  \item   $u\times u \mapsto$ $\boxed{\sigma}$ if expression is, e.g. $\sigma = 1_{\$}$ for the unit instance call $\bb{U}(u).$
  \item   $u\times v \mapsto$ $\boxed{\sigma_0^x\sigma_1^y}$ if the expression is, e.g. $\sigma = [u,v]$ or $\sigma = 1_{\$}$ for the unit instance call $\bb{U}([u,v]).$
\end{itemize}
\end{itemize}
\end{enumerate}
\item[List of symbols:]$ $\\
In the above FP68 DNN algorithms there is limited set of DNN processing units. However, they are representative for the contemporary architectures used in adaptive image recognition and generally in multimedia data processing.
\begin{enumerate}
\item $\bb{C}$: \un{C}onvolution -- linear processing unit,
\item $\bb{F}$: \un{F}ull \un{c}onnection -- linear processing unit,
\item $\bb{P}$: \un{P}ooling -- linear/nonlinear data subsampling unit (options: \un{a}verage, \un{g}lobal average, and \un{m}ax),
\item $\bb{I}$: \un{I}nstance normalization (with none decoration field) -- normalization of data within feature slices (with decorations it is also used  for \un{I}nterpolation),
\item $\bb{Q}$: universal symbol for any processing unit, 
\item $\bb{U}$: \un{U}ser defined processing unit -- an operation specified by the network reusable fragment,
\item $\cl{R}$: \un{R}eLU (Rectified Linear Unit) -- nonlinear element-wise processing unit (denoted by symbol $r$ when included)\footnote{The slope of the leaky ReLU$_k$ defined in \eqref{eq:lst}, is denoted by index $k$: $\cl{R}_k$ and $r_k$ if included.}, 
\item $\cl{B}$: \un{B}atch normalization -- data normalization in data batches (denoted by symbol {\it b} when included),
\item $\cl{S}$: \un{S}igmoid -- nonlinear element-wise processing unit \eqref{eq:lst} (denoted by symbol {\it s} when included)
\item $\cl{H}$: \un{H}yperbolic tangent -- element-wise unit \eqref{eq:lst}  (denoted by symbol {\it h} when included).
\end{enumerate}
Note that in the above schemes element-wise processing units are usually included into the preceding units -- it is not only for compactness of drawing, but for time and memory complexity optimization.
\end{description}

The full picture of data flow will be complete if we know how the processing unit transforms its input tensor into the output one. Each category of units works according its specific procedure. In the presentation of unit "internals" we separate element-wise processing units.

\subsubsection{Convolution unit}

\begin{description}
  \item [Symbolic fields:] By examples: 
\begin{itemize}
\item Example A: \xconv{5^x3^y2_{\sigma}^x1_{\sigma}^y}{512}{p}{}{}
$\left(\text{same as \xconv{5_k^x3_k^y2_{\sigma}^x1_{\sigma}^y}{512}{p}{}{}}\right)$
\begin{enumerate}
\item Fields one and two:
 The input tensor convolution with $N_a'\eqd N_a^{out}=512$ kernels of shape $N_a\times 5\times 3$, where $N_a\eqd N_a^{in}$ follows from the input tensor shape $X^{in}$. Moreover, the convolution is sub-sampled by two in $x$ axis and preserved in $y$ axis. The padding is applied -- by default two pixels in $x$ axis and one pixel in $y$ axis.
\item Field three: The padding is applied and therefore the spatial shape is not changed.
\item Field four: No parameter sharing.
\item Field five: No inclusion of other units.
\end{enumerate}
\item  Example B: \xconv{2_{\sigma}3}{128}{}{\alpha}{br} 
$\left(\text{same  as \xconv{3_k2_{\sigma}}{128}{}{\alpha}{br}}\right)$
$\left(\text{same  as \xconv{2_{\sigma}}{128}{}{\alpha}{br}}\right)$
\begin{enumerate}
\item Fields one and two:
 The input tensor convolution with $N_a'=128$ kernels of shape $N_a\times 3\times \dots\times 3$, where $N_a$ follows from the input tensor shape\footnote{The dimensionality of spatial domain for the input tensor cannot be deduced directly. It will be known when the symbolic inputs of the module are bounded.}. Moreover the convolution is sub-sampled by two for all axes. The last form of the decoration follows from adopting the default kernel size to $3.$ 
\item Field three: The padding is not applied -- therefore the spatial resolution of the output tensor will be reduced slightly.
\item Field four: Parameter sharing with all units having the same sharing label $\alpha$.
\item Field five: The batch normalization (symbol $b$) and ReLU (symbol $r$) are included to be performed just after convolution.
\end{enumerate}
\item Example C: \xconv{3^y3^x3^d}{24}{cp_r}{}{r}
\begin{enumerate}
\item Field one: Each kernel is 3D. 
\item Field two: There are $24$ such kernels.
\item Field three: Padding is reflective ($p_r$) wrt to the  input tensor boundary, the mask is trimmed to be \un{c}asual, i.e. at the lexicographic order the weights with the location pointing above the current element in the input tensor are all equal to zero.
\item Field four: None.
\item Field five: ReLU (symbol $r$) is included to be performed just after convolution.
\end{enumerate}
\item Example D: \xconv{3}{128}{ps_d}{}{b}
\begin{enumerate}
\item Field one: Each kernel is of size $3$ in signal domain. 
\item Field two: There are $128$ such kernels.
\item Field three: The kernel is separable\footnote{The separable kernel $W[d,p]$ wrt to depth axis $d$ satisfies the equation: $W[d,p]=W_1[d]\cdot W_2[p]$, where $p$ belongs to the signal domain.} wrt depth axis and padding is applied.
\item Field four: None.
\item Field five: Batch normalization (symbol $b$) is included to be performed just after convolution.
\end{enumerate}
\item Parametric examples: \xconv{1}{f_0}{}{}{br}, \xconv{k_2}{f_4}{p}{}{br}, \xconv{1}{f_5}{}{}{br}\quad
interpretable for expression $\boxed{\text{\xexpression{k = 1_{\$};\  f = 2_{\$}}}}$ 
\quad 
with the instance call of user defined unit {\it incept:}\quad
\xunitinstance{incept}{1}{\underbrace{[1,3,5,3]}_{1_{\$}}, \underbrace{[8,12,12,8,8,4]}_{2_{\$}}}
\end{itemize}
  \item [Tensor actions:]  If the input tensor is $X^{in}$ then the output tensor $X^{out}$ is defined according the general formula \eqref{eq:conv-gen}. The principle for gradient flow computation by the convolution unit is illustrated in \eqref{code:grad-flow-conv} while the gradient for parameters are given by the explicit equations \eqref{eq:conv-grad-wb}.
\end{description}

\subsubsection{Full connection unit}

\begin{description}
  \item [Symbolic fields:] By example:\ 
\begin{itemize}
\item Example A: \xdense{}{136}{}{\beta}{br}
\begin{enumerate}
\item Field one,three: None
\item Field two: $N_a^{out}=136$ -- the number of output features
\item Field four: Sharing via label $\beta$
\item Field five: Including actions of batch normalization and ReLU
\end{enumerate} 
\item Example B: \xdense{y}{2048}{}{}{}
\begin{enumerate}
\item Field one: The slicing field is used here to define slices along which the full connection is performed. Here $y$ means that along the axis labeled by $y$ there is full connecting with all features, at fixed other signal domain coordinates.
\item Field two: $N_a^{out}=2048$ -- the number of output features
\item Fields three, four, five: None
\end{enumerate} 
\item Parametric example: \xdense{}{1_{\$}}{}{}{br} represents the example A if the instance call has the form $\bb{U}(136).$ 
\end{itemize}
  \item [Tensor actions:] If the input tensor is $X^{in}$ then the output tensor $X^{out}$ is defined according the general formula \eqref{eq:fc-gen}. The gradient flow through the unit is given by the formula \eqref{eq:fc-grad-x} while the gradient for parameters by \eqref{eq:fc-grad-w}, and \eqref{eq:fc-grad-b}.
\end{description}

\subsubsection{Pooling unit}
\begin{description}
  \item [Symbolic fields:] By examples:
  \begin{itemize}
\item Example {\it maximum}: \xpool{3}{}{m}{}{}  
\begin{enumerate}
\item Field one: The {\it max} pooling is performed in signal regular windows of width three. 
\item Field two: None as no change for attribute number in pooling.
\item Field three: {\it m} to indicate the type of pooling
\item Fields four,five: None
\end{enumerate}  
\item Example {\it average}: \xpool{3^x2^y}{}{a}{}{}
\begin{enumerate}
\item Field one: The {\it average} pooling is performed in signal irregular windows of width three and height two.
\item Fields two, four, five: None
\item Field three:  {\it a} to indicate the type of pooling
\end{enumerate} 
\item Example {\it global average}: \xpool{g}{}{a}{}{}  
\begin{enumerate}
\item Field one: {\it g} to indicate global operation in attribute slices.
\item Fields two, four, five: None
\item Field three: {\it a} to indicate type of global operation.
\end{enumerate}
\item Parametric examples:
\begin{itemize}
\item Example A:
\xpool{k_3}{}{m}{}{}\quad
interpretable for expression $\boxed{\text{\xexpression{k = 1_{\$};\  f = 2_{\$}}}}$ 
\quad 
with instance call of user defined unit {\it incept:}\quad
\xunitinstance{incept}{1}{\underbrace{[1,3,5,3]}_{1_{\$}}, \underbrace{[8,12,12,8,8,4]}_{2_{\$}}}
\item Example B:
\xpool{k^x_2k^y_2\sigma^x_2\sigma^y_2}{}{m}{}{}
interpretable for expression\\[5pt]
$
\boxed{\text{
\xexpression{
g = 1_{\$};\ \sigma^x = \lfloor n_x/ g\rfloor;\ \sigma^y = \lfloor n_y/ g\rfloor};\ 
\xexpression{k^x = \sigma^x+n_x\bmod g;\ k^y = \sigma^y+n_y\bmod g}}}
$ 
\quad 
with instance call of user defined unit 
\xunitinstance{spp}{1}{[5,3,1]}, where $n_y\times n_x$ is the spatial resolution of the input for this max pooling unit.
\end{itemize}  
 \end{itemize}
  \item [Tensor actions:] In pooling window defined in the spatial domain, separately in attribute slices of input tensor the aggregation is performed, effectively working as sub-sampling. For global pooling the window is extreme -- it covers all spatial domain and the aggregation computes the single value per input attribute. The drawback of max pooling during training is the significant memory overhead for the maximum locations in each sampling window.
If the input tensor is $X^{in}$ then the output tensor $X^{out}$ is defined according the general formula for max pooling \eqref{eq:max-pool} and for average pooling \eqref{eq:avg-pool}. The gradient flow through the unit is given by the formula \eqref{eq:max-pool-grad} for max pooling and by \eqref{eq:avg-pool-grad} for average pooling.
\end{description}

\subsubsection{Interpolation unit}

\begin{description}
  \item [Symbolic fields:] By example:\ 
\begin{itemize}
\item Example A: \xinterp{125}{b}
\begin{enumerate}
\item Field one: The resolution $n_s^{in}$ of signal axes is changed to $n_s^{out}\eqd \lfloor 1.25\cdot n_s^{in}\rfloor.$ 
\item Field two: The interpolation is \un{b}ilinear.
\item Field three, four, five: None.
\end{enumerate} 
\item Example B: \xinterp{75}{3_rg}
\begin{enumerate}
\item Field one: The resolution $n_s^{in}$ of signal axes is changed to $n_s^{out}\eqd \lfloor 0.75\cdot n_s^{in}\rfloor.$ 
\item Field two: The decimation is \un{g}aussian in the neighborhood of radius $r=3.$
\item Field three, four, five: None.
\end{enumerate} 
\end{itemize}
  \item [Tensor actions:] If the input tensor is $X^{in}$ then the output tensor $X^{out}$ is defined according to the general formula \eqref{eq:interp-nk}. The gradient flow through the unit is given by the formula \eqref{eq:interp-grad}.
\end{description}

\subsubsection{Batch normalization unit}
\begin{description}
  \item [Symbolic fields:] By example:\ \xbatch
\begin{enumerate}
\item Fields one, two, three, four, five: None
\end{enumerate}  
  \item [Tensor actions:] Component-wise operation for this unit and its gradient flow equation are given by the formulas
  \eqref{eq:bn-1}, \eqref{eq:bn-2}, \eqref{eq:bn-grad-1},
  \eqref{eq:bn-grad-2}.
\end{description}

\subsubsection{Instance normalization unit}

\begin{description}
\item [Symbolic fields:]\ \xinstant
\begin{enumerate}
\item Fields one, two, three, four, five: None.
\item If included in another symbol it is represented by $i.$
\end{enumerate}  
  \item [Tensor actions:] The instance normalization is defined by the formula \eqref{eq:in-1} while its gradient flow by \eqref{eq:in-grad-1}.
\end{description}

\subsubsection{ReLU/leaky and ReLU unit}
\begin{description}
  \item [Symbolic fields:] By example:\ \xrelu,\quad \xrelu$_k$
\begin{enumerate}
\item Index $k$ denotes the slope $k\cdot0.01$ for the negative input.
\item Fields one, two, three, four, five: None.
\end{enumerate}  
  \item [Tensor actions:] Component-wise operation for this unit and its gradient flow equation are given in the equations \eqref{eq:lst}.
\end{description}


\subsubsection{Sigmoid and hyperbolic tangent units}
\begin{description}
  \item [Symbolic fields:] By example:\ \xsigmo,\quad \xtanh
\begin{enumerate}
\item Fields one, two, three, four, five: None.
\item If the operations are included into another one then the letters {\it s,h} are used, respectively.
\end{enumerate}  
  \item [Tensor actions:] Component-wise operation for these units and their gradient flow equations are given by the formulas
  \eqref{eq:lst}.
\end{description}


\subsection{Symbolic representation: DAG network of symbols}

STNN symbolic representation includes few simple composition rules for assembling decorated symbols into DAG (Directed Acyclic Graph) of symbols:
\begin{enumerate}
\item Juxtaposition  -- concatenation of decorated symbols into the string of symbols, e.g.:

\xgen{1}{2}{3}{4}{5}
\xgen{1}{2}{3}{4}{5}
$\dots$
\xgen{1}{2}{3}{4}{5}


The symbols represent either processing units or input units. 

\item Concatenated symbols create components being preceded and followed by labels which play the role of links between units, e.g.:

\xfromlabel{\beta_1}
\xgen{1}{2}{3}{4}{5}
$\dots$
\xgen{1}{2}{3}{4}{5} 
\xtolabel{\beta_2}

In case of input unit the label identifies the input tensor, otherwise the label before component identifies a unit from where the data is requested while the label after the component is used to match such requests. 

\item Any symbol representing a processing unit can be followed by a label which can be requested by any other component, e.g.:

\xgen{1}{2}{3}{4}{5}
\xtolabelto{\beta_1}
\xgen{1}{2}{3}{4}{5}
$\dots$
\xgen{1}{2}{3}{4}{5}
\xtolabel{\beta_3}

The labeling should avoid looping of data requests like in:

\xfromlabel{\beta_1}
\xgen{1}{2}{3}{4}{5}
\xtolabelto{\beta_2}
\xgen{1}{2}{3}{4}{5}
\xtolabel{\beta_1}

\item The {\it adder link} is an exception since its output label, say $\beta$, is also the request for data from all units having the same output label $\beta$:

\xfromlabel{\beta_3}
\xgen{1}{2}{3}{4}{5}
\xtolabelto{\beta_4}
\xgen{1}{2}{3}{4}{5}
\dots
\xgen{1}{2}{3}{4}{5}
\xtolabeltoadd{\beta_4}
\xgen{1}{2}{3}{4}{5}
\dots
\xgen{1}{2}{3}{4}{5}
\xtolabel{\beta_5}

In the notation\ 
\xgen{1}{2}{3}{4}{5}
\xtolabeltoadd{\beta_4} \quad the operation $+$ requests the data not only from the preceding unit but also from any unit labeled by $\beta_4.$

\item There is {\it shortcut} notation simplifying the above adder construction by hiding the labels:

\xfromlabel{\beta_3}
\xgen{1}{2}{3}{4}{5}
\Big\langle
\xgen{1}{2}{3}{4}{5}
\dots
\xgen{1}{2}{3}{4}{5}
\Big\rangle
\xgen{1}{2}{3}{4}{5}
\dots
\xgen{1}{2}{3}{4}{5}
\xtolabel{\beta_5}

\item Interconnecting (of blocks and/or symbolic inputs) can be also achieved by merging of inputs and splitting of outputs, e.g.:
\begin{itemize}
\item Merging of tensors labeled by names  of units where they are produced or names for split output branches:

\xmerge{merger\ input\ labels}{y}
\xgen{1}{2}{3}{4}{5}
\xgen{1}{2}{3}{4}{5}
$\dots$
\xgen{1}{2}{3}{4}{5}

where $y$ stands for stacking input tensors along their $y$ axis.
The inputs can be stacked along spatial and feature (depth) axes: $x,\dots,d$. 
\item splitting of output tensor using the specified operator -- then the block label is replaced by the sequence of labels (one for each output of splitter), e.g.:

\xgen{1}{2}{3}{4}{5}
\xgen{1}{2}{3}{4}{5}
$\dots$
\xgen{1}{2}{3}{4}{5}
\xsplit{x}{splitter\ outputs\ names}

where $x$ stands for unstacking output tensor along the $x$ axis.
\end{itemize}


\item Encapsulating components into reusable units by definition of new unit:

\xunitdef{rc}{
\xresid{
\xconv{3}{64}{p}{}{~r}
\xconv{2_{\sigma}3}{64}{p}{}{br}
}{}}

\item Bounding of the input source labels by tensor shapes. It actually gives the symbolic instance of the designed network. The network instance is labeled to be used for the training different options of input tensors like (RGB versus Y images, small resolution versus high resolution images, different batch sizes, etc.) 

\xbound{A}{2}{\alpha_1:=shape, \dots}

Then the consistency of any assignment for input shapes with the defined symbolic DAG can be checked before the actual signal flow takes place.
\end{enumerate}

The STNN (symbolic tensor neural network) can be verified on axes compatibility, however to propagate data signals from inputs to outputs, all symbolic inputs must be bound to the data tensors with real data in allocated buffers. Then in topological order the processing units are visited, since then we are sure that all inputs of the unit are ready to be processed. 

\subsection{BNF grammar for symbolic neural networks}

STNN DAG is defined from units by composition rules. In the similar formal way, the programming languages (like {\tt Javascript}) or file formats (like {\tt JSON}) are described to make parsing and code generation possible.
However, the composition rules use primitive tokens and grammar rules which are augmented by semantic rules.  Moreover, there is a subtle difference between a formal language and the set of all DNN DAGs -- the former is the set of symbol strings while the latter is the set of annotated graphs which are not sequential objects, in general. 

Despite this basic difference, the grammar rules for DNN DAGS are possible as the Backus-Naur Form  (BNF) notation  can be used via the concept of labels assigned to DNN decorated unit symbols. 

BNF definition is the sequence of rule definitions (cf. Knuth's paper \cite{KnuthD64a}):
\[ SyntacticElement ::= CompositionRule\]
The composition rule is an algebraic expression, i.e. the sequence of syntactic operators applied to syntactic elements. We use the following syntactic operators, listed according their decreasing priorities:
\begin{enumerate}
  \item $x^+$ -- composition  by the repetition of elements of type $x$, 
  \item $x^{+\sigma}$ -- composition  by the repetition of elements of type $x$ separated by a symbol $\sigma$ (e.g. $\sigma$ equals to the comma $,$, the dot $.$, or OR symbol $|$),   
  \item $x^{\uplus}$ --  collection of elements of type $x$, i.e. the repetition order is not important, 
\item $xy$ -- juxtaposition, i.e. placing the element $y$ after the element $y$,
  \item $[x]$ -- the element $x$ can be joined optionally in the sequence being defined, 
  \item $\left<x\right>$ -- shortcut brackets which define the residual element\footnote{The {\it shortcut brackets} is the only operator which refers directly to the nonlinear graph structure of STNN. Generally links in the graph being defined, are produced via labeling mechanism of unit blocks, components, and input units.},
  \item $x|y$ -- alternative between elements $x$ and $y$,
  \item $\dots$ -- a list continuation operator for atomic symbols (to be filled by more symbols).
\end{enumerate}

Any formal grammar defining one final syntactic element, like STNN, can be presented in bottom-up or top-down way. We prefer the bottom-up presentation. However, then a bird's eye view for STNN formal element is of merit:
\begin{enumerate}
  \item The instance of STNN is the collection of instances of STNN components augmented by definitions of user units. 
  \item The units defined by user are the special reusable network fragments
 which beside the explicit instancing can be implicitly instanced (like custom units) just by the juxtaposition operator. 
  \item The implicit instancing for the units is possible since their default parameters can be used. 
  \item In case of non default settings of parameters the instance id field is defined instead of custom decorating fields. 
  \item The whole picture is closed by shape bounding operator which assigns the specific input tensor(s) shape(s).
  \item There is possibility to make more than one shape bounding\footnote{It is just a {\it syntactic/semantic sugar}.} for the net in order to test the impact of different resolution and signal channels on the network performance.
\end{enumerate}

We divide the BNF grammar for STNN into the following parts: 
units, inputs, labels, links, blocks, merge/split elements, components, user defined units, and symbolic nets.

\begin{description}
  \item [Units:]$ $\\
\begin{description}
\item [Custom units:]  
\[
\begin{array}{rcl}
CustomUnit          & ::= & 
\decor{\bb{C}}
\left|\ \decor{\bb{F}}\ \right.
\left|\ \decor{\bb{P}}\right.
\left|\ \decor{\bb{I}}\right.
\left|\ \decor{\bb{Q}}\right. 
\dots \left|\ \cl{R} \right.
\left|\ \cl{B}\ \right.
\left|\ \cl{S}\ \right.
\left|\ \cl{H}\ \right.
\left|\ \dots \right.
\end{array}
\]
\item [User defined units:]  
\[
\begin{array}{rcl}
UserUnit          & ::= & \xunit{unit\ name}{[instance\ id]}{}\end{array}
\]
\item [All units:]  
\[
\begin{array}{rcl}
Unit          & ::= & CustomUnit\ |\ UserUnit
\end{array}
\]
\end{description}
The following syntactic/semantic constraints are imposed:
\begin{itemize}
  \item The custom units list of symbols consists of two sublists --
   the list of symbols for computing units which process tensor in element-wise way, and the list for other computing units,
  \item The symbols $\bb{U}$ 
  cannot be used as custom symbol -- it is reserved for user defined units and network fragments.
  \item Each symbol (also calligraphic one) can be decorated by fields -- their existence in the given position and the syntax for their values dependent on units type.
  \item The user unit is defined by the user which joins its definition to the current definition of STNN instance.
  \item The user unit type is identified by its name which is located in the center-up field. The type name is the part of the unit definition.
  \item The user unit instance is optional (center-down field) and it is used only if non default parameters are to be set what requires the explicit instance of the unit.
  \item The explicit instance of any unit can be used in more than one place  -- so we distinguish between the instance creation and its use which actually is instancing of the unit instance\footnote{It is different approach comparing to the objective programming languages where the use of class instance always refers to the same object -- here we get the different part of STNN being defined.}.
  \item All decorating fields get the limited place for their values. Therefore the unit type names are rather acronyms than real names -- for instance \xunit{au}{5}{} or just \xunit{bu}{}{}{}.
\end{itemize}

  \item [Inputs:]
\[
\begin{array}{rcl}
InputUnit &::=& \text{\xin{a}{b}{id\ label}}
\end{array}
\]
  

The symbol $\cl{I}$ is used to define the network's input nodes. Contrary to the network's output nodes which are identified by output tensor labels, the input symbol provides not only the id label of a tensor but also the symbolic information on the geometric axes (field $a$) and the attribute axis (field $b$) with features or channels of the signal tensor. The actual resolution of each axis is assigned along the hyper parameters are defined for a STNN instance. During the forward signal flow the memory block assigned for input tensor is identified via {\it id} of this symbol.

If two different computing units take the data from the same input tensor labeled by $id$ then we precede both units by the input symbol with the same label $id,$ i.e. the input tensor is shared by different units of the network.

Despite the symmetry appeal we do not define the output symbol for the following reasons:
while the input tensor initiates the forward signal propagation, the output data references can be defined in any place of the network, for instance two output nodes can be on the same signal propagation path,
some data tensors are needed only for model training, a subset of them for model use, and some for analysis/visualization tasks -- the three categories which can be handled by labels categories without additional overhead for output symbol $\cl{O}$.

The fields used get the following constraints:
\begin{itemize}
  \item right-up: attribute axis description, e.g.: $3_a, 3_d, bgr$ denote the same \un{a}ttribute (\un{d}epth) axis with three features, the order $bgr$ means that the pixel storage begins form red component,
  \item right-down: signal axes description, e.g.: $xyz$ denotes $3D$ signal domain without resolution specification; $256_y128_x$ denotes the $2D$ signal domain with axis labeled by $yx$, the data is stored in $256$ rows with $128$ pixels each; $200_{xy}$ denotes the image with $200$ columns each of $200$ adjacent pixels; $200_{yx}$ denotes the image with $200$ rows each of $200$ adjacent pixels:
 \xin{xyz}{}{\alpha}\quad \xin{256_y128_x}{3}{rgb}\quad 
 \xin{200_{xy}}{3_a}{\alpha}\quad \xin{200_{yx}}{3_d}{\alpha}
\end{itemize}
  \item [Labels:]$ $\\
\begin{description}
\item [to use in links:] 
\[
\begin{array}{rcl}
Label           & ::= & UniqueName
\end{array}
\]
\item [to enumerate user unit inputs:] 
\[
\begin{array}{rcl}
UnitInputLabel  & ::= & \alpha\ |\ {id}_{\alpha}
\end{array}
\]
\item [to enumerate user unit outputs:]
\[
\begin{array}{rcl}
UnitOutputLabel & ::= & \omega\ |\ {id}_{\omega}
\end{array}
\]
\end{description}
 
 When two units are chained in a STNN component then the number of outputs in the first one should be equal to the number of inputs in the second one. The matching of inputs to outputs is based on their labels, e.g. the input $3_{\alpha}$ requests the data from the output $3_{\omega}$.

\item [Links:]$ $\\
\begin{description}
\item [to multi-cast unit output tensor:]
\[
\begin{array}{rcl|rcl}
ToLabel & ::= & \xtolabel{Label} \quad&\quad
ToLabelTo & ::= & \quad\xtolabelto{Label}\\
ToLabelUnique & ::= &  \xtolabel{Label}
\end{array}
\]
\item [to uni-cast unit outputs:]
\[
\begin{array}{rcl|rcl}
ToLabelAdd & ::= & \quad\xtolabeladd{Label}
 \quad&\quad
ToLabelToAdd & ::= & \quad\xtolabeltoadd{Label}
\end{array}
\]
\item [to request unit output:]
\[
\begin{array}{rcl}
FromLabel    & ::= & \xxfromlabel{Label}
\end{array}
\]
\item [to refer to unit output:]
\[
\begin{array}{rcl}
ToRefLabelTo      & ::= & \quad\xtoreflabelto{Label}
\end{array}
\]
\end{description}
\item [Blocks:]$ $\\
\begin{description}
\item [Residual blocks:]
\[
\begin{array}{rcl}
ShortCutBlock & ::= & 
\left.\xresid{$Unit^+$}{} \quad\right|\quad\xxresid{$Unit^+$}
\end{array}
\]
\item [Cast adders:]
\[
\begin{array}{rcl}
CastAdder & ::= & 
\left.\xresid{$Unit^{+|}$}{} \quad\right|\quad\xxresid{$Unit^{+|}$}
\end{array}
\]

\item [All blocks:]
\[
\begin{array}{rcl}
Block & ::= & Unit^+\ |\ ShortCutBlock\ |\ CastAdder
\end{array}
\]
\end{description}
The shortcut brackets for a sequence of units creates the {\it residual} element which computes the function $Y=G(X)=F(X)+X$. Then the original sequence of units actually computes the function $F(X)=Y-X$ (the claimed residuum). In the gradient dual network the dual residual element transforms the gradient $\od{\cl{E}}{Y}$ to $\od{\cl{E}}{X}$ by the following equation (cf. \eqref{jacoby}):
\begin{equation}
Y=G(X)=F(X)+X \lra \left[\od{\cl{E}}{X} = J_G(X)\od{\cl{E}}{Y} = J_F(X)\od{\cl{E}}{Y}+\od{\cl{E}}{Y}\right]
\end{equation}
It means that the original gradient when the unit sequences are getting  shortcuts, is increased by the input gradient. This amplification of the gradient avoids the usual attenuation by the consecutive layers of the dual network. This mathematical trick enables the very long signal propagation length, i.e. many layers.

The definition of the shortcut block in terms of STNN symbols needs an extra label, say $\beta$, to label the input tensor $X$: 
\begin{center}
\xshortcut{$\ \ Units\ \ $}{\beta}
\end{center}
The incompatibility of input/output tensors and the need for the projection we denote by the label with the exclamation mark:
\begin{center}
\xshortcut{$\ \ Units\ \ $}{!\beta}
\end{center}

The only constraint for the sequence of computing units to be bracketed is the {\it tensor shape compatibility of input and output}: the shape of tensor $X$ equals to the shape of the resulting tensor $F(X).$ 
If this constraint is not respected (like in ResNet-50 architecture) the residual equation is modified by a linear projection operator $P$. The exclamation mark just after the left bracket means that the projection $P$ is implicitly used.

If $P$ is constant then, according to the equation \eqref{eq:grad-flow}, the gradient flow equation for the residual element is modified as follows:
\begin{equation}
Y=G(X)=F(X)+PX \lra \left[\od{\cl{E}}{X} = J_F(X)\od{\cl{E}}{Y}+\tp{P}\od{\cl{E}}{Y}\right]
\end{equation}

The {\it cast adder} is a natural extension of the residual block, i.e. instead of two branches with one branch having no processing or with processing for only shape conforming to the other branch, we have two or more branches with the specified sequence of processing units.

For instance the cast adder block \xresid{
\xconv{2_{\sigma}1}{128}{p}{}{~r}\ $\Big|$
\xconv{3}{128}{ps_d}{}{br}
\xconv{3}{128}{ps_d}{}{b}
\xpool{2_{\sigma}3}{}{m}{}{}
}{} 
defines the block with two branches. If $X$ is the input tensor of this block which processed in branch one to the tensor $Y_1$ and in the branch two to the compatible tensor $Y_2$ then the output tensor $Y=Y_1+Y_2.$

If the exclamation mark $!$ is used within a cast adder then it means that the input tensor is incompatible with branch outputs and it is conformed in its shortcut branch.

\item [Merge/Split:]$ $\\
\begin{description}
\item [to merge unit inputs:]
\[
\begin{array}{rcl}
MergeUnit & ::= & \xmerge{Label\ Label^+}{\Box}\ Unit
\end{array}
\]
\item [to split unit output:]
\[
\begin{array}{rcl}
UnitSplit & ::= & Unit\ \xsplit{\Box}{Label\ Label^+}
\end{array}
\]
\end{description}

  \item [Components:]$ $
\begin{description}
\item [to begin component:]  
\[
\begin{array}{rcl}
BeginComponent & ::= & InputUnit\ |\ FromLabel\ \  Block\\[5pt]&&|\ MergeUnit
\end{array}
\]
\item [to end component:]  
\[
\begin{array}{rcl}
EndComponent   & ::= & ToLabel\ |\ ToLabelAdd\ |\ UnitSplit
\end{array}
\]
\item [to fill component:]  
\[
\begin{array}{rcl}
ComponentElement & ::= & Block\ |\ ToLabelTo\ \ Block\ |\ ToLabelToAdd\ \ Block\\[5pt]
&&|\ ToRefLabelTo\ \ Block
\end{array}
\]
\item [to assemble components:]  
\[
\begin{array}{rcl}
Component    & ::= & BeginComponent\ \ [ComponentElement^+]\ \ EndComponent
\end{array}
\]
\end{description}  

{\it Remark:} The component has always its input tensor and output tensor(s) defined, i.e. the unit labels provide the consistent links between network inputs and outputs.

\item [User defined units:]$ $\\
\begin{description}
\item [to instance unit:]
\[
\begin{array}{rcl}
UserUnitInstance      & ::= &  
\xunitinstance{unit\ name}{instance\ id}{free\ labels\ bounding}
\end{array}
\]
\item [to define unit:]
\[
\begin{array}{rcl}
UserUnitDefinition    & ::= &  \xunitdef{unit\ name}{$Component^{\uplus}$}\quad[UserUnitInstance^+]
\end{array}
\]
\item [to collect unit definitions:]
\[
\begin{array}{rcl}
UserUnitDefinitions &::=& \boxed{UserUnitDefinition^+}
\end{array}
\]
\end{description}  
  \item [Symbolic net:]
\[
\begin{array}{rcl}
SyblicNetInstance & ::= & \xbound{net\ name}{net\ id}{NetTrainParams}\\[5pt]
SyblicNet   & ::= & 
[UserUnitDefinitions]\ Component^{\uplus}\ SyblicNetInstance^{\uplus}
\end{array}
\]
  
\end{description}

\paragraph{Network training parameters\\}

We assume the following syntax for {\it network training parameters}:
\[
\begin{array}{rcl}
InputDef & ::= & InputLabel\ :=\ InputShape;\\
GoalType & ::= & loss\ |\ gain\\
OptimaDef & ::= & LeftBracket\ GoalType, OptimaAcronym,\\&& EquationReference\ RightBracket;\\
NetTrainParams & ::= & InputDef^+\ \ OptimaDef
\end{array}
\]
In order to train the network, beside the input/output labels, somehow the lossy function and its gradient should be specified together with the optimization method, its options and its own parameters, as well. 

Except the gradient computation for lossy function which can be attempted to be described by STNN methodology (cf. Paszke et al. \cite{Paszke17a} -- the paper on automatic differentiation in PyTorch), the optimization can be "symbolized" only partially.

\noindent Acronyms for the selected loss/gain functions:{\it SoftMax, definedGoal}.


\noindent Acronyms for the stochastic gradient descent/ascent (SGD/SGA) methods (cf. Ruder's tutorial \cite{Ruder17a}): {\it MomentumSGD, NesterovSGD, AdagradSGD, AdadeltaSGD, RMSpropSGD, AdamSGD, AdamaxSGD, NadamSGD.}


\subsection{Few notes on loss/gain functions and stochastic optimization}

\subsubsection{Note on SoftMax\label{sec:soft-max}}

Let $y\eqd F(x)\inv{K}$ be the vector of scores for the input $x$. Assume that $x$ belongs to the class with the index $t(x)\in[K].$ Then:
\begin{enumerate}
  \item   $y[k]$ -- the score of class $k$, $k\in[K]$,
  \item   $z[k]\eqd e^{y[k]}$ -- the transformed score which results in logarithmic scale for $y$ components,
  \item   $\ds p_x[k]\eqd \ds\frac{z[k]}{\sum_{i\in[K]}z[i]}$ -- the probability that $x$ belongs to class $k$,
  \item    $\cl{E}(x) \eqd -\ln p_x[t(x)]$ -- the loss function\footnote{Since $-\ln p_k=(-\log p_k)\cdot(1/\log e)$ is proportional to the Shannon's information measure, we should rather use the name {\t Shannon information} of target class for SoftMax distribution. There are two  other names used for the same loss function: (a) KL divergence or (b) cross entropy of $\bm{1}_{k}$ distribution with SoftMax distribution.} computed as the Kullback-Leibler divergence measure $D_{KL}(q_x||p_x)$ between the desired probability distribution $q_x$ for the class $t(x)$ ($q_x[k]\eqd \bm{1}_{k=t(x)}$) and the defined above probability distribution $p_x$,
 \item 
$\boxed{\od{\cl{E}}{y} = p_x-q_x}$ -- the gradient for {\it SoftMax} loss function wrt $y$, where $y=F(x).$
\end{enumerate}

\subsubsection{Note on SGD\label{sec:note-sgd}}

In order to show where is the main loop of each training algorithm using stochastic gradient optimization, the pseudo-codes for the classical {\it Momentum} algorithm \cite{RumelhartD88a} and for the recent {\it AdaM} algorithm \cite{Kingma18a} are given below. For both algorithms their accelerated modifications are also shown (Nesterov \cite{Nesterov83a} and AdaM \cite{Kingma18a}).

\noindent$
\begin{array}{l}
Momentum(\od{f}{\theta},\theta_0,\alpha)\ \ \{\\[2pt]
\bm{in:\ } \od{f}{\theta}: \text{batch gradient of stochastic goal function}\\[3pt] 
\bm{in:\ } \theta_0\inv{n}: \text{initial model}\\[3pt] 
\bm{in:\ } \alpha>0: \text{learning rate}\\[3pt] 
\bm{in:\ } \beta\in[0,1): \text{momentum coefficient}\\[3pt] 
\hspace*{5mm} m_0\ass 0_n\ \text{(*gradient mean*)};\ 
t\ass 0\ \text{(*discrete time*)}
\end{array}
$ 

\noindent$
\begin{array}{l}
\hspace*{5mm} \bm{while}\ (\theta_t \text{ is not stabilized})\ \{\\
\hspace*{10mm} 
\begin{array}{l}
t\ass t+1\\
\left.g_t\ass \od{f}{\theta}(\theta_{t-1},x_{(t-1)N_b:tN_b}) \right\}\xrightarrow{Nesterov}\left\{
g_t\ass \od{f}{\theta}(\theta_{t-1}-\beta m_{t-1},x_{(t-1)N_b:tN_b})
\right.\\
m_t\ass \beta m_{t-1}+\alpha g_t;\ \theta_t\ass \theta_{t-1}-m_t
\end{array}\\
\hspace*{5mm}\}\\
\bm{out:\ } \theta_t\\
\}\\
\text{-----------------}
\end{array}
$ 

Algorithm {\em AdaM} (Adaptive Moments) modifies the model $\theta_t$ at the moment $t$ by the exponentially weighted gradient realizations $g_1,\dots,g_t$ being normalized component-wise by the squared roots of exponentially weighted gradient component squares.

In the algorithm below the following notation is used: (a) the square of the vector $g\inv{R}:$ $g^2 \eqd [g(1)^2,\dots,g(n)^2],$ is implemented as the square of each component for $g,$ (b) the component-wise there are also the squared root of the vector  and the division of  two vectors, (c) addition of scalar $\epsilon$ to the vector  $v$ denotes its addition  to each component of $v$, (d) $x_{a:b}$ denotes the sequence of data elements  $x_a,x_{a+1},\dots,x_{b-1}$.

\noindent$
\begin{array}{l}
AdaM(\od{f}{\theta},\theta_0,\alpha,\beta_1,\beta_2,\epsilon)\ \ \{\\[2pt]
\bm{in:\ } \od{f}{\theta}: \text{batch gradient of stochastic goal function}\\ 
\bm{in:\ } \theta_0\inv{n}: \text{initial model}\\ 
\bm{in:\ } \alpha>0: \text{learning rate}\\ 
\bm{in:\ } \beta_1\in[0,1): \text{rate of decay for gradient weight}\\ 
\bm{in:\ } \beta_2\in\big[0,1\big): \text{rate of decay for squared gradient weight}\\ 
\bm{in:\ } \epsilon>0: \text{constant to avoid zero division}\\ 
\hspace*{5mm} m_0\ass 0_n\ \text{(*gradient mean*)};\ 
v_0\ass 0_n\ \text{(*squared gradient mean*)};\ 
t\ass 0\ \text{(*discrete time*)}
\end{array}
$ 

\noindent$
\begin{array}{l}
\hspace*{5mm} \bm{while}\ (\theta_t \text{ is not stabilized})\ \{\\
\hspace*{10mm} 
\begin{array}{l}
\ \ t\ass t+1;\ g_t\ass \od{f}{\theta}(\theta_{t-1},x_{(t-1)N_b:tN_b})\\
\ \ m_t\ass \beta_1m_{t-1}+(1-\beta_1)g_t;\ v_t\ass \beta_2v_{t-1}+(1-\beta_2)g_t^2\\[2pt]
\left.\begin{array}{l}
\hat{m}_t\ass m_t/(1-\beta_1^t);\ \hat{v}_t\ass v_t/(1-\beta_2^t)\\
\theta_t\ass \theta_{t-1}-\alpha\hat{m}_t/(\sqrt{\hat{v}_t}+\epsilon)
\end{array}\right\}
\xrightarrow{\hat{\epsilon}\eqd \epsilon\sqrt{1-\beta_2}}
\left\{\begin{array}{l}
\alpha_t\ass\alpha\sqrt{1-\beta_2^t}/(1-\beta_1^t)\\
\theta_t\ass \theta_{t-1}-\alpha_tm_t/(\sqrt{v_t}+\hat{\epsilon})
\end{array}\right.
\end{array}\\
\hspace*{5mm}\}\\
\bm{out:\ } \theta_t\\
\}\\
\text{-----------------}
\end{array}
$ 

\noindent Recommended default setting for AdaM are: $\alpha=0.001, \beta_1=0.9, \beta_2=0.999, \epsilon = 10^{-8}.$ However the selection is application dependent.

\paragraph{Quality measures for stochastic online optimization\\}

The algorithm {\it AdaM} is convergent in the stochastic sense, i.e. if the stochastic gradient is bounded to a certain ball then the optimizer {\em AdaM} is {\it accurate in the infinity with degree} $k=\frac{1}{2}$. The mathematical meaning of the concept of accuracy and degree of convergence are defined as follows:

Let the stochastic optimizing problem be defined by the stochastic process of goal function $f_t(\theta)$ with parameters $\theta$ which belong to the certain set of admissible parameters $\cl{A}\subseteq\bb{R}^n.$ Any optimizing algorithm defines the stochastic process of parameters $\theta_t\in\cl{A}$ and implicitly the process of the best parameter $\theta^{\ast}_t$ in time interval $\big[1,t\big]:$ 
\begin{equation}\label{eq:regret-1}
\ds\theta^{\ast}_t \eqd \arg\min_{\theta\in\cl{A}}\frac{\sum_{\tau=1}^tf_{\tau}(\theta)}{t}
\end{equation}
The quality of the optimizing algorithm working in online mode can be measured by the stochastic process of inaccuracy and its limit with probability one:
\begin{enumerate}
  \item {\em stochastic process of inaccuracy} $R_t$ (fancy called regret) is defined by the expected inaccuracy in the time interval $\big[1,t\big]:$
  \begin{equation}\label{eq:regret-2}
  \begin{array}{rcl}
R_0 & \eqd & 0\\[5pt]  
R_t & \eqd & \frac{(t-1)R_{t-1}+f_t(\theta_t)-f_t(\theta_t^{\ast})}{t}\\
& & t=1,\dots
\end{array}
  \end{equation}
  \item {\em limit in the infinity of the inaccuracy process}:\\ $\ds R_{\infty} = \lim_{t\ra\infty}R_t,$ where the convergence occurs with the probability one
\item {\em degree of accuracy:}
{\em the stochastic optimizer is accurate in the infinity with degree $k$, $k\in\bb{R}_+$} iff $R_{\infty}=0$ and $R_t = O(t^{-k})$, i.e. the inaccuracy sequence tends to zero with probability one at least as fast as the sequence $\frac{1}{t^k}.$
\end{enumerate}

\subsubsection{Note on GAN\label{sec:note-gan}}
By the {\it mutatedGAN} lossy function and the optimization algorithm we mean any proposal described in the literature which is based on the probabilistic concept originally described by Goodfellow et al. in their seminal paper \cite{Goodfellow14a}: {\it Generative Adversarial Networks}. GAN consists of two parts: generator $G$ which solves an image transformation problem (e.g. image lossy compression, image high-resolution synthesis, media embedding etc.), discriminator $D$ which computes the probability $D(x)$ that $x$ is coming from the original data source, i.e. it is not generated by $G$. Then $D(x)$ should be maximized for any data $x$ and $D(G(z))$ should be statistically minimized for some latent random data source $z$  with uniform distribution.
 Then the optimal combination of $(G,D)$ creates the Nash equilibrium point:
\begin{equation}
(G^{\ast},D^{\ast}) = \arg\min_G\arg\max_D\left[
\bb{E}_x[\log D(\tilde{x})]+\bb{E}_z[\log(1-D(G(\tilde{z}))]
\right]
\end{equation}
The authors \cite{Goodfellow14a} proved the probabilistic convergence of a generic algorithm which interleaves in batches SGA (stochastic gradient ascent) for the gain function $\log D(x)+\log(1-D(G(z))$ and SGD ((stochastic gradient descent) for the loss function $\bb{E}_z[\log(1-D(G(z)))].$

The mutations of GAN will be referred at the digital media compression (section \ref{sec:compress}) and embedding  (section \ref{sec:embed}) applications. The proposals mainly extend the original probabilistic goal function for GAN by application specific terms generalizing the concept of the generator while keeping the idea of the discriminator. Moreover, since the original optimization procedure appeared to be unstable for digital media applications, the authors provide various heuristic solutions to avoid the problem: from design of special units to modification of interaction between SGD and SGA phases in the GAN's original concept. 

For digital media compression and embedding the ideas coming from applications on image color transfer (cf. Ulyanov \cite{Ulyanov16a}) where the instance normalization replaces the batch normalization, appeared most effective. It means also that the batch gradient is replaced by the instant gradient and in the optimizer, the SGA phase is interleaved with the SGD phase on the basis of single elements drawn randomly from the training data base.

\subsection{Example: VGG-16 architecture in SyblicNet notation\label{sec:vgg}}

K. Simonyan and  A. Zisserman described in their seminal paper \cite{SimonyanZ14a},
{\it Very Deep Convolutional Networks for Large-Scale Image Recognition}, presented the network which now serves for the community as the universal image feature extractor. Those deep features are computed by the VGG-16 network in the unit tensors identified in STNN notation by labels $vgg_i$ with indexes $i=1,\dots,5.$ The original training algorithm for VGG-16 is interested only in the output denoted by the label $out.$ From the size of the last full connection layer we see that VGG-16 can recognize image objects from (up to) $1000$ classes.

\begin{description}
\item[VGG-16 as unstructured STNN:]$ $\\[5pt] 
\hspace*{-12mm}
\doublebox{
\begin{tabular}{l}
\xin{yx}{3}{rgb}
\xconv{3}{64}{}{}{r}
\xconv{3}{64}{}{}{r}
\xpool{2}{}{m}{}{}
\xtoreflabelto{vgg_1}
\xconv{3}{128}{}{}{r}
\xconv{3}{128}{}{}{r}
\xpool{2}{}{m}{}{}
\xtoreflabelto{vgg_2}\\[5pt]
\xconv{3}{256}{}{}{r}
\xconv{3}{256}{}{}{r}
\xconv{3}{256}{}{}{r}
\xpool{2}{}{m}{}{}
\xtoreflabelto{vgg_3}
\xconv{3}{512}{}{}{r}
\xconv{3}{512}{}{}{r}
\xconv{3}{512}{}{}{r}
\xpool{2}{}{m}{}{}
\xtoreflabelto{vgg_4}\\[5pt]
\xconv{3}{512}{}{}{r}
\xconv{3}{512}{}{}{r}
\xconv{3}{512}{}{}{r}
\xpool{2}{}{m}{}{}
\xtoreflabelto{vgg_5}
\xdense{}{4096}{}{}{}
\xdense{}{4096}{}{}{}
\xdense{}{1000}{}{}{s}
\xtolabel{score}\\[5pt]
\hline
\xbound{vgg}{1}{
\begin{array}{l}
rgb := 112_{xy}3_c;\ 
optima := [loss, MomentumSGD, SoftMax\ \eqref{eq:soft-max-loss}]
\end{array}
}\\
\xbound{vgg}{2}{
\begin{array}{l}
rgb := 224_{xy}3_c;\ optima := [loss, MomentumSGD, SoftMax\  \eqref{eq:soft-max-loss}]
\end{array}
}
\end{tabular}
}
\item[VGG-16 as structured STNN:\label{item:vgg-16}] User units definitions and their instances:\\[5pt] 
\hspace*{-10mm}
\framebox{
\begin{tabular}{l}
\xunitdef{c2}{
\begin{tabular}{l}
\xexpression{f = 64\cdot 1_{\$}}\\
\hline
\xconv{3}{f}{}{}{r}
\xconv{3}{f}{}{}{r}
\xpool{2}{}{m}{}{}
\end{tabular}
}\ \ 
\xunitinstance{c2}{1}{1}\ \ 
\xunitinstance{c2}{2}{2}
\\[15pt]
\xunitdef{c3}{
\begin{tabular}{l}
\xexpression{f = 256\cdot 1_{\$}}\\
\hline
\xconv{3}{f}{}{}{r}
\xconv{3}{f}{}{}{r}
\xconv{3}{f}{}{}{r}
\xpool{2}{}{m}{}{}
\end{tabular}
}\ \ 
\xunitinstance{c3}{1}{1}\ \ 
\xunitinstance{c3}{2}{2}
\\[15pt]
\xunitdef{fc}{
\xdense{}{1_{\$}}{}{}{}
}\ \ 
\xunitinstance{fc}{1}{4096}\ \ 
\xunitinstance{fc}{2}{1000}
\end{tabular}
}
\item[VGG-16 as structured STNN:] The main architecture -- 
the regular output unit used for training and testing is $O_{out}$ while labels $vgg_i,$ $i=1,\dots,5$, are added to be used for training of other CNN architectures:\\[5pt]
\hspace*{-10mm}
\doublebox
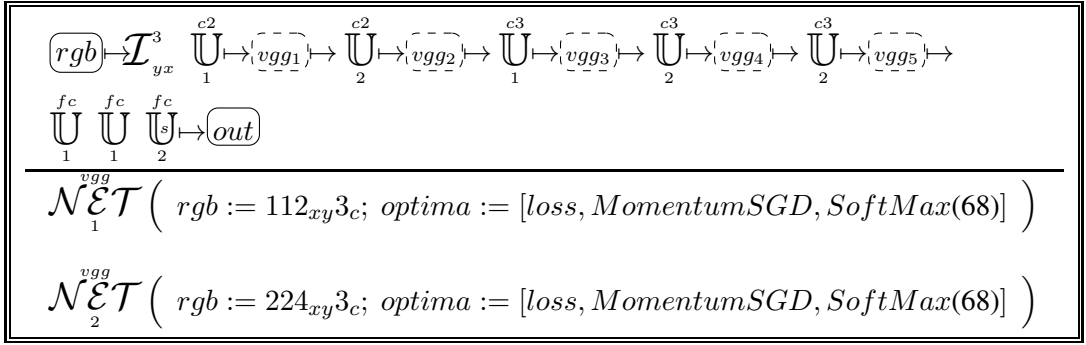
{
\begin{tabular}{l}
\xin{yx}{3}{rgb}
\xunit{c2}{1}{}
\xtoreflabelto{vgg_1}
\xunit{c2}{2}{}
\xtoreflabelto{vgg_2}
\xunit{c3}{1}{}
\xtoreflabelto{vgg_3}
\xunit{c3}{2}{}
\xtoreflabelto{vgg_4}
\xunit{c3}{2}{}
\xtoreflabelto{vgg_5}\\[5pt]
\xunit{fc}{1}{}
\xunit{fc}{1}{}
\xunit{fc}{2}{s}
\xtolabel{out}
\\[5pt]
\hline
\xbound{vgg}{1}{
\begin{array}{l}
rgb := 112_{xy}3_c;\ optima := [loss, MomentumSGD, SoftMax \eqref{eq:soft-max-loss}]
\end{array}
}\\[15pt]
\xbound{vgg}{2}{
\begin{array}{l}
rgb := 224_{xy}3_c;\ optima := [loss, MomentumSGD, SoftMax \eqref{eq:soft-max-loss}]
\end{array}
}
\end{tabular}
}
\end{description}

\subsection{Example: dual symbolic neural network and its BNF grammar}

\begin{description}
\item[Dual VGG-16 as unstructured STNN:]$ $\\[5pt] 
\hspace*{-10mm}
\doublebox
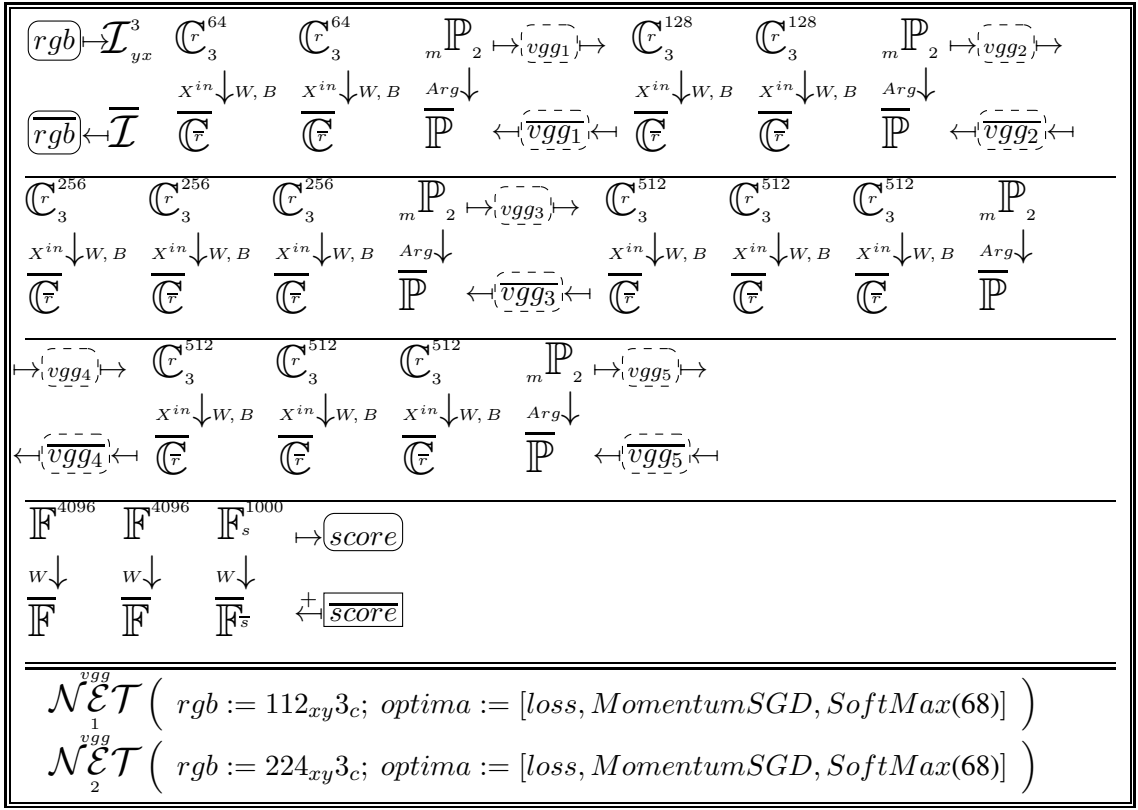
{
\begin{tabular}{l}
\tabxin{yx}{3}{rgb}
\tabxconv{3}{64}{}{}{r}
\tabxconv{3}{64}{}{}{r}
\tabxpool{2}{}{m}{}{}
\tabxtoreflabelto{vgg_1}
\tabxconv{3}{128}{}{}{r}
\tabxconv{3}{128}{}{}{r}
\tabxpool{2}{}{m}{}{}
\tabxtoreflabelto{vgg_2}\\[25pt]\hline
\tabxconv{3}{256}{}{}{r}
\tabxconv{3}{256}{}{}{r}
\tabxconv{3}{256}{}{}{r}
\tabxpool{2}{}{m}{}{}
\tabxtoreflabelto{vgg_3}
\tabxconv{3}{512}{}{}{r}
\tabxconv{3}{512}{}{}{r}
\tabxconv{3}{512}{}{}{r}
\tabxpool{2}{}{m}{}{}\\[25pt]\hline
\tabxtoreflabelto{vgg_4}
\tabxconv{3}{512}{}{}{r}
\tabxconv{3}{512}{}{}{r}
\tabxconv{3}{512}{}{}{r}
\tabxpool{2}{}{m}{}{}
\tabxtoreflabelto{vgg_5}\\[25pt]\hline
\tabxdense{}{4096}{}{}{}
\tabxdense{}{4096}{}{}{}
\tabxdense{}{1000}{}{}{s}
\tabxtolabel{score}\\[25pt]\hline\hline
\xbound{vgg}{1}{
\begin{array}{l}
rgb := 112_{xy}3_c;\ optima := [loss, MomentumSGD, SoftMax \eqref{eq:soft-max-loss}]
\end{array}
}\\
\xbound{vgg}{2}{
\begin{array}{l}
rgb := 224_{xy}3_c;\ optima := [loss, MomentumSGD, SoftMax \eqref{eq:soft-max-loss}]
\end{array}
}
\end{tabular}
}
\item[Dual VGG-16 as structured STNN:] User units definitions and their instances:\\[5pt] 
\hspace*{-10mm}
\framebox
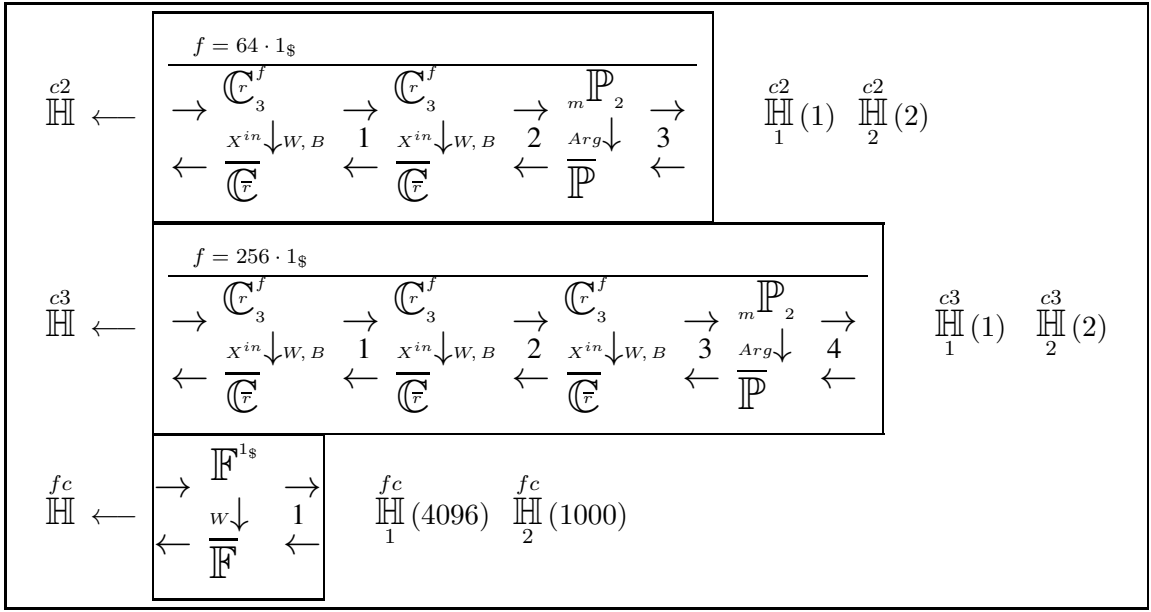
{
\begin{tabular}{l}
\xhybridunitdef{c2}{
\begin{tabular}{l}
\xexpression{f = 64\cdot 1_{\$}}\\
\hline
\lrtabbeg
\tabxconv{3}{f}{}{}{r}{1}
\lrtab
\tabxconv{3}{f}{}{}{r}{2}
\lrtab
\tabxpool{2}{}{m}{}{}{3}
\lrtabend
\end{tabular}
}
\quad\xhybridunitinstance{c2}{1}{1}\ \ \xhybridunitinstance{c2}{2}{2}
\\[30pt]
\xhybridunitdef{c3}{
\begin{tabular}{l}
\xexpression{f = 256\cdot 1_{\$}}\\
\hline
\lrtabbeg
\tabxconv{3}{f}{}{}{r}{1}
\lrtab
\tabxconv{3}{f}{}{}{r}{2}
\lrtab
\tabxconv{3}{f}{}{}{r}{3}
\lrtab
\tabxpool{2}{}{m}{}{}{4}
\lrtabend
\end{tabular}
}
\quad\xhybridunitinstance{c3}{1}{1}
\ \ \xhybridunitinstance{c3}{2}{2}
\\[30pt]
\xhybridunitdef{fc}{
\lrtabbeg
\tabxdense{}{1_{\$}}{}{}{}{1}
\lrtabend
}
\quad\xhybridunitinstance{fc}{1}{4096}\ \ \xhybridunitinstance{fc}{2}{1000}
\end{tabular}
}

\item[Dual VGG-16 as structured STNN:] The main architecture -- 
the regular output unit used for training and testing is $O_{out}$ while labels $vgg_i,$ $i=1,\dots,5$, are added to be used for training of other CNN architectures:\\[5pt]
\hspace*{-10mm}
\doublebox
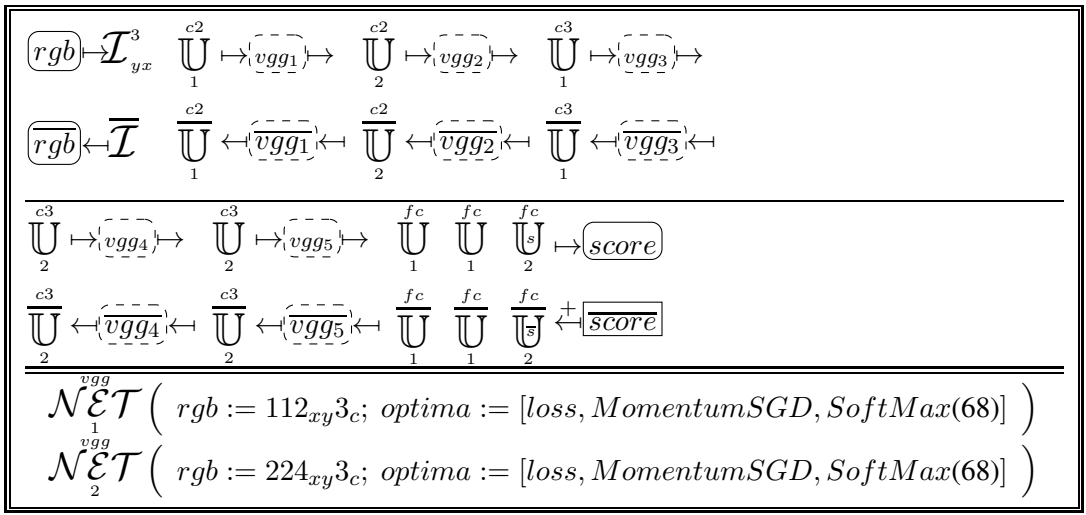
{
\begin{tabular}{l}
\tabxin{yx}{3}{rgb}
\tabxunit{c2}{1}{}
\tabxtoreflabelto{vgg_1}
\tabxunit{c2}{2}{}
\tabxtoreflabelto{vgg_2}
\tabxunit{c3}{1}{}
\tabxtoreflabelto{vgg_3}\\[30pt]
\hline
\tabxunit{c3}{2}{}
\tabxtoreflabelto{vgg_4}
\tabxunit{c3}{2}{}
\tabxtoreflabelto{vgg_5}
\tabxunit{fc}{1}{}
\tabxunit{fc}{1}{}
\tabxunit{fc}{2}{s}
\tabxtolabel{score}
\\[5pt]
\hline\hline
\xbound{vgg}{1}{
\begin{array}{l}
rgb := 112_{xy}3_c;\ optima := [loss, MomentumSGD, SoftMax \eqref{eq:soft-max-loss}]
\end{array}
}\\
\xbound{vgg}{2}{
\begin{array}{l}
rgb := 224_{xy}3_c;\ optima := [loss, MomentumSGD, SoftMax \eqref{eq:soft-max-loss}]
\end{array}
}
\end{tabular}
}
\end{description}

\subsection{BNF grammar for joined STNN and DSTNN}

The elements of dual symbolic tensor neural network in general depend strongly on their original counterparts. The action of dual unit processing is defined if we get not only the dual input data, the parameters of the unit, but usually the input data of the original unit. Therefore making the separate rules though it is possible, brings little benefit for understanding the error backward flow. The joined (hybrid) definition of BNF for STNN and DSTNN is more relevant to this goal.

Joined BNF definition is the sequence of rule definitions for the original elements together with their dual elements:
\[
\begin{array}{rcl}
\otabx{SyntacticElement}{\ov{Syntactic Element}}{14pt}{10mm}
& ::= &
\otabxx{CompositionRule}{Referenced}{Data}{\ov{CompositionRule}}{10mm}
\end{array}
\]

The composition rules and and the dual composition rule are again algebraic expressions, i.e. the sequence of syntactic operators applied to syntactic elements and their duals. We use the same following syntactic operators as for STNN.


\begin{description}
  \item [Units:]$ $\\
\begin{description}
\item [Custom units:]  
\[
\begin{array}{l}
\begin{array}{rcl}
\otabx{CustomUnit}{\ov{CustomUnit}}{12pt}{10mm}        & ::= & 
\left.
\mymkkh{\hspace*{8mm}\tabxconv{a}{b}{c}{d}{e}{}}{12mm}
\quad\right|\quad 
\left.
\mymkkh{\hspace*{8mm}\tabxdense{a}{b}{c}{d}{e}{}}{12mm}
\quad\right|\quad 
\left.
\mymkkh{\hspace*{8mm}\tabxpool{a}{b}{c}{d}{e}{}}{12mm}
\quad\right|\quad \dots \\[35pt]
&&
\left.
\mymkkh{\hspace*{8mm}\tabxrelu}{12mm}
\quad\right|\quad 
\left.
\mymkkh{\hspace*{8mm}\tabxrelup{k}}{12mm}
\quad\right|\quad 
\left.
\mymkkh{\hspace*{8mm}\tabxbatch}{12mm}
\quad\right|\quad 
\dots \\[35pt]
&&
\left.
\mymkkh{\hspace*{8mm}\tabxsigmo}{12mm}
\quad\right|\quad 
\left.
\mymkkh{\hspace*{8mm}\tabxtanh}{12mm}
\quad\right|\quad 
\dots
\end{array}
\\[70pt]
\text{\scriptsize where\ 
$\left\{
\begin{array}{l}
(abcd):\ unit\ parametrization\ fields\ (used\ if\ defaults\ are\ to\ be\ changed),\\
(e):\ field\ for\ nesting\ unit\ symbols\ (e.g.\ r\ for\ \cl{R}),\\
(X^{in}, Arg):\ data\ referenced\ from\ unit\ resources,\\
(W, B, \Sigma):\ parametric\ data\ referenced\ from\ unit\ resources.
\end{array}
\right.$}
\end{array}
\]

\item [User defined units:]  
\[
\begin{array}{rcl}
\otabx{UserUnit}{\ov{UserUnit}}{12pt}{10mm}  & ::= &
\mymkkh{\hspace*{8mm}\tabxunit{a}{b}{c}}{12mm}
\quad \text{\scriptsize where\ 
$\left\{
\begin{array}{l}
(a)\ unit\ class\ name\ (mandatory),\\
(b)\ instance\ id\ (optional),\\
(c)\ nested\ symbol\ (optional).
\end{array}
\right.
$
}
\end{array}
\]
\item [All units:]  
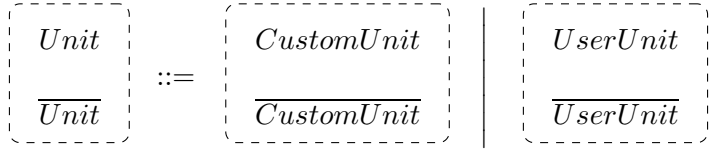
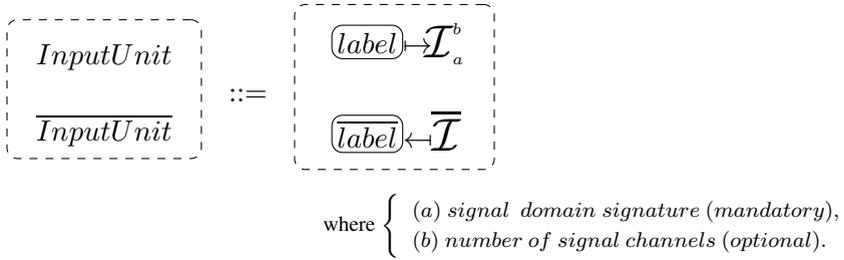
\[
\begin{array}{rcl}
\otabx{Unit}{\ov{Unit}}{12pt}{10mm}
& ::= & 
\left.\otabx{CustomUnit}{\ov{CustomUnit}}{12pt}{10mm}
\quad \right|\quad
\otabx{UserUnit}{\ov{UserUnit}}{12pt}{10mm}
\end{array}
\]
\end{description}
  \item [Inputs:]
\[
\begin{array}{rcl}
\otabx{InputUnit}{\ov{InputUnit}}{12pt}{10mm}
 &::=& 
 \mymkkh{\hspace*{8mm}\tabxin{a}{b}{label}}{12mm}\\[30pt]&&
 \quad \text{\scriptsize where\ 
$\left\{
\begin{array}{l}
(a)\ signal\ \ domain\ signature\ (mandatory),\\
(b)\ number\ of\ signal\ channels\ (optional).
\end{array}
\right.
$
}
\end{array}
\]
  
  \item [Labels:]$ $\\
\begin{description}
\item [to use in links:]  
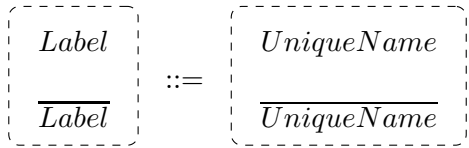
\[
\begin{array}{rcl}
\otabx{Label}{\ov{Label}}{12pt}{10mm}
& ::= & 
\otabx{UniqueName}{\ov{UniqueName}}{12pt}{10mm}
\end{array}
\]
\item [to enumerate user unit inputs:]  
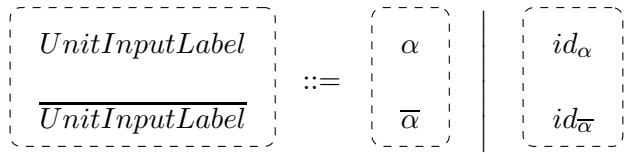
\[
\begin{array}{rcl}
\otabx{UnitInputLabel}{\ov{UnitInputLabel}}{12pt}{10mm}
& ::= & 
\left.\otabx{\alpha}{\ov{\alpha}}{12pt}{10mm}
\quad\right|\quad
\otabx{id_{\alpha}}{id_{\ov{\alpha}}}{12pt}{10mm}
\end{array}
\]
\item [to enumerate user unit outputs:]  
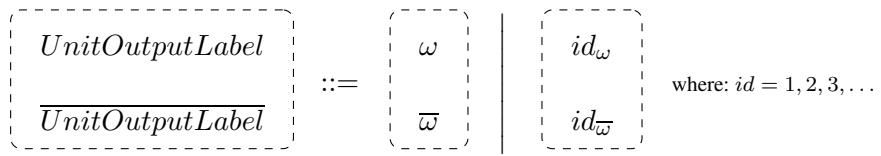
\[
\begin{array}{rcl}
\otabx{UnitOutputLabel}{\ov{UnitOutputLabel}}{12pt}{10mm}
& ::= & 
\left.\otabx{\omega}{\ov{\omega}}{12pt}{10mm}
\quad\right|\quad
\otabx{id_{\omega}}{id_{\ov{\omega}}}{12pt}{10mm}
\quad\text{\scriptsize where:\ $id=1,2,3,\dots$}
\end{array}
\]
\end{description}

\item [Links:]$ $\\

\begin{description}
\item [to multi-cast unit output tensor:]
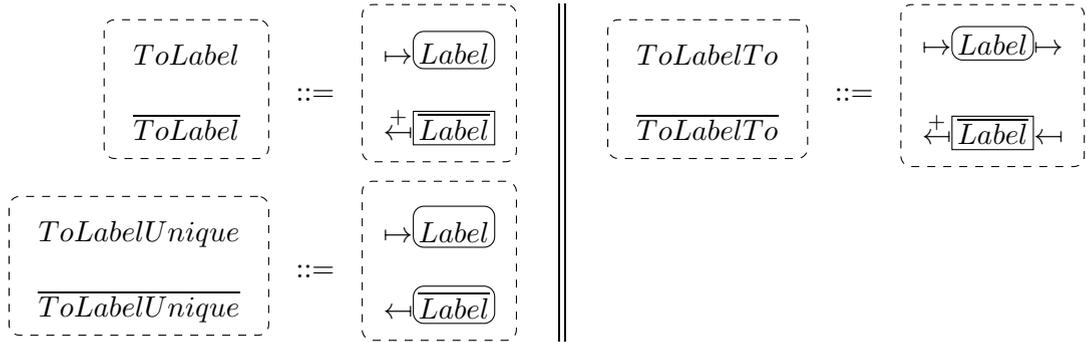
\[
\hspace*{-10mm}
\begin{array}{rcl||rcl}
\otabx{ToLabel}{\ov{ToLabel}}{12pt}{10mm}
 & ::= & 
 \mymkkh{\hspace*{8mm}\tabxtolabel{Label}}{12mm}
  \quad&\quad
\otabx{ToLabelTo}{\ov{ToLabelTo}}{12pt}{10mm}
 & ::= & 
 \mymkkh{\hspace*{8mm}\tabxtolabelto{Label}}{12mm}
 \\[25pt]
\otabx{ToLabelUnique}{\ov{ToLabelUnique}}{12pt}{10mm}
      & ::= & 
 \mymkkh{\hspace*{8mm}\tabxtolabelunique{Label}}{12mm}
&& 
\end{array}
\]

\item [to uni-cast unit outputs:]
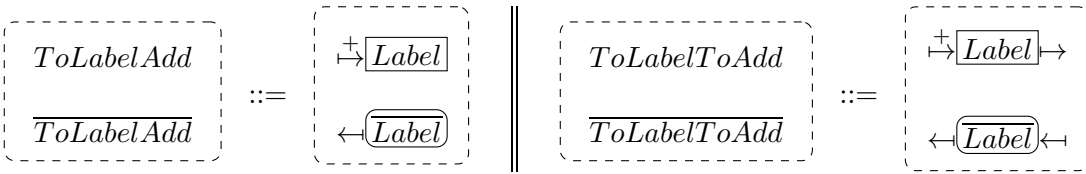
\[
\hspace*{-10mm}
\begin{array}{rcl||rcl}
\otabx{ToLabelAdd}{\ov{ToLabelAdd}}{12pt}{10mm}
 & ::= & \mymkkh{\hspace*{8mm}\tabxtolabeladd{Label}}{12mm}
 \quad&\quad
\otabx{ToLabelToAdd}{\ov{ToLabelToAdd}}{12pt}{10mm}
 & ::= & 
 \mymkkh{\hspace*{8mm}\tabxtolabeltoadd{Label}}{12mm}
\end{array}
\]

\item [to request unit output:]
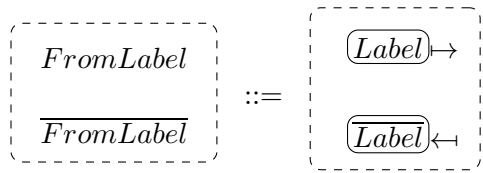
\[
\begin{array}{rcl}
\otabx{FromLabel}{\ov{FromLabel}}{12pt}{10mm}
    & ::= & 
\mymkkh{\hspace*{8mm}\tabxfromlabel{Label}}{12mm}
\end{array}
\]
\item [to refer to unit output:]
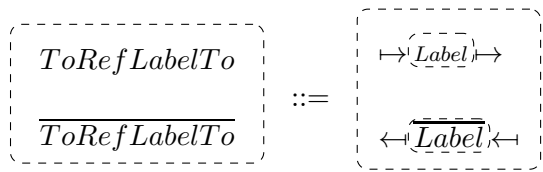
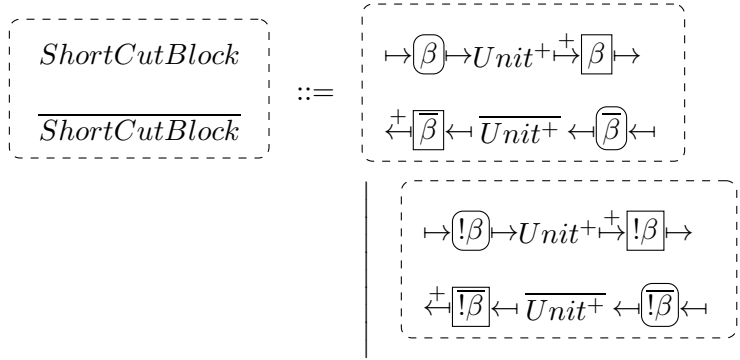
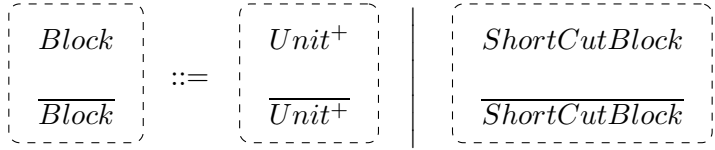
\[
\begin{array}{rcl}
\otabx{ToRefLabelTo}{\ov{ToRefLabelTo}}{12pt}{10mm}
      & ::= & 
 \mymkkh{\hspace*{8mm}\tabxtoreflabelto{Label}}{12mm}
\end{array}
\]
\end{description}
\item [Blocks:]$ $\\
\begin{description}
\item [Residual blocks:]
\[
\begin{array}{rcl}
\otabx{ShortCutBlock}{\ov{ShortCutBlock}}{12pt}{10mm}
 & ::= & 
 \mymkkh{\hspace*{8mm}\tabxshortcut{$\ Unit^+\ $}
 {$\ \ov{Unit^+}\ $}{\beta}}{12mm} 
 \\[25pt]&&
 \left|\quad
 \mymkkh{\hspace*{8mm}\tabxshortcut{$\ Unit^+\ $}
 {$\ \ov{Unit^+}\ $}{!\beta}}{12mm} 
 \right.
 \end{array}
\]
\item [All blocks:]
\[
\begin{array}{rcl} 
\otabx{Block}{\ov{Block}}{12pt}{10mm}
 & ::= & 
\left. 
\otabx{Unit^+}{\ov{Unit^+}}{12pt}{10mm}  
 \quad\right|\quad 
\otabx{ShortCutBlock}{\ov{ShortCutBlock}}{12pt}{10mm} 
 \end{array}
\]
\end{description}

\item [Merge/Split:]$ $\\
\begin{description}
\item [to merge unit inputs:]
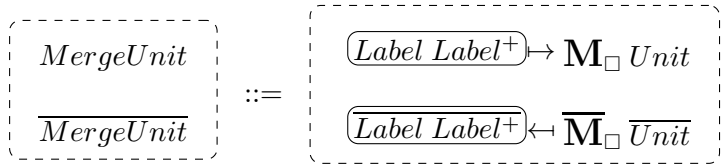
\[
\begin{array}{rcl}
\otabx{MergeUnit}{\ov{MergeUnit}}{12pt}{10mm}
    & ::= & 
\mymkkh{\hspace*{8mm}\tabxmergeunit{Label\ Label^+}{\Box}{Unit}}{12mm}
\end{array}
\]
\item [to split unit output:]
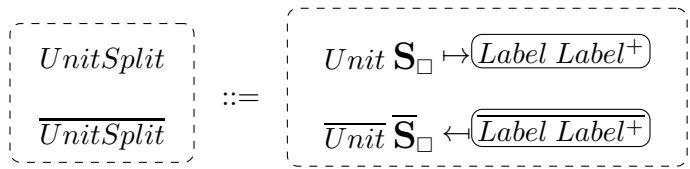
\[
\begin{array}{rcl}
\otabx{UnitSplit}{\ov{UnitSplit}}{12pt}{10mm}
    & ::= & 
\mymkkh{\hspace*{8mm}\tabxunitsplit{\Box}{Label\ Label^+}{Unit}}{12mm}
\end{array}
\]
  \end{description}

  \item [Components:]$ $\\
\begin{description}
\item [to begin component:]  
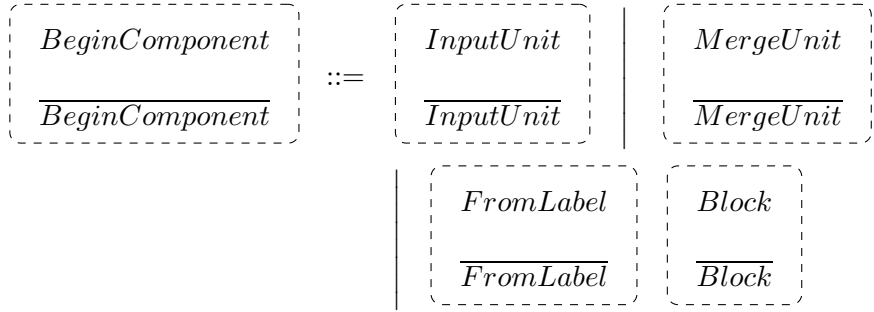
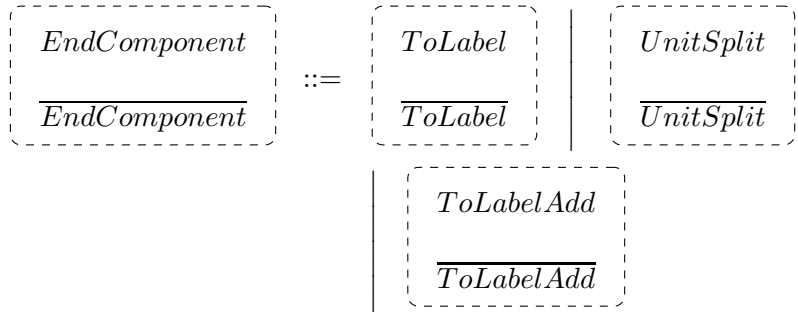
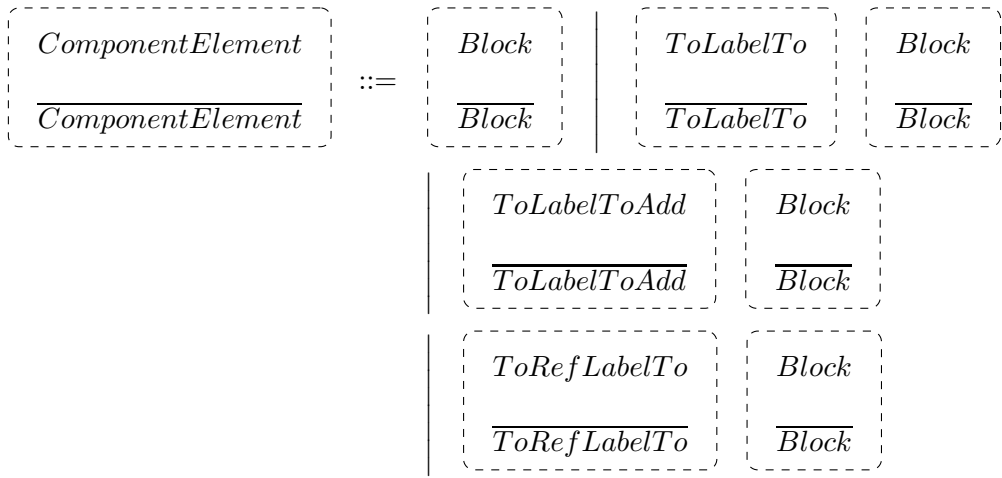
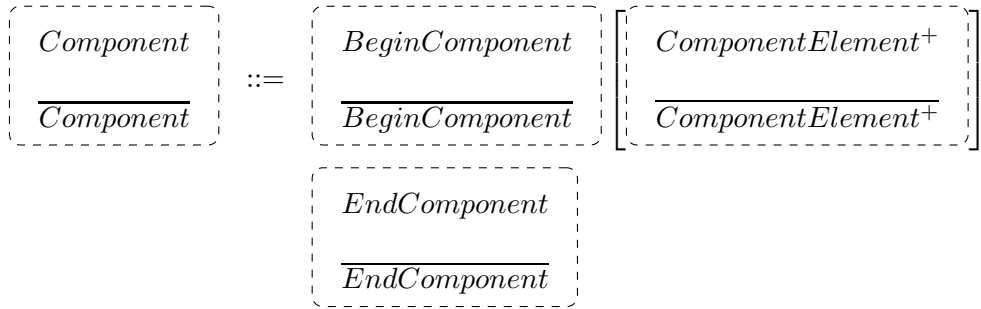
\[
\begin{array}{rcl}
\otabx{BeginComponent}{\ov{BeginComponent}}{12pt}{10mm}
 & ::= & 
 \left.
\otabx{InputUnit}{\ov{InputUnit}}{12pt}{10mm} 
\quad\right|\quad
\otabx{MergeUnit}{\ov{MergeUnit}}{12pt}{10mm} 
\\[25pt]
&&\left|\quad
\otabx{FromLabel}{\ov{FromLabel}}{12pt}{10mm}\quad
\otabx{Block}{\ov{Block}}{12pt}{10mm}
\right.
\end{array}
\]

\item [to end component:]  
\[
\begin{array}{rcl}\otabx{EndComponent}{\ov{EndComponent}}{12pt}{10mm}
   & ::= & 
\left.
\otabx{ToLabel}{\ov{ToLabel}}{12pt}{10mm}    
\quad\right|\quad
\otabx{UnitSplit}{\ov{UnitSplit}}{12pt}{10mm} 
\\[25pt]
&&\left|\quad
\otabx{ToLabelAdd}{\ov{ToLabelAdd}}{12pt}{10mm} 
\right.
\end{array}
\]

\item [to fill component:]  
\[
\begin{array}{rcl}
\otabx{ComponentElement}{\ov{ComponentElement}}{12pt}{10mm}
 & ::= &
 \left.
\otabx{Block}{\ov{Block}}{12pt}{10mm} 
\quad\right|\quad
\otabx{ToLabelTo}{\ov{ToLabelTo}}{12pt}{10mm}\quad 
\otabx{Block}{\ov{Block}}{12pt}{10mm} 
\\[25pt]
&&\left|\quad
\otabx{ToLabelToAdd}{\ov{ToLabelToAdd}}{12pt}{10mm}\quad 
\otabx{Block}{\ov{Block}}{12pt}{10mm} 
\right.
\\[25pt]
&&\left|\quad
\otabx{ToRefLabelTo}{\ov{ToRefLabelTo}}{12pt}{10mm}\quad
\otabx{Block}{\ov{Block}}{12pt}{10mm} 
\right. 
\end{array}
\]
\item [to assemble components:]
\[
\begin{array}{rcl}
\otabx{Component}{\ov{Component}}{12pt}{10mm}
    & ::= & 
\otabx{BeginComponent}{\ov{BeginComponent}}{12pt}{10mm} 
\left[
\otabx{ComponentElement^+}{\ov{ComponentElement^+}}{12pt}{10mm}
\right]
\\[25pt]
&&\otabx{EndComponent}{\ov{EndComponent}}{12pt}{10mm}
\end{array}
\]
\end{description}
  
\item [User defined hybrid units:]$ $\\
\begin{description}
\item [to instance hybrid unit:]
\[
\begin{array}{rcl}
UserHybridUnitInstance      & ::= &  
\xhybridunitinstance{unit\ name}{instance\ id}{free\ labels\ bounding}\\[10pt]
\end{array}
\]

\item [to define hybrid unit:]
\[
\begin{array}{rcl}
UserHybridUnitDefinition    & ::= &  \xhybridunitdef{unit\ name}{$\mymkkh{
\begin{array}{l}
Component\\[5pt]
\ov{Component}
\end{array}
}{7mm}^{\uplus}$}\\[15pt]
&&\quad[UserHybridUnitInstance^+]\\[10pt]
UserHybridUnitDefinitions &::=& \boxed{UserHybridUnitDefinition^+}
\end{array}
\]

\item [to see parts of hybrid unit:]
\[
\begin{array}{rcl}
HybridUnitEquivalence    & ::= & 
\overset{xxx}{\text{\Large $\bb{H}$}}\ \equiv
\quad
\mymkkh{
\begin{array}{l}
\overset{xxx}{\text{\Large $\bb{U}$}}\\[1pt]
\overset{xxx}{\text{\Large $\ov{\bb{U}}$}}
\end{array}
}{10mm}
\\[25pt]
HybridUnitInstanceEquivalence    & ::= &
\underset{i}{\overset{xxx}{\text{\Large $\bb{H}$}}}\ \equiv
\quad
\mymkkh{
\begin{array}{l}
\underset{i}{\overset{xxx}{\text{\Large $\bb{U}$}}}\\[7pt]
\underset{i}{\overset{xxx}{\text{\Large $\ov{\bb{U}}$}}}
\end{array}
}{12mm}
\end{array}
\]
\end{description}
  
  \item [Symbolic hybrid net:]
\[
\begin{array}{rcl}
SyblicNetInstance & ::= & \xbound{net\ name}{net\ id}{network\ hyper\ parameters}\\[5pt]
SyblicHybridNet   & ::= & 
[UserHybridUnitDefinitions]\ 
\mymkkh{
\begin{array}{l}
Component\\[5pt]
\ov{Component}
\end{array}
}{7mm}^{\uplus}\\[5pt]
&& SyblicNetInstance^{\uplus}
\end{array}
\]
  
\end{description}

\subsection{Neural modules in symbolic notation -- examples}

\subsubsection{Inception module\label{sec:incept}}

The {\it inception module} combines results of several convolution layers with the fixed kernel sizes (e.g. $1\times 1,$ $3\times 3$, $5\times 5$) and one max pooling layer with small window. It was introduced by Szegedy et al. in the paper \cite{SzegedyC14a}: {\it Going Deeper with Convolutions}.

In order to keep the time complexity reasonable, the convolution with higher sizes are preceded by trivial convolutions with masks of size $1\times 1$ in the reduced amount of filters. For the pooling layer the reduction of filters is performed after the pooling. The inception module is a good opportunity to show how the merging unit and its dual (the splitting unit) are used to create reusable symbolic definition in STNN. For all convolutions the batch normalization and the ReLU activation are added.  

\paragraph{Inception module in STNN\\}

Assuming square domains for convolution kernels and square pooling window, let the hyper-parameter vector $w$ of their widths in signal domain (with pooling window as the last one) represents the inception structure. Typically number of branches $n_b=4$ and $w=[1,3,5,3].$ 

Let $f$ be the hyper-parameter vector specifying number of filters (aka depth) for each convolution layer. The arrangement of elements in $f$ corresponds to the arrangement of elements $w_i$, $i=0,\dots,n_b-1$ respecting the rules with the initial $j\ass 0$:
\begin{enumerate}
  \item if $w_i=1$ or $i=n_b-1$ then the single convolution layer occurs in the branch $i$ and $f_j$ is the number of filters for this layer, $j\ass j+1,$
  \item otherwise $f_j$ is the number of filters  for convolution with $1\times 1$ kernel and $f_{j+1}$ is the depth of the convolution layer which follows this depth reduction layer, afterwards $j\ass j+2.$
\end{enumerate}
For example for the above $w=[1,3,5,3]$, the filters hyper-parameters $f=[8,12,12,8,8,4]$ are interpreted as follows:
\begin{enumerate}
  \item the branch $i=0$ includes one convolution layer of kernel $1\times 1$ with $f_0=8$,
  \item the branch $i=1$ includes one convolution layer of kernel $1\times 1$ with $f_1=12$ filters followed by one convolution layer of kernel $3\times 3$ with $f_2=12$ filters,
  \item the branch $i=2$ includes one convolution layer of kernel $1\times 1$ with $f_3=8$ filters followed by one convolution layer of kernel $5\times 5$ with $f_4=12$ filters,
  \item the branch $i=3$ includes one pooling layer with window $3\times 3$ followed by convolution layer of kernel $1\times 1$ with $f_5=4$ filters.
\end{enumerate}
Let us note note that for the above example the inception module does not change the spatial resolution and produces all together 
$(f_0+f_2+f_4+f_5)$ features.

The user defined unit for the inception could be defined as follows:

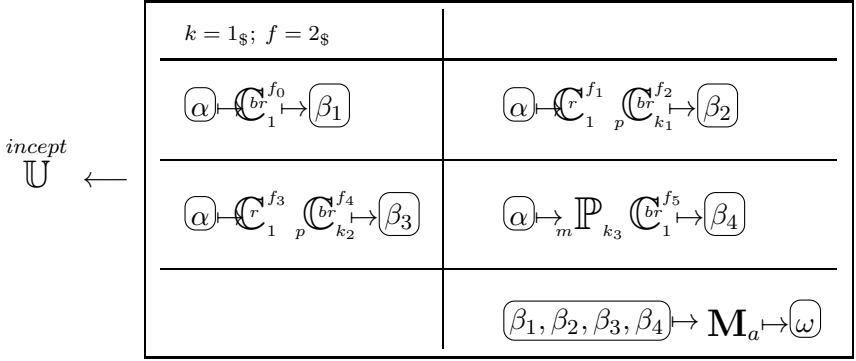
\xunitdef{incept}{
\begin{tabular}{l|l}
\xexpression{k = 1_{\$};\  f = 2_{\$}} 
&\\[2pt]
\hline\\[-8pt]
\xfromlabel{\alpha}\xconv{1}{f_0}{}{}{br}\xtolabel{\beta_1} \quad&\quad
\xfromlabel{\alpha}\xconv{1}{f_1}{}{}{r}\xconv{k_1}{f_2}{p}{}{br}\xtolabel{\beta_2}\\[10pt]
\hline\\[-5pt]
\xfromlabel{\alpha}\xconv{1}{f_3}{}{}{r}\xconv{k_2}{f_4}{p}{}{br}\xtolabel{\beta_3} \quad&\quad
\xfromlabel{\alpha}\xpool{k_3}{}{m}{}{}\xconv{1}{f_5}{}{}{br}\xtolabel{\beta_4}\\[10pt]
\hline\\[-5pt]
\quad&\quad\xmerge{\beta_1,\beta_2,\beta_3,\beta_4}{a}\xtolabel{\omega}
\end{tabular}
}\\

\noindent For example we could get a instance of the above inception as follows:\\ 
\hspace*{5mm}\xunitinstance{incept}{1}{[1,3,5,3], [8,12,12,8,8,4]}

\noindent The hybrid unit of inception has the form:

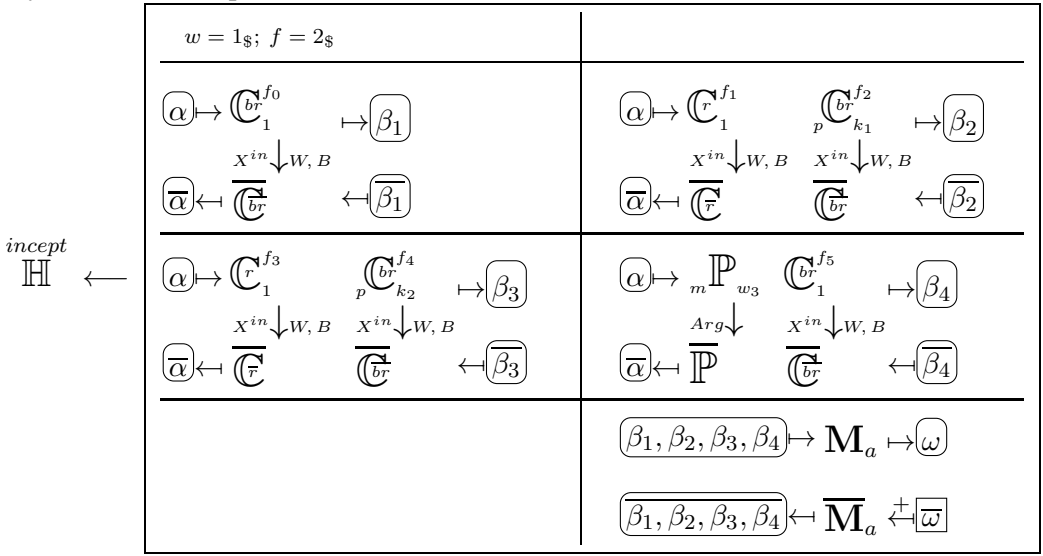
\xhybridunitdef{incept}{
\begin{tabular}{l|l}
\xexpression{w = 1_{\$};\  f = 2_{\$}} &\\[2pt]
\hline\\[-8pt]
\tabxfromlabel{\alpha}\tabxconv{1}{f_0}{}{}{br}
\tabxtolabelunique{\beta_1} \quad&\quad
\tabxfromlabel{\alpha}\tabxconv{1}{f_1}{}{}{r}
\tabxconv{k_1}{f_2}{p}{}{br}\tabxtolabelunique{\beta_2}\\[15pt]
\hline\\[-10pt]
\tabxfromlabel{\alpha}\tabxconv{1}{f_3}{}{}{r}
\tabxconv{k_2}{f_4}{p}{}{br}\tabxtolabelunique{\beta_3} \quad&\quad
\tabxfromlabel{\alpha}\tabxpool{w_3}{}{m}{}{}
\tabxconv{1}{f_5}{}{}{br}\tabxtolabelunique{\beta_4}\\[15pt]
\hline\\[-10pt]
&\quad \tabxmerge{\beta_1,\beta_2,\beta_3,\beta_4}{a}\tabxtolabel{\omega}
\end{tabular}
}

\subsubsection{SPP module\label{sec:spp}}

{\it Spatial Pyramid Pooling} (SPP) is the neural module which combines the outputs of the fixed number of pooling operations. It was proposed by He et al. in the paper \cite{HeK14a}: {\it Spatial Pyramid Pooling in Deep Convolutional Networks for Visual Recognition}.

The pooling spatial windows create the fixed grids (e.g. $3\times 3,$ $2\times 2$, $1\times 1$), organized hierarchically  producing a representation for pyramid of features. SPP is usually an intermediate layer between the feature extraction layers and the classification layers, i.e. after the cascade of CNN-filters and before the final sequence of FC layers.

Let the grid of pooling windows has the resolution $g\times g$ for the tensor spatial domain of resolution $n_x\times n_y$. Then the pooling  strides $s^x,s^y$ and window size $w\times h$ are defined by the formulas:
\begin{equation}
\left[s^x\eqd\left\lfloor\frac{n_x}{g}\right\rfloor,\ s^y\eqd\left\lfloor\frac{n_y}{g}\right\rfloor\right] \lra \left[w = s^x+n_x\bmod g,\ h=s^y+n_y\bmod g\right]
\end{equation} 

Let us notice that in the original paper [He, Zhang et al.]
the authors define the pooling stride $s\eqd\lfloor n/g\rfloor,$ and the pooling window width $w\eqd\lceil n/g\rceil.$ Since 
\[
\lfloor n/g\rfloor\cdot(g-1)+\lceil n/g\rceil = n\ \text{ iff }
n\bmod g<2
\]
then the ignored by the authors, margin $\delta_n$ in the signal domain of input tensors equals to 
\[
\delta_n = 
\left\{
\begin{array}{ll}
0 & \text{if }\ n\bmod g\leq 1,\\
(n\bmod g)-1 & \text{otherwise.}
\end{array}
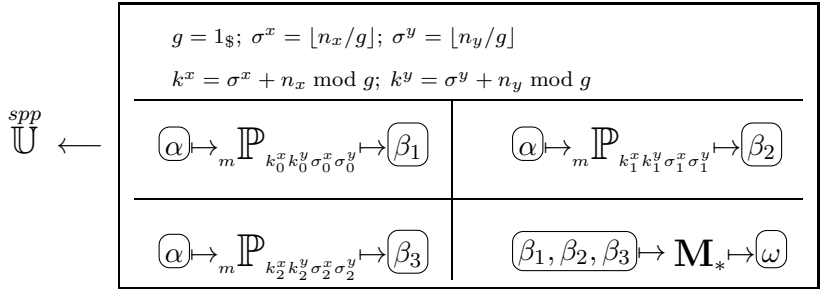
\right.
\]

\xunitdef{spp}{
\begin{tabular}{l|l}
\multicolumn{2}{l}{$
\begin{array}{l}
\xexpression{g = 1_{\$};\ \sigma^x = \lfloor n_x/ g\rfloor;\ \sigma^y = \lfloor n_y/ g\rfloor}\\
\xexpression{k^x = \sigma^x+n_x\bmod g;\ k^y = \sigma^y+n_y\bmod g}
\end{array}
$}\\[2pt]\hline\\[-8pt]
\xfromlabel{\alpha}
\xpool{k^x_0k^y_0\sigma^x_0\sigma^y_0}{}{m}{}{}\ \xtolabel{\beta_1} 
\quad&\quad
\xfromlabel{\alpha}
\xpool{k^x_1k^y_1\sigma^x_1\sigma^y_1}{}{m}{}{}\ \xtolabel{\beta_2} 
\\[10pt]
\hline\\[-5pt]
\xfromlabel{\alpha}
\xpool{k^x_2k^y_2\sigma^x_2\sigma^y_2}{}{m}{}{}\ \xtolabel{\beta_3} 
\quad&\quad
\xmerge{\beta_1,\beta_2,\beta_3}{\ast}\xtolabel{\omega}
\end{tabular}
}\\

\xunitinstance{spp}{1}{[5,3,1]}

\xhybridunitdef{spp}{
\begin{tabular}{l|l}
\multicolumn{2}{l}{$
\begin{array}{l}
\xexpression{g = 1_{\$};\ s^x = \lfloor n_x/ g\rfloor;\ s^y = \lfloor n_y/ g\rfloor}\\
\xexpression{w = s^x+n_x\bmod g;\ h = s^y+n_y\bmod g}
\end{array}
$}\\[2pt]\hline\\[-8pt]
\tabxfromlabel{\alpha}
\tabxpool{k^x_0k^y_0\sigma^x_0\sigma^y_0}{}{m}{}{}\ \tabxtolabelunique{\beta_1} 
\quad&\quad
\tabxfromlabel{\alpha}
\tabxpool{k^x_1k^y_1\sigma^x_1\sigma^y_1}{}{m}{}{}\ \tabxtolabelunique{\beta_2} 
\\[10pt]
\hline\\[-5pt]
\tabxfromlabel{\alpha}
\tabxpool{k^x_2k^y_2\sigma^x_2\sigma^y_2}{}{m}{}{}\ \tabxtolabelunique{\beta_3} 
\quad&\quad
\tabxmerge{\beta_1,\beta_2,\beta_3}{\ast}\tabxtolabel{\omega}
\end{tabular}
}\\

\subsection{STNN/DSTNN via \LaTeX\ commands}

\subsubsection{\LaTeX\ commands}

\begin{enumerate}
\item Units:
\begin{enumerate}
\item custom units
\begin{enumerate}
\item convolution:
$\Bigg\{$
{\small
\begin{tabular}{ll}
regular: & \verb!\xconv{}{}{}{}{}!\\
dual: & \verb!\dxconv{}!\\
hybrid: & \verb!\tabxconv{}{}{}{}{}!
\end{tabular}
}
\item full connection:\\
$\Bigg\{$
{\small
\begin{tabular}{ll}
regular: & \verb!\xdense{}{}{}{}{}!\\
dual: & \verb!\dxdense{}!\\
hybrid: & \verb!\tabxdense{}{}{}{}{}{}!
\end{tabular}
}
\item pooling:
$\Bigg\{$
{\small
\begin{tabular}{ll}
regular: & \verb!\xpool{}{}{}{}{}!\\
dual: & \verb!\dxpool{}!\\
hybrid: & \verb!\tabxpool{}{}{}{}{}!
\end{tabular}
}
\item interpolating:
$\Bigg\{$
{\small
\begin{tabular}{ll}
regular: & \verb!\xinterp{}{}!\\
dual: & \verb!\dxinterp!\\
hybrid: & \verb!\tabxinterp{}{}!
\end{tabular}
}
\item rectified linear unit:
$\Bigg\{$
{\small
\begin{tabular}{ll}
regular: & \verb!\xrelu!\\
dual: & \verb!\dxrelu!\\
hybrid: & \verb!\tabxrelu!
\end{tabular}
}
\item leaky rectified linear unit:
$\Bigg\{$
{\small
\begin{tabular}{ll}
regular: & \verb!\xrelup{}!\\
dual: & \verb!\dxrelup{}!\\
hybrid: & \verb!\tabxrelup{}!
\end{tabular}
}
\item batch normalization:
$\Bigg\{$
{\small
\begin{tabular}{ll}
regular: & \verb!\xbatch!\\
dual: & \verb!\dxbatch!\\
hybrid: & \verb!\tabxbatch!
\end{tabular}
}
\item instance normalization:
$\Bigg\{$
{\small
\begin{tabular}{ll}
regular: & \verb!\xinstant!\\
dual: & \verb!\dxinstant!\\
hybrid: & \verb!\tabxinstant!
\end{tabular}
}
\item sigmoid unit:
$\Bigg\{$
{\small
\begin{tabular}{ll}
regular: & \verb!\xsigmo!\\
dual: & \verb!\dxsigmo!\\
hybrid: & \verb!\tabxsigmo!
\end{tabular}
}
\item hyperbolic tangent unit:
$\Bigg\{$
{\small
\begin{tabular}{ll}
regular: & \verb!\xtanh!\\
dual: & \verb!\dxtanh!\\
hybrid: & \verb!\tabxtanh!
\end{tabular}
}
\end{enumerate}
\item user defined units:
$\Bigg\{$
{\small
\begin{tabular}{ll}
regular: & \verb!\xunit{}{}{}!\\
dual: & \verb!\dxunit{}{}{}!\\
hybrid: & \verb!\tabxunit{}{}{}!
\end{tabular}
}
\end{enumerate}

\item Inputs:
$\Bigg\{$
{\small
\begin{tabular}{ll}
regular: & \verb!\xin{}{}{}!\\
dual: & \verb!\dxin{}!\\
hybrid: & \verb!\tabxin{}{}{}!
\end{tabular}
}

\item Links:
\begin{enumerate}
\item to multi-cast unit output tensor:
\begin{itemize}
\item at the end of component:
$\Bigg\{$
{\small
\begin{tabular}{ll}
regular: & \verb!\xtolabel{}!\\
dual: & \verb!\dxtolabel{}!\\
hybrid: & \verb!\tabxtolabel{}!
\end{tabular}
}
\item in the middle of the component:
$\Bigg\{$
{\small
\begin{tabular}{ll}
regular: & \verb!\xtolabelto{}!\\
dual: & \verb!\dxtolabelto{}!\\
hybrid: & \verb!\tabxtolabelto{}!
\end{tabular}
}
\item with unique target:
$\Bigg\{$
{\small
\begin{tabular}{ll}
regular: & \verb!\xtolabelunique{}!\\
dual: & \verb!\dxtolabelunique{}!\\
hybrid: & \verb!\tabxtolabelunique{}!
\end{tabular}
}
\end{itemize}
\item to uni-cast unit outputs:
\begin{itemize}
\item at th end of component:
$\Bigg\{$
{\small
\begin{tabular}{ll}
regular: & \verb!\xtolabeladd{}!\\
dual: & \verb!\dxtolabeladd{}!\\
hybrid: & \verb!\tabxtolabeladd{}!
\end{tabular}
}
\item in the middle of the component:
$\Bigg\{$
{\small
\begin{tabular}{ll}
regular: & \verb!\xtolabeltoadd{}!\\
dual: & \verb!\dxtolabelatodd{}!\\
hybrid: & \verb!\tabxtolabeltoadd{}!
\end{tabular}
}
\end{itemize}
\item to request unit output\footnote{The request is issued by the units located in the beginning of components.}:
$\Bigg\{$
{\small
\begin{tabular}{ll}
regular: & \verb!\xfromlabel{}!\\
dual: & \verb!\dxfromlabel{}!\\
hybrid: & \verb!\tabxfromlabel{}!
\end{tabular}
}
\item to refer to unit output:
$\Bigg\{$
{\small
\begin{tabular}{ll}
regular: & \verb!\xtoreflabelto{}!\\
dual: & \verb!\dxtoreflabelto{}!\\
hybrid: & \verb!\tabxtoreflabelto{}!
\end{tabular}
}
\end{enumerate}

\item Residual blocks:
$\Bigg\{$
{\small
\begin{tabular}{ll}
regular (short)\footnote{Without and with exclamation mark.}: & \verb!\xresid{}{}, \xxresid{}!\\
regular: &\verb!\xshortcut{}{}!\\
dual: & \verb!\dxshortcut{}{}!\\
hybrid: & \verb!\tabxshortcut{}{}{}!
\end{tabular}
}
\item Merge/Split:
\begin{enumerate}
\item to merge unit inputs:
$\Bigg\{$
{\small
\begin{tabular}{ll}
regular: & \verb!\xmerge{}{}!\\
dual: & \verb!\dxmerge{}{}!\\
hybrid: & \verb!\tabxmerge{}{}!
\end{tabular}
}

\item to split unit output:
$\Bigg\{$
{\small
\begin{tabular}{ll}
regular: & \verb!\xsplit{}{}!\\
dual: & \verb!\dxsplit{}{}!\\
hybrid: & \verb!\tabxsplit{}{}!
\end{tabular}
}
\end{enumerate}

\item User defined units:
\begin{enumerate}
\item to instance unit:
$\bigg\{$
{\small
\begin{tabular}{ll}
regular: & \verb!\xunitinstance}{}{}{}!\\
hybrid: & \verb!\xhybridunitinstance}{}{}{}!
\end{tabular}
}
\item to define unit:
$\bigg\{$
{\small
\begin{tabular}{ll}
regular: & \verb!\xunitdef{}{}!\\
hybrid: & \verb!\xhybridunitdef{}{}!
\end{tabular}
}
\end{enumerate}

\item Symbolic net instance: $\big\{$\verb!xbound{}{}{}!
\item Flow arrows: 
$\Bigg\{$
{\small
\begin{tabular}{ll}
flow start: \verb!\lrtabbeg!\\
flow middle: \verb!\lrtab!\\
flow end: \verb!\lrtabend!
\end{tabular}
}
\end{enumerate}

\subsubsection{\LaTeX script for VGG-16}

\begin{enumerate}
\item VGG-16 as unstructured STNN:
\begin{scriptsize}
\begin{verbatim}
\xin{yx}{3}{rgb}
\xconv{3}{64}{}{}{r} \xconv{3}{64}{}{}{r} \xpool{2}{}{m}{}{} 
\xtoreflabelto{vgg_1} \xconv{3}{128}{}{}{r} \xconv{3}{128}{}{}{r}
\xpool{2}{}{m}{}{} \xtoreflabelto{vgg_2} \xconv{3}{256}{}{}{r}
\xconv{3}{256}{}{}{r} \xconv{3}{256}{}{}{r} \xpool{2}{}{m}{}{}
\xtoreflabelto{vgg_3} \xconv{3}{512}{}{}{r} \xconv{3}{512}{}{}{r}
\xconv{3}{512}{}{}{r} \xpool{2}{}{m}{}{} \xtoreflabelto{vgg_4}
\xconv{3}{512}{}{}{r} \xconv{3}{512}{}{}{r} \xconv{3}{512}{}{}{r}
\xpool{2}{}{m}{}{} \xtoreflabelto{vgg_5} \xdense{}{4096}{}{}{}
\xdense{}{4096}{}{}{} \xdense{}{1000}{}{}{s}
\xtolabel{score}
\xbound{vgg}{1}{rgb := 112_{xy}3_c;\ 
                optima := [loss, MomentumSGD, SoftMax \eqref{eq:soft-max-loss}]}
\xbound{vgg}{2}{rgb := 224_{xy}3_c;
                optima := [loss, MomentumSGD, SoftMax \eqref{eq:soft-max-loss}]}
\end{verbatim}
\end{scriptsize}
\item VGG-16 as unstructured STNN/DSTNN:
\begin{scriptsize}
\begin{verbatim}
\tabxin{yx}{3}{rgb}
\tabxconv{3}{64}{}{}{r} \tabxconv{3}{64}{}{}{r} \tabxpool{2}{}{m}{}{}
\tabxtoreflabelto{vgg_1} \tabxconv{3}{128}{}{}{r} \tabxconv{3}{128}{}{}{r}
\tabxpool{2}{}{m}{}{} \tabxtoreflabelto{vgg_2} \tabxconv{3}{256}{}{}{r}
\tabxconv{3}{256}{}{}{r} \tabxconv{3}{256}{}{}{r} \tabxpool{2}{}{m}{}{}
\tabxtoreflabelto{vgg_3} \tabxconv{3}{512}{}{}{r} \tabxconv{3}{512}{}{}{r}
\tabxconv{3}{512}{}{}{r} \tabxpool{2}{}{m}{}{} \tabxtoreflabelto{vgg_4}
\tabxconv{3}{512}{}{}{r} \tabxconv{3}{512}{}{}{r} \tabxconv{3}{512}{}{}{r}
\tabxpool{2}{}{m}{}{} \tabxtoreflabelto{vgg_5} \tabxdense{}{4096}{}{}{}
\tabxdense{}{4096}{}{}{} \tabxdense{}{1000}{}{}{s}
\tabxtolabel{score}
\xbound{vgg}{1}{rgb := 112_{xy}3_c;\ 
               optima := [loss, MomentumSGD, SoftMax \eqref{eq:soft-max-loss}]}
\end{verbatim}
\end{scriptsize}
\end{enumerate}

\section{CNN formulas for digital media applications}\label{sec:creams}

\subsection{STNN for image compression\label{sec:compress}} 

Agustsson et al. from ETH in Zurich in their recent paper \cite{Agustsson18a}, {\it Generative Adversarial Networks for Extreme Learned Image Compression}, describe the novel CNN solution for lossy compression of images which outperforms BPG ({\it Better Portable Graphics}) method not only in bitrate (more than 50\% savings) but in the the visual quality, as well. The proposed method elaborates the classical auto-encoder scheme which was so far inferior (even in CNN architectures) than the contemporary compression standards. 

There are many factors considered in this novel neural algorithm: 
\begin{itemize}
\item learning by mixing the traditional rate-distortion (RD) quality measure with GAN (Generative Adversary Network) which looks for discriminant function to maximize the probability of distinguishing between the input and the reconstructed image and in the same time to minimize the lossy function wrt encoder and decoder combination,
  \item learning using mean squared error (MSE) not only for the reconstructed image but for deep image features delivered by VGG-16 network from tensors labeled by $vgg_i$, $i=1,\dots 5,$ and by GAN module, as well,
  \item deep composition of convolutions with ReLU preceded by instant normalization (not batch normalization) in order to get the nonlinear image transformation into a latent representation having wide attribute axis,
  \item high quality approximation of the discrete non-differentiable quantization process by the continuous, differentiable one.
\end{itemize}

The authors in the same paper describe  yet another compression scheme which uses the regions of interests (ROI) to allocate bit budget for them while non-interesting areas are generated with highly realistic content. In this tutorial the STNN notation for ETH CODEC  is presented only global (non ROI) case.

\paragraph{Nash equation for GAN\\}

Let the encoder implements the function $E$, the decoder -- function $G$ (the GAN's generator), the quantizer -- the function $q$, the entropy coder -- the function $H$,  and the GAN's discriminator -- the function  $D$. Let also $d$ denotes the distortion measure. Then for the input image $x$, we get the quantized image $\widehat{w},$ the reconstructed image $\widehat{x}$ and the number of bits $n$ from the arithmetic coder using the probabilities estimated:
\[
\widehat{x} = G(\widehat{w},v),\ \widehat{w} = q(E(x)),\ n = H(\widehat{w})\ .
\]
Then we have the following Nash equation for the optimized functions $E,G,D$ at some weighting coefficients $\lambda,\beta$ and the Least Squares divergence (contrary to the KL divergence) for the probability distributions, and MSE for the image distortion measure $d$ taken between $x$ and the reconstructed image increased by the MSE between $vgg_i$ features, $i=1,\dots,5$:
\begin{equation}\label{eq:eth-codec-gan}
\begin{array}{l}
z\eqd[\widehat{w},v],\ \cl{L}_{GAN}(G,D)\eqd \bb{E}_x[(D(x)-1)^2]+\bb{E}_z[D^2(G(z)] \lra\\[1pt]
(E,G,D) = \ds \arg\min_{E,G}\arg\max_D\left[\cl{L}_{GAN}(G,D)+
\lambda\bb{E}_{x,v}[d(x,G(z)]+\beta\bb{E}_xH(\widehat{w})\right]\\[1pt]
\end{array}
\end{equation}
Authors noticed that the entropy term $H(\widehat{w})$ is always bounded by the number of elements in $\widehat{w}$ times the entropy of the uniform source with $L$ symbols, i.e. times $\log_2L$, where $L$ equal to number of quantization levels and therefore the bitrate can be controlled by the hyper-parameters of the proposed architecture of $E$. According to the authors it justifies dropping of the entropy term from the goal function by setting $\beta=0.$ For the quality term the coefficient $\lambda=10.$

\subsubsection{STNN diagrams for ETH GAN image codec while testing}

\paragraph{Encoder -- modeling for training and testing\\[5pt]}
\xunitdef{model}{
\xconv{7}{60}{p_r}{}{ir}
\xconv{2_{\sigma}3}{120}{p_r}{}{ir}
\xconv{2_{\sigma}3}{240}{p_r}{}{ir}
\xconv{2_{\sigma}3}{480}{p_r}{}{ir}
\xconv{2_{\sigma}3}{960}{p_r}{}{ir}
\xconv{3}{C}{p_r}{}{ir}
}\\[10pt]
where: 
\begin{itemize}
  \item $C$ -- the tuning parameter for RD (Rate-Distortion) trade-off,
  \item $p_r$ -- padding by reflection of tensor's boundary values,
  \item $ir$ -- convolution is followed by instant normalization and ReLU activation, 
  \item $N_b=1$ -- (mini) batch size is equal to $1$.
\end{itemize}
\paragraph{Decoder -- modeling  for testing and training\\[5pt]}
\xunitdef{decmod}{
\xconv{3}{960}{p_r}{}{ir}
\xresid{\xconv{3}{960}{p_r}{}{}\xconv{3}{960}{p_r}{}{}}{9}
\xconv{2_{\sigma}3}{480}{tp_r}{}{ir}
\xconv{2_{\sigma}3}{240}{tp_r}{}{ir}
\xconv{2_{\sigma}3}{120}{tp_r}{}{ir}
\xconv{2_{\sigma}3}{60}{tp_r}{}{ir}
\xconv{7}{3}{p_r}{}{ir}
}\\[10pt]
where the transposed convolutions fit to the encoders convolutions according to the auto-encoder scheme while the preceding sequence of nine residual blocks is experimentally tuned to enhance learning stage and the reconstructed images.
\paragraph{Quantizer while testing\\[5pt]}
\xunitdef{tquant}{
$\ds 
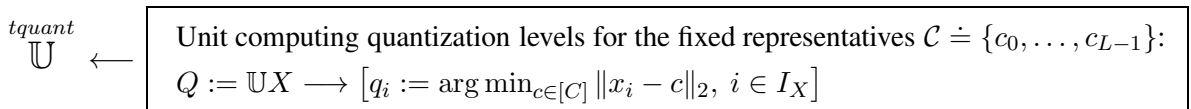
\begin{array}{l}
\text{Unit computing quantization levels for the fixed representatives $\cl{C}\eqd\{c_0,\dots,c_{L-1}\}$:}\\[1pt]
Q := \bb{U}X \lra \left[
q_i := \arg\min_{c\in[C]}\|x_i-c\|_2,\ i\in I_X\right]
\end{array}
$}
where:
\begin{itemize}
\item the notation $\bb{U}X$ represents transforming aspect of the quantization layer,
  \item the quantization process is scalar one,
  \item the quantization intervals are implicitly defined by their representatives,
  \item the representatives are fixed -- they belong to the set of hyper parameters of CNN solution,
  \item the authors considered uniform scalar quantizers, 
  e.g.:\\ for $L=5$ and $\cl{C}\eqd [-2,-1,0,+1,+2]$,
  \item for the above representatives, the quantization $q(x)$ is simply the integer rounding operation of real values, clamping them to the interval $[-2,+2]$: 
\[
\text{clamp}_a^b(i) \eqd \max(a,\min(i,b)) \lra q(x)\eqd \text{clamp}_{-2}^{+2}\left(\left\lfloor x+\frac{1}{2}\right\rfloor\right)
\]
\end{itemize}

\paragraph{Encoder while testing\\[5pt]}
\frame
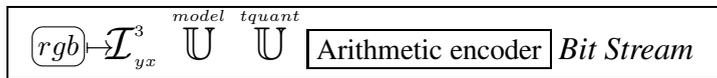
{
\begin{tabular}{l}
\xin{yx}{3}{rgb}
\xunit{model}{}{}
\xunit{tquant}{}{}
\framebox{Arithmetic encoder}
{\large\em Bit Stream}
\end{tabular}
}\\[5pt]
where:
\begin{itemize}
  \item arithmetic encoder is implemented by an external code and it converts the 3D volume of the image features into the bit stream using conditional probabilities of features,
  \item arithmetic encoder either elaborates adaptively the conditional (casual) probability distributions\footnote{Then the quantization representatives must be changed to their indexes in the table $[c_0,\dots,c_{L-1}]$ since the standard arithmetic codecs work with symbols. In the author's example $[-2,-1,0,+1,+2]$ this means adding $2$ and converting to integer value.}or
  \item arithmetic encoder exploits the learned CNN module which computes the conditional probability distribution of $L$ symbols for each feature in 3D volume separately (cf. the modules of encoder and the probability estimator, presented below in the training scheme).
\end{itemize}


\paragraph{Decoder while testing\\[5pt]}
\frame
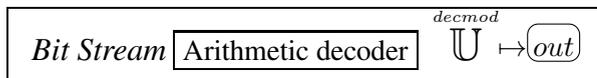
{
\begin{tabular}{l}
{\large\em Bit Stream}
\framebox{Arithmetic decoder}
\xunit{decmod}{}{}
\xtolabel{out}
\end{tabular}
}\\[5pt]
where:
\begin{itemize}
  \item arithmetic decoder is implemented by an external code and it converts the encoded bit stream into 3D volume of the image features using conditional probabilities,
  \item arithmetic decoder elaborates\footnote{In case of CNN probability estimator, the convolutions compute all probabilities before using them to encoding/decoding.} adaptively the conditional probability distributions matching the steps of the arithmetic encoder -- it is the reason of using the casual probability context by the encoder.
\end{itemize}

\subsubsection{STNN diagrams for ETH GAN image codec while training}

\paragraph{Quantizer while learning\\[5pt]}

As described in the previous work by Mentzer et al. in \cite{Mentzer18a}: {\it Conditional Probability Models for Deep Image Compression.}
  {\it Conditional Probability Models for Deep Image Compression,}
the unit is computing the soft quantization  for the fixed representatives $\cl{C}\eqd\{c_0,\dots,c_{L-1}\}$\\[5pt]
\xunitdef{lquant}{
$
\begin{array}{l}
\tilde{Q} := \bb{U}X \lra 
\tilde{q}_i := 
\frac
{\ds\sum_{k=1}^Le^{-\sigma(x_i-c_k)^2}\cdot c_k}
{\ds\sum_{l=1}^Le^{-\sigma(x_i-c_l)^2}},\quad i\in I_X
\end{array}
$
}\\[5pt]

The soft quantization approximates in differentiable way the crisp quantization function which is discrete in nature.

\paragraph{Casual probability estimator\\[5pt]}

The authors of \cite{Mentzer18a} present also the CNN solution to estimate the conditional (casual) probability of quantizing symbols which is necessary for arithmetic coding.

\xunitdef{condprob}{
\xconv{3^y3^x3^d}{24}{cp_r}{}{r}
\xresid{
\xconv{3^y3^x3^d}{24}{cp_r}{}{r}
\xconv{3^y3^x3^d}{24}{cp_r}{}{}
}{}
\xconv{3^{xy}3^d}{96}{cp_r}{}{r}
}

The final feature extractor of the probability estimator computes for each of $24$ features the probability distribution wrt to four symbols\footnote{Note that in the final architecture proposed in \cite{Agustsson18a} the probability estimator was removed from the optimization stage.}. Therefore there are $96=4\cdot 24$ features in total.

\paragraph{GAN architecture: generator\\}

The architecture for the encoder $E$ and the decoder $G$ is based on 
Wang et al. \cite{Wang17a}: {\it High resolution image synthesis and semantic manipulation with conditional GANs.} In turn it is improved architecture of the baseline architecture proposed by 
Isola et. al. in \cite{Isola16a} as {\it pix2pix} network: {\it Image-to-image translation with conditional adversarial networks.} The contribution of Wang's et al. architecture is in the generator $G$ which consists of global part $G_1$ (based on Johnson et. al \cite{Johnson16a}) and the local enhancer $G_2.$ It seems that in this GAN's relay, the torch of ideas has been carried for a while by the team if Radford et.al \cite{Radford15a} who noticed that GAN can give for images an unsupervised representation what appears so important for compression and embedding tasks in digital research area. 

Interestingly, the key paper in the above GAN's relay \cite{Agustsson18a} describes its architecture using symbolic names for units and some combinations of them. The notation apparently was introduced in the appendix of Zhu's et al. paper \cite{Zhu17a}: {\it Unpaired image-to-image translation using cycle-consistent adversarial networks.} For instance the string {\tt 7s1-k} denotes our \xconv{7}{k}{}{}{ir} convolution unit combined with the instance normalization and ReLU.

\begin{itemize}
\item Unit encoder-decoder $encDec$ is defined separately as it is also included into the discriminator of GAN's architecture. Since its architecture is static, i.e. there is no special hyper parameters to be set, this kind of instancing is not applied.

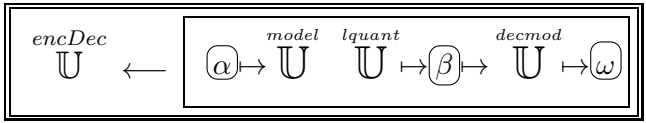
\doublebox{
\xunitdef{encDec}{
\xfromlabel{\alpha}
\xunit{model}{}{}\ \ 
\xunit{lquant}{}{}
\xtolabelto{\beta}
\xunit{decmod}{}{}
\xtolabel{\omega}
}\\
}


\item The codec network is optionally supplemented by the casual probability modeling to control the entropy of the output bit stream. However, in the recent solution authors skipped in the training the probability estimator. It is left here to show alternatives in the design of DNN image codec which plays the role of the generator in GAN's architecture of ETH solution. Therefore, the name of this network is {\it ethEncDec.} 

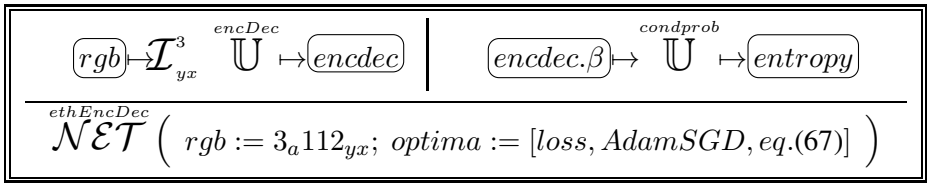
\doublebox{
\begin{tabular}{l}
\begin{tabular}{l|l}
\xin{yx}{3}{rgb}
\xunit{encDec}{}{}
\xtolabel{encdec} \quad & \quad
\xfromlabel{encdec.\beta}
\xunit{condprob}{}{}
\xtolabel{entropy}
\end{tabular}\\[10pt]
\hline
\xbound{ethEncDec}{}{
\begin{array}{l}
rgb := 3_a112_{yx};\ optima := [loss,AdamSGD,eq.\eqref{eq:eth-enc-dec}]
\end{array}}
\end{tabular}
}

Since, the loss function {\it ethG} depends on features to be computed  by the discriminator {\it nviD} its description is postponed till its definition.
\end{itemize}

\paragraph{GAN architecture: NVI discriminator\\[5pt]}

Since except the one, all authors of the described discriminator (Wang et al. \cite{Wang17a}) were affiliated with NVIDIA corporation, the name of this core part of GAN's solution is {\it NVI discriminator.}
\begin{itemize}
\item Unit $gan$ is typical CNN module which applies $4$ times the convolution operation with $4\times 4$ mask, reducing each time the spatial resolution by $2$ while increasing attribute resolution  $2$ times. The final full connection layer with the sigmoid nonlinearity computes the probability that the input is generated by the codec, i.e. that it is not the original image.

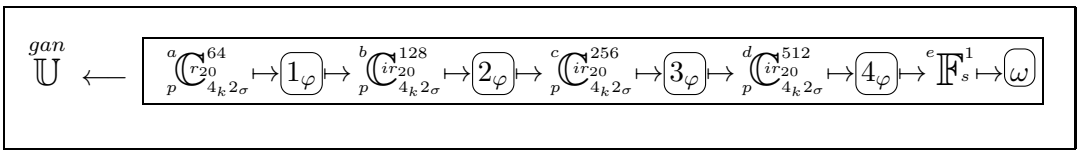
\frame{
\begin{tabular}{l}
\\[-5pt]
\xunitdef{gan}{
\xconv{4_k2_{\sigma}}{64}{p}{a}{r_{20}}
\xtolabelto{1_{\varphi}}
\xconv{4_k2_{\sigma}}{128}{p}{b}{ir_{20}}
\xtolabelto{2_{\varphi}}
\xconv{4_k2_{\sigma}}{256}{p}{c}{ir_{20}}
\xtolabelto{3_{\varphi}}
\xconv{4_k2_{\sigma}}{512}{p}{d}{ir_{20}}
\xtolabelto{4_{\varphi}}
\xdense{}{1}{}{e}{s}
\xtolabel{\omega}
}\\$ $
\end{tabular}
}

The discriminator exploits the three {\it gan} units, each for different resolution of input image:

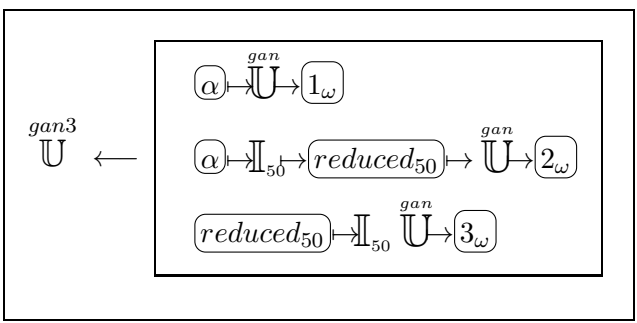
\frame{
\begin{tabular}{l}
\\[-5pt]
\xunitdef{gan3}{
\begin{tabular}{l}
\xfromlabel{\alpha}\xunit{gan}{}{}\xtolabel{1_{\omega}}\\[5pt]
\xfromlabel{\alpha}\xinterp{50}{}\xtolabelto{reduced_{50}}
\xunit{gan}{}{}\xtolabel{2_{\omega}}\\[5pt]
\xfromlabel{reduced_{50}}\xinterp{50}{}\xunit{gan}{}{}\xtolabel{3_{\omega}}
\end{tabular}
}\\$ $
\end{tabular}
}

The above unit offers $4$ features (labeled by $1_{\varphi},\dots,4_{\varphi}$) computed for three different resolutions of  the input image. Those $12$ tensors will be used by the loss function of {\it ethG}. 

\item The network $ethD$ is to discriminate the original and the reconstructed (decoded) images. To this goal the discrimination function $D^2(x)+(1-D(G(z))^2$ is computed for the original resolution $100\%$, reduced two times giving $50\%$ of the original resolution, and finally $25\%$ resolution is processed:

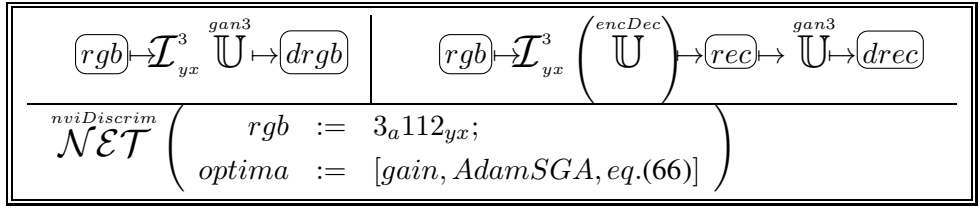
\doublebox{
\begin{tabular}{l}
\begin{tabular}{l|l}
\xin{yx}{3}{rgb}\xunit{gan3}{}{}
\xtolabel{drgb} \quad&\quad
\xin{yx}{3}{rgb}\xunitcall{encDec}{}{}\xtolabelto{rec}
\xunit{gan3}{}{}\xtolabel{drec}\\[5pt]
\end{tabular}\\[3pt]
\hline
\xbound{nviDiscrim}{}{
\begin{array}{rcl}
rgb &:=& 3_a112_{yx};\\
optima &:=& [gain,AdamSGA,eq.\eqref{eq:nvi-discrim-gain}]
\end{array}}
\end{tabular}
}

\end{itemize}


\paragraph{Gain function for {\it nvi} discriminator\\[5pt]}

ETH authors replaced in the discriminator the usual $\log$  function by the square function. There are two groups for discriminant values produced by {\it nviDiscrim} net -- one coming from two image pyramids: $d^{rgb}\inv{3}$ for the original rgb image and $d^{rec}\inv{3}$ for the recovered one. We combine them by components to get the final discriminant value:
\begin{equation}\label{eq:nvi-discrim-gain}
gain(d^{rgb},d^{rec}) \eqd 
\sum_{i=0}^2\left[\left(d_i^{rgb}\right)^2+\left(1-d_i^{rec}\right)^2\right] = \tp{\bm{1}}_3
\overbrace{
\left[
\begin{array}{l}
(d_0^{rgb})^2+\left(1-d_0^{rec}\right)^2\\[3pt]
(d_1^{rgb})^2+\left(1-d_1^{rec}\right)^2\\[3pt]
(d_2^{rgb})^2+\left(1-d_2^{rec}\right)^2
\end{array}
\right]}^{x+y}
\end{equation}

Hence, using the algebraic units and linear algebra expressions we get the symbolic notation for the gain function of {\it nviDiscrim} tensor neural network:\\

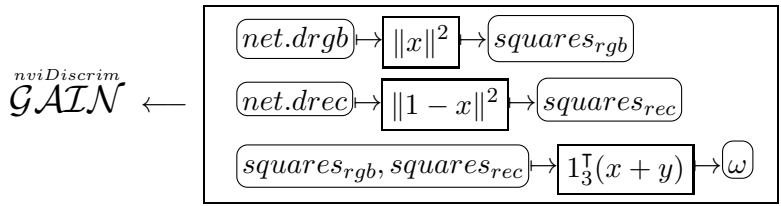
\xgaindef{nviDiscrim}{}{
\begin{tabular}{l}
\xfromlabel{net.drgb}
\xmath{\|x\|^2}
\xtolabel{squares_{rgb}}\\[5pt]
\xfromlabel{net.drec}
\xmath{\|1-x\|^2}
\xtolabel{squares_{rec}}\\[5pt]
\xfromlabel{squares_{rgb},squares_{rec}}
\xmath{\tp{1}_3(x+y)}
\xtolabel{\omega}
\end{tabular}
}

\paragraph{VGG feature extractor\\[5pt]}

The {\it deep features} of VGG-16 net are obtained from five tensors of this standard neural architecture. Namely, the output of instances for the user defined units {\it c2, c3} are  considered as important attributes of natural scene images. 

\begin{itemize}
\item Unit {\it c2}:
\xunitdef{c2}{
\begin{tabular}{l}
\xexpression{f = 64\cdot 1_{\$}}\\
\hline
\xconv{3}{f}{}{}{r}
\xconv{3}{f}{}{}{r}
\xpool{2}{}{m}{}{}
\end{tabular}
}\ \ 
\xunitinstance{c2}{1}{1}\ \ 
\xunitinstance{c2}{2}{2}

\item Unit {\it c3}:
\xunitdef{c3}{
\begin{tabular}{l}
\xexpression{f = 256\cdot 1_{\$}}\\
\hline
\xconv{3}{f}{}{}{r}
\xconv{3}{f}{}{}{r}
\xconv{3}{f}{}{}{r}
\xpool{2}{}{m}{}{}
\end{tabular}
}\ \ 
\xunitinstance{c3}{1}{1}\ \ 
\xunitinstance{c3}{2}{2}

\item Module {\it vgg}:
\xunitdef{vgg}{
\begin{tabular}{l}
\xfromlabel{1_{\alpha}}
\xunit{c2}{1}{}
\xtolabelto{1_{\varphi}}
\xunit{c2}{2}{}
\xtolabelto{2_{\varphi}}
\xunit{c3}{1}{}
\xtolabelto{3_{\varphi}}
\xunit{c3}{2}{}
\xtolabelto{4_{\varphi}}
\xunit{c3}{2}{}
\xtolabel{5_{\varphi}}\\[5pt]
\end{tabular}
}
\end{itemize}

\paragraph{Loss function for {\it ethEncDec}\\[5pt]}

Let 
\begin{itemize}
  \item $k$ be the index of VGG-16 feature, i.e. $vgg_k^{rgb}$ is the $j$-th feature tensor accessed via  $vrgb.\varphi$ group of tensors, and $vgg_k^{rec}$ is the $j$-th feature tensor accessed via $vrec.\varphi,$ 
  \item $i$ be the level index of image pyramid $i=0,1,2$ in the discriminator net, i.e. $d^{rec}_i$ , are $nvi^{rgb}_{ij}$ ($nvi^{rec}_{ij}$) is the $i$-th discrimination value feature tensor accessed via $net.drgb.\omega$ group and the $j$-th feature tensor accessed via $net.drgb.\omega.\varphi$ group. The loss function for {\it ethEncDec} net joins feature distortion terms $\|f^{rgb}-f^{rec}\|$ weighted by $\lambda=10$ and the discrimination term $(1-d)^2$.
\end{itemize}

\begin{equation}\label{eq:eth-enc-dec}
\begin{array}{rcl}
loss(rgb,rec,vgg,nvi,d) 
&\eqd&
\ds\sum_{i=0}^{2}\left(1-d^{rec}_i\right)^2+\|rgb-rec\|^2\\[15pt]
&&
\ds+\lambda\cdot\left(\ds\sum_{k=0}^{4}\left\|vgg_k^{rgb}-vgg_k^{rec}\right\|^2+\ds\sum_{i=0}^{2}\sum_{j=0}^{3}\left\|nvi_{ij}^{rgb}-nvi_{ij}^{rec}\right\|^2\right)
\end{array}
\end{equation}

Access paths -- explanation:
\begin{itemize}
  \item $d^{rec}_{ij}\ \mapsto\ net.drgb.(i+1)_{\omega},$
  \item $rgb\ \mapsto\ net.rgb$,\quad $rec\ \mapsto\ net.rgb$,
  \item $vgg^{rgb}_k\ \mapsto\ vrgb.(k+1)_{\varphi}$,\quad
  $vgg^{rec}_k\ \mapsto\ vrec.(k+1)_{\varphi}$,
  \item $nvi^{rgb}_{ij}\ \mapsto\ net.drgb.(i+1)_{\omega}.(j+1)_{\varphi},$\quad
  $nvi^{rec}_{ij}\ \mapsto\ net.drec.(i+1)_{\omega}.(j+1)_{\varphi}.$
\end{itemize}


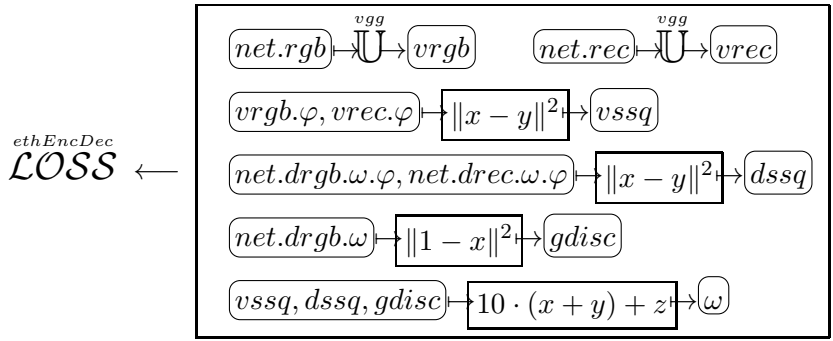
\xlossdef{ethEncDec}{}{
\begin{tabular}{l}
\xfromlabel{net.rgb}\xunit{vgg}{}{}\xtolabel{vrgb}\quad\quad
\xfromlabel{net.rec}\xunit{vgg}{}{}\xtolabel{vrec}\\[5pt]
\xfromlabel{vrgb.\varphi,vrec.\varphi}\xmath{\|x-y\|^2}\xtolabel{vssq}\\[5pt]
\xfromlabel{net.drgb.\omega.\varphi,net.drec.\omega.\varphi}\xmath{\|x-y\|^2}\xtolabel{dssq}\\[5pt]
\xfromlabel{net.drgb.\omega}\xmath{\|1-x\|^2}\xtolabel{gdisc}\\[5pt]
\xfromlabel{vssq,dssq,gdisc}\xmath{10\cdot(x+y)+z}\xtolabel{\omega}
\end{tabular}
}

\paragraph{Optimization details\\}

As the SGD optimizer the {\it AdaM} method is applied  with the learning rate of $0.0002$ and the batch size equal to one (remember that the instance normalization is used instead of the batch normalization).

As we know, the entropy term was dropped from the optimization, and the quality together with GAN's goal function were optimized. However, the hyper-parameters wrt bitrate lead to impressive upper bound. Namely, $L=5$ -- the number of quantization levels with representatives $\{-2,1,0,1,2\}$,\ $C=2$ -- the number of features for the encoder output, imply the bitrate:
\[
\frac{H(\widehat{w})}{n_xn_y}=
\frac{\log_2(5)\cdot\frac{n_x}{16}\frac{n_y}{16}\cdot 2}{n_xn_y}\approx 0.018[bpp]
\]
The actual bitrate when the probability estimator of \cite{Mentzer18a} and the arithmetic codec is used is even less.

\subsection{STNN for object recognition in images}

We consider here the speaker recognition neural architecture VGGVox/res50 developed recently by VGG group \cite{ChungJ18a}: {\it VoxCeleb2: Deep Speaker Recognition}. 

Since the speech is represented by spectrograms (magnitude part) we can still claim that speakers are recognized in images. 
Spectrograms were computed using the Hamming window of width
$25ms$ and step $10ms$. For 3 seconds of speech it results in images of resolution $512_y300_x$. Moreover, the mean and variance normalization is performed on every frequency bin of the spectrogram.

The feature extraction is totally based on ResNet-50 presented by He et. al in \cite{HeK15a}: {\it Deep Residual Learning for Image Recognition}. However, the classification part of the  architecture and two loss functions defined by the authors make this solution unique. The authors achieved EER (Equal Error Rate) below $4\%$ for the large and demanding {\it VoxCeleb2} dataset.

\begin{description}
\item [VGGVox/res50 in STNN:] Unit definitions and instancing:\\[5pt]
\framebox
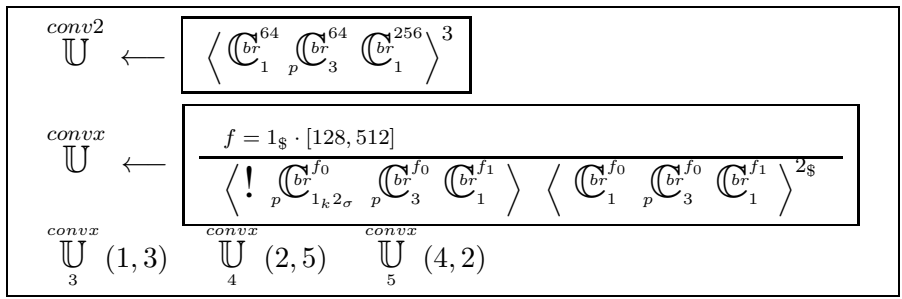
{
\begin{tabular}{l}
\xunitdef{conv2}{
\xresid{\xconv{1}{64}{}{}{br}\xconv{3}{64}{p}{}{br}
\xconv{1}{256}{}{}{br}}{3}
}
\\[10pt]
\xunitdef{convx}{
\begin{tabular}{l}
\xexpression{f = 1_{\$} \cdot [128,512]}\\
\hline
\xxresid{
\xconv{1_k2_{\sigma}}{f_0}{p}{}{br}
\xconv{3}{f_0}{p}{}{br}
\xconv{1}{f_1}{}{}{br}
}
$
\text{
\xresid{
\xconv{1}{f_0}{}{}{br}
\xconv{3}{f_0}{p}{}{br}
\xconv{1}{f_1}{}{}{br}
}{2_{\$}}}
$
\end{tabular}
}\\[10pt]
\xunitinstance{convx}{3}{1,3}\ \ \ \ 
\xunitinstance{convx}{4}{2,5}\ \ \ \
\xunitinstance{convx}{5}{4,2}
\end{tabular}
}


\item [VGGVox/res50 in STNN:] The main architecture for $5994$ speakers:\\[5pt]
\doublebox
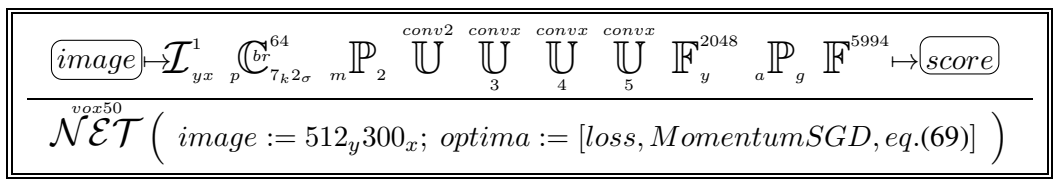
{
\begin{tabular}{l}
\xin{yx}{1}{image}
\xconv{7_k2_{\sigma}}{64}{p}{}{br}
\xpool{2}{}{m}{}{}
\xunit{conv2}{}{}
\xunit{convx}{3}{}
\xunit{convx}{4}{}
\xunit{convx}{5}{}
\xdense{y}{2048}{}{}{}
\xpool{g}{}{a}{}{}
\xdense{}{5994}{}{}{}
\xtolabel{score}\\[5pt]
\hline
\xbound{vox50}{}{
\begin{array}{l}
image := 512_y300_x;\ optima := [loss, MomentumSGD, eq. \eqref{eq:soft-max-loss-stable}]
\end{array}
}
\end{tabular}
}
\end{description}

Like the most of contemporary classifiers in digital media applications the loss function is based on {\it SoftMax} unit as defined in \ref{sec:soft-max}. Just to make the description complete the symbolic loss diagram is implemented using the algebraic notation:
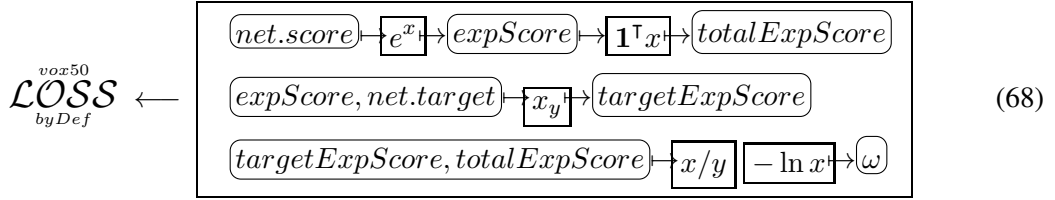
\begin{equation}\label{eq:soft-max-loss}
\text{
\xlossdef{vox50}{byDef}{
\begin{tabular}{l}
\xfromlabel{net.score}\xmath{e^x}\xtolabelto{expScore}\xmath{\tp{\bm{1}}x}\xtolabel{totalExpScore}\\[5pt]
\xfromlabel{expScore,net.target}\xmath{x_y}\xtolabel{targetExpScore}\\[5pt]
\xfromlabel{targetExpScore,totalExpScore}\xmath{x/y}\xmath{-\ln x}\xtolabel{\omega}
\end{tabular}
}}
\end{equation}
In the above diagram the label {\it net.target} denotes the ground truth class index (equals to $t$ in the notation of section \ref{sec:soft-max}) of the current input which is always provided while the classification model is trained. Referring to the gradient flow through the above diagram we observe the standard algebraic function clearly differentiable with the exception of $z\eqd x_y$ with the index $y$ where the flow stops. However, the flow through $x$ is possible only through the component $i=y$. Namely:
\[
(\od{\cl{E}}{x})[i] = \sp{\cl{E}}{x[i]} = \sp{\cl{E}}{z}\sp{z}{x[i]} = 
\sp{\cl{E}}{z}\sp{x[y]}{x[i]}=\od{\cl{E}}{z}\cdot\bm{1}_{i=y},\quad i\in I_x\ .
\]

For very big scores the floating point overflow may happen. We can avoid problem of "big numbers", if the scores are conditioned by subtracting from them the value equal to their maximum:
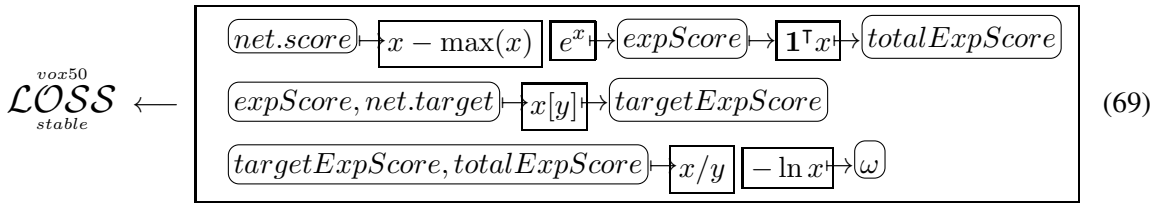
\begin{equation}\label{eq:soft-max-loss-stable}
\text{
\xlossdef{vox50}{stable}{
\begin{tabular}{l}
\xfromlabel{net.score}\xmath{x-\max(x)}\xmath{e^x}\xtolabelto{expScore}\xmath{\tp{\bm{1}}x}\xtolabel{totalExpScore}\\[5pt]
\xfromlabel{expScore,net.target}\xmath{x[y]}\xtolabel{targetExpScore}\\[5pt]
\xfromlabel{targetExpScore,totalExpScore}\xmath{x/y}\xmath{-\ln x}\xtolabel{\omega}
\end{tabular}
}}
\end{equation}

The conditioning of scores introduced the $x-\max(x)$ operation which changes the gradient flow. It seems that the change could work for better convergence as:
\begin{itemize}
  \item without the change the gradient at this point $p-q$ is negative for each component (cf. section \ref{sec:soft-max}),
  \item after the change the gradient  $p-q$ is not changed at any component except the one with\\ $i=\arg\max(x):$
\end{itemize}
\[
\begin{array}{rcl}
y=x-\max(x),\ i\in I_x &\lra&\\
(\od{\cl{E}}{x})[i] &=& \ds\sp{\cl{E}}{x[i]} = \sp{\cl{E}}{y[i]}\sp{y[i]}{x[i]} = \sp{\cl{E}}{y[i]}\sp{[x[i](1-\bm{1}_{i=\arg\max(x)})]}{x[i]}=\\[10pt]
&&(1-\bm{1}_{i=\arg\max{x}})\cdot(\od{\cl{E}}{y})[i]\ .
\end{array}
\]
Hence, if the target class $t$ is confirmed by the maximum probability $p[i]$, $i=t$ then the gradient component equals to zero while others are not changed, potentially increasing the loss function wrt them. Then as expected the optimal Bayesian classifier is followed by the SoftMax layer.

\paragraph{Optimization details\\}

The authors of VGGvox/res50 net report in  \cite{ChungJ18a} the following parameters for optimization:
\begin{itemize}
  \item batch size: $N_b=64$,
  \item optimization method: {\it MomentumSGD},
  \item momentum coefficient: $\beta=0.09,$ with decay $5\cdot 10^{-4},$
  \item learning rate: $\alpha=10^{-2}$ with logarithmic decay to $10^{-8}
  $,
  \item the termination condition is: the number of epochs is equal to $30$ or the validation error stops to decrease.
\end{itemize}

\subsection{STNN for data embedding in images\label{sec:embed}}


Digital media embedding in other digital media objects like images has recently made the significant progress, Dong et al. present in \cite{DongS18a}, the complete (end-to-end) image hiding deep neural algorithm with excellent performance: {\it Invisible Steganography via Generative Adversarial Network}.

The authors proposed a steganographic system for hiding the gray scale image {\it Ysecret} into RGBhost image of the same resolution with unique properties:
\begin{enumerate}
  \item the system consists of three modules:
\begin{enumerate}   
  \item the encoder to embed the secret image into the host image,
  \item the decoder decoder to extract the secret image from the host image,
  \item the stego-analyzer to be used as the discriminator in GAN model,
\end{enumerate}
  \item the hiding capacity is one 8-bit gray scale pixel per 24-bit RGB pixel,
  \item the fidelity of the host noised by hiding secret image measured as structure similarity index SSIM is about $0.985$,
  \item the fidelity of the extracted secret image is on average $0.97,$
  \item the visual subjective evaluation of color constancy is high as actually nothing is hidden in chroma -- the embedding refers only the luminance component,
  \item the visual inspection of the residual image $|host-stegoHost|$ confirms its invisibility,
  \item the steganographic accuracy measured by the probability of correct discrimination\footnote{The detector of hiding events is built as the separate CNN-based binary classifier.} of the host image from its stego-host counterpart is about $73\%$.
\end{enumerate}

\paragraph{Encoder while testing\\}

\begin{enumerate}
\item Definition of {\it inception} unit:

\noindent
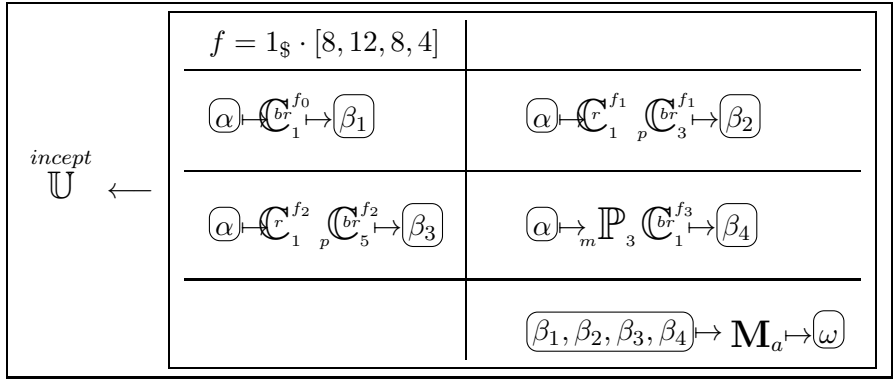
\framebox{
\xunitdef{incept}{
\begin{tabular}{l|l}
$f = 1_{\$}\cdot[8,12,8,4]$ &\\[2pt]
\hline\\[-8pt]
\xfromlabel{\alpha}\xconv{1}{f_0}{}{}{br}\xtolabel{\beta_1} \quad&\quad
\xfromlabel{\alpha}\xconv{1}{f_1}{}{}{r}\xconv{3}{f_1}{p}{}{br}\xtolabel{\beta_2}\\[10pt]
\hline\\[-5pt]
\xfromlabel{\alpha}\xconv{1}{f_2}{}{}{r}\xconv{5}{f_2}{p}{}{br}\xtolabel{\beta_3} \quad&\quad
\xfromlabel{\alpha}\xpool{3}{}{m}{}{}\xconv{1}{f_3}{}{}{br}\xtolabel{\beta_4}\\[10pt]
\hline\\[-5pt]
\quad&\quad\xmerge{\beta_1,\beta_2,\beta_3,\beta_4}{a}\xtolabel{\omega}
\end{tabular}
}}
\item Definition of {\it inception unit with shortcut}:

\noindent
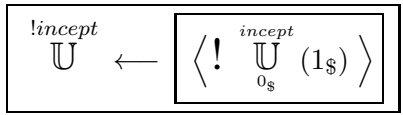
\framebox{
\xunitdef{!incept}{\xxresid{
\xunitinstance{incept}{0_{\$}}{1_{\$}}
}}}
\item Instances of inception units with shortcuts:

\noindent
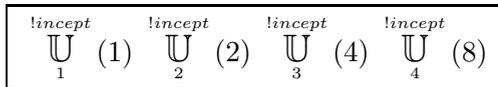
\framebox{
\xunitinstance{!incept}{1}{1}
\xunitinstance{!incept}{2}{2}
\xunitinstance{!incept}{3}{4}
\xunitinstance{!incept}{4}{8}
}
\item Definition of encoding module for invisible steganography:

\noindent
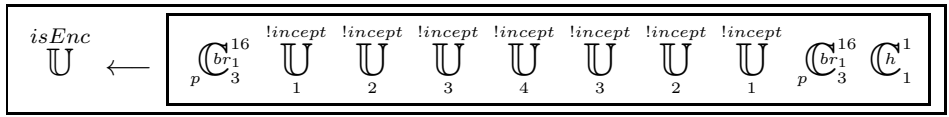
\framebox{
\xunitdef{isEnc}{
\xconv{3}{16}{p}{}{br_1}
\xunit{!incept}{1}{}
\xunit{!incept}{2}{}
\xunit{!incept}{3}{}
\xunit{!incept}{4}{}
\xunit{!incept}{3}{}
\xunit{!incept}{2}{}
\xunit{!incept}{1}{}\\[7pt]
\xconv{3}{16}{p}{}{br_1}
\xconv{1}{1}{}{}{h}
}}

\item Invisible steganography encoder: main module

\noindent
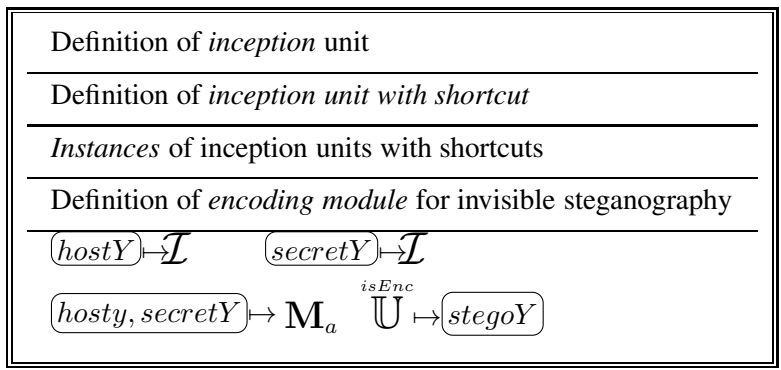
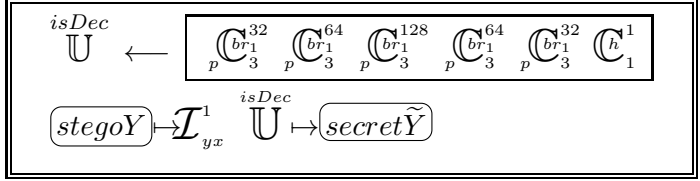
\doublebox{
\begin{tabular}{l}
Definition of {\it inception} unit
\\[3pt]\hline
Definition of {\it inception unit with shortcut}
\\[3pt]\hline
{\it Instances} of inception units with shortcuts
\\[3pt]\hline
Definition of {\it encoding module} for invisible steganography
\\[3pt]\hline
\xin{}{}{hostY}\quad\quad \xin{}{}{secretY}\\
\xmerge{hosty,secretY}{a}
\xunit{isEnc}{}{}
\xtolabel{stegoY}\\[5pt]
\end{tabular}
}
\end{enumerate}

\paragraph{Decoder while testing\\}

\doublebox{
\begin{tabular}{l}
\xunitdef{isDec}{
\xconv{3}{32}{p}{}{br_1}
\xconv{3}{64}{p}{}{br_1}
\xconv{3}{128}{p}{}{br_1}
\xconv{3}{64}{p}{}{br_1}
\xconv{3}{32}{p}{}{br_1}
\xconv{1}{1}{}{}{h}
}\\[5pt]
\xin{yx}{1}{stegoY}
\xunit{isDec}{}{}
\xtolabel{secret\widetilde{Y}}\\[5pt]
\end{tabular}
}

\paragraph{Baseline steganographic system while training\\}

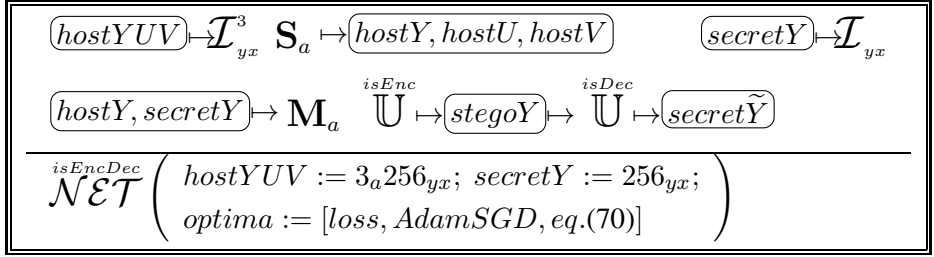
\doublebox{
\begin{tabular}{l}
\xin{yx}{3}{hostYUV}
\xsplit{a}{hostY,hostU,hostV}
\quad\quad
\xin{yx}{}{secretY}
\\[5pt]
\xmerge{hostY,secretY}{a}
\xunit{isEnc}{}{}
\xtolabelto{stegoY}
\xunit{isDec}{}{}
\xtolabel{secret\widetilde{Y}}\\[5pt]
\hline
\xbound{isEncDec}{}{
\begin{array}{l}
hostYUV := 3_a256_{yx};\  secretY := 256_{yx};\\
optima := [loss, AdamSGD, eq. \eqref{eq:stego-ssim-gan-loss}]
\end{array}
}
\end{tabular}
}\\[5pt]

The loss function for training {\it isEncDec} is a weight sum of several terms with one related to GAN, one to MSE and few others to image structural similarity defined by expressions wrt second order statistics. The statistics are computed locally as convolutions with the fixed (non learned) Gaussian coefficients.

\noindent Let
$hostY\mapsto\ Y,\  hostU\mapsto\ U,\ hostV\ \mapsto\ V,\ stegoY\mapsto\ \widetilde{Y},\ secretY\ \mapsto\ S,\ and\  
secret\widetilde{Y}\ \mapsto\ \widetilde{S}\ .
$ 
Then
\begin{equation}\label{eq:stego-ssim-gan-loss}
\begin{array}{rcl}
loss(Y,U,V,\widetilde{Y},S,\widetilde{S}) &\eqd&
 \gamma_{gan}\cdot loss_{gan}(\widetilde{Y},U,V)+\\[5pt]
 &&\gamma_{mse}\cdot\left[loss_{mse}(Y,\widetilde{Y})+loss_{mse}(S,\widetilde{S})\right]+\\[5pt]
 &&\gamma_{ssim}\cdot\left[2-loss_{ssim}(Y,\widetilde{Y})-loss_{ssim}(S,\widetilde{S})\right]+\\[5pt]
 &&\gamma_{mssim}\cdot\left[2-loss_{mssim}(Y,\widetilde{Y})-loss_{mssim}(S,\widetilde{S})\right]
\end{array}
\end{equation}

According to the GAN theory the generator minimizes recognition probability for stego image to be  a genuine image with hidden secret image:
\begin{equation}
\begin{array}{l}
s=ganEnc(\widetilde{Y}UV)) \lra loss_{gan}(\widetilde{Y},U,V) =\ds \log\left(\frac{e^{s_1}}{e^{s_0}+e^{s_1}}\right)\\[20pt]
\text{
\xunitdef{gloss}
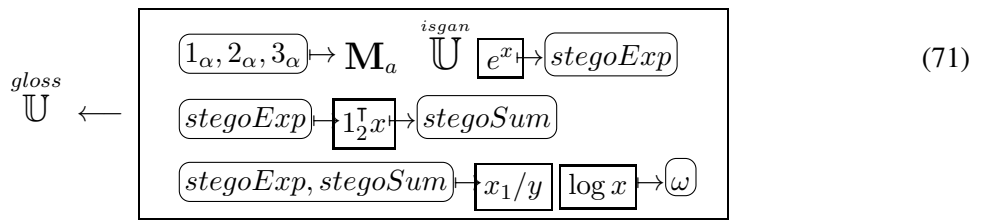
{
\begin{tabular}{l}
\xmerge{1_{\alpha},2_{\alpha},3_{\alpha}}{a}
\xunit{isgan}{}{}\xmath{e^x}\xtolabel{stegoExp}\\[5pt]
\xfromlabel{stegoExp}\xmath{\tp{1}_2x}\xtolabel{stegoSum}\\[5pt]
\xfromlabel{stegoExp,stegoSum}\xmath{x_1/y}\xmath{\log{x}}
\xtolabel{\omega}
\end{tabular}
}}
\end{array}
\end{equation}

The mean squared error $MSE(A,\widetilde{A})$ for any image $A$ with $N$ pixels and its modified version $\widetilde{A}$ with the same number of pixels is defined as the normalized by $N$, the squared Euclidean norm of $A-\widetilde{A}.$ Then the gradient flow is proportional to $A-\widetilde{A}$:
\begin{equation}
MSE(A,\widetilde{A}) \eqd \frac{1}{N}\|A-\widetilde{A}\|_2^2 \lra
\od{\cl{E}}{\widetilde{A}} = \frac{2}{N}(A-\widetilde{A})\odot\od{\cl{E}}{MSE}
\end{equation}

The structural similarity measures (of signal fidelity  are given by rational formulas wrt to the means $\mu$, standard deviations $\sigma$, variances $\sigma^2$, and covariances $\sigma_{xy}$ of signals $x,y$ for any local signal domain window (cf. \cite{Wang03a}). The average similarity over all such local windows is the scalar measure we look for.

\begin{equation}
\begin{array}{l}
\ds l(x,y)\eqd \frac{2\mu_x\mu_y+C_1}{\mu_x^2+\mu_y^2+C_1}, \quad\quad
c(X,Y)\eqd \frac{2\sigma_x\sigma_y+C_2}{\sigma_x^2+\sigma_y^2+C_2},\ \quad\quad
s(X,Y)\eqd \frac{\sigma_{xy}+C_2/2}{\sigma_x\sigma_y+C_2/2}\\[10pt]
\ds SSIM(X,Y)\eqd l(x,y)\cdot c(x,y)\cdot s(x,y) = 
\frac{2\mu_x\mu_y+C_1}{\mu_x^2+\mu_y^2+C_1}\cdot
\frac{2\sigma_{xy}+C_2}{\sigma_x^2+\sigma_y^2+C_2}
\end{array}
\end{equation}
where for the signal with $b$ bits per sample the regularizing constants $C_i=K_i\cdot 2^{b}$, $i=1,2$ (authors of \cite{Wang03a} verified $K_1=0.01, K_2=3K_1$). The measures compare the following image features: $l(x,y)$ -- signal luminance, $c(x,y)$ -- signal contrast, $s(x,y)$ -- signal structure, $SSIM(x,y)$ -- signal similarity. 

The measure MSSIM(x,y) evaluates in the  signal pyramid $(x_i,y_i),$ $i=1,\dots,m$, the cumulative contrast $c(x_1,y_1)\cdots c(x_m,y_m)$ multiplied by the luminance $l_m$ of the lowest resolution signal.
\begin{equation}
MSSIM_m(x,y)\eqd l(x_m,y_m)\cdot\prod_{i=1}^mc(x_i,y_i)
\end{equation}
The signal pyramid is the result of $50\%$ decimation preceded by Gaussian kernel applied $m-1$ times. In symbolic notation it is implemented by the layer $\xinterp{50}{g}$ with the fixed (non learned) kernel parameters.

\paragraph{Second order statistical units\\}

Using the convolution layer \xconv{11}{}{G}{}{} with fixed Gaussian kernel, the computations of statistics can be compactly represented in the symbolic notation:
\begin{enumerate}
\item Signal mean \xunit{E[\cdot]}{}{}:\\
\xunitdef{E[\cdot]}{
\xconv{11}{}{g}{}{}
}\\[3pt]
where the option $g$ denotes the fixed Gaussian kernel matching to kernel size, e.g.: $5\sigma+1=11.$
\item Signal variance unit {\it var} (with the mean value $\mu$ given as the second input, the second output is the mean value squared, i.e. $\mu^2$): 

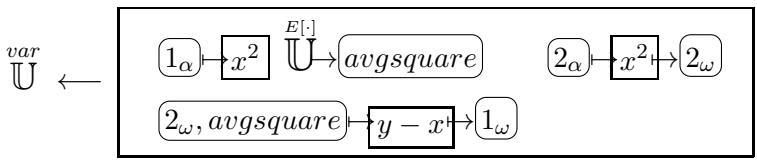
\xunitdef{var}{
\begin{tabular}{l}
\xfromlabel{1_{\alpha}}\xmath{x^2}
\xunit{E[\cdot]}{}{}\xtolabel{avgsquare}
\quad\quad
\xfromlabel{2_{\alpha}}\xmath{x^2}\xtolabel{2_{\omega}}
\\[5pt]
\xfromlabel{2_{\omega},avgsquare}\xmath{y-x}\xtolabel{1_{\omega}}
\end{tabular}
}
\item Signal mean and variance unit {\it mvar}:

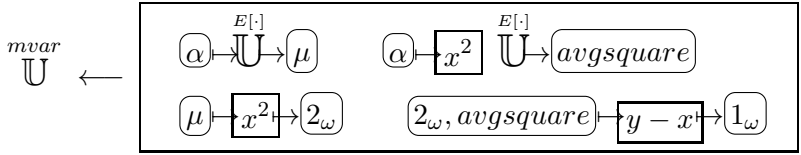
\xunitdef{mvar}{
\begin{tabular}{l}
\xfromlabel{\alpha}\xunit{E[\cdot]}{}{}\xtolabel{\mu}
\quad\quad
\xfromlabel{\alpha}\xmath{x^2}
\xunit{E[\cdot]}{}{}\xtolabel{avgsquare}
\\[5pt]
\xfromlabel{\mu}\xmath{x^2}\xtolabel{2_{\omega}}
\quad\quad
\xfromlabel{2_{\omega},avgsquare}\xmath{y-x}\xtolabel{1_{\omega}}
\end{tabular}
}

\item Signal covariance unit {\it cov} (with the signal means $\mu_x, \mu_y$ given as the third and the fourth input, the second output is the product of mean values): 

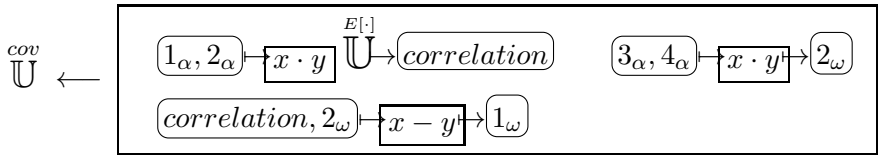
\xunitdef{cov}{
\begin{tabular}{l}
\xfromlabel{1_{\alpha},2_{\alpha}}\xmath{x\cdot y}\xunit{E[\cdot]}{}{}\xtolabel{correlation}\quad\quad
\xfromlabel{3_{\alpha},4_{\alpha}}\xmath{x\cdot y}\xtolabel{2_{\omega}}\\[5pt]
\xfromlabel{correlation,2_{\omega}}\xmath{x-y}\xtolabel{1_{\omega}}
\end{tabular}
}
\end{enumerate}

\paragraph{Signal similarity units\\}

\begin{enumerate}
\item Signal similarity -- the luminance factor {\it lumf} (inputs are $\mu_x^2,$ $\mu_y^2$, and $\mu_x\mu_y$):

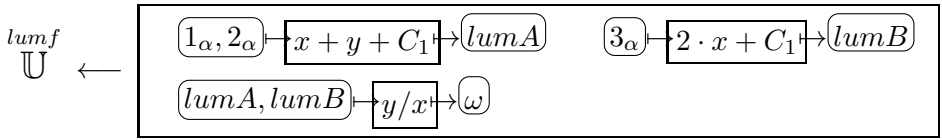
\xunitdef{lumf}{
\begin{tabular}{l}
\xfromlabel{1_{\alpha},2_{\alpha}}\xmath{x+y+C_1}\xtolabel{lumA}\quad\quad\xfromlabel{3_{\alpha}}\xmath{2\cdot x+C_1}\xtolabel{lumB}\\[5pt]
\xfromlabel{lumA,lumB}\xmath{y/x}\xtolabel{\omega}
\end{tabular}
}

\item Signal similarity -- the contrast factor {\it conf} (inputs are the signal variances, beside the contrast value the second output is the product of standard deviations):

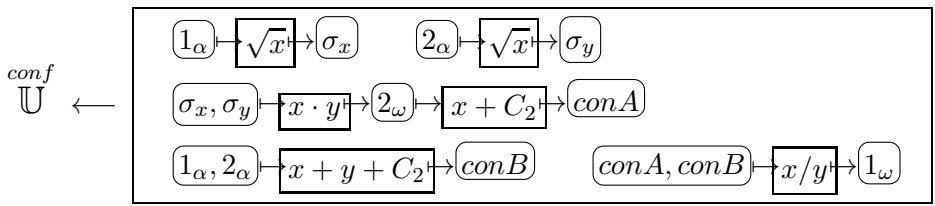
\xunitdef{conf}{
\begin{tabular}{l}
\xfromlabel{1_{\alpha}}\xmath{\sqrt{x}}\xtolabel{\sigma_x}\quad\quad \xfromlabel{2_{\alpha}}\xmath{\sqrt{x}}\xtolabel{\sigma_y}
\\[5pt]
\xfromlabel{\sigma_x,\sigma_y}\xmath{x\cdot y}\xtolabelto{2_{\omega}}\xmath{x+C_2}\xtolabel{conA}
\\[5pt]
\xfromlabel{1_{\alpha},2_{\alpha}}\xmath{x+y+C_2}\xtolabel{conB}\quad\quad
\xfromlabel{conA,conB}\xmath{x/y}\xtolabel{1_{\omega}}
\end{tabular}
}

\item Signal similarity -- the structure factor {\it strf} (inputs are: the covariance followed by the product of standard deviations):

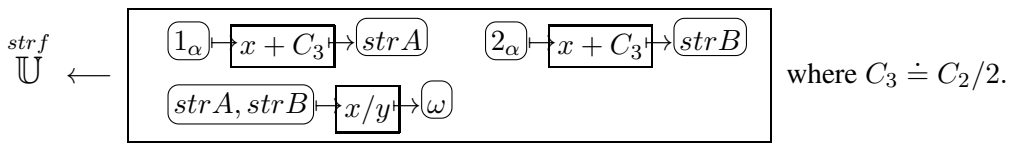
\xunitdef{strf}{
\begin{tabular}{l}
\xfromlabel{1_{\alpha}}\xmath{x+C_3}\xtolabel{strA}\quad\quad
\xfromlabel{2_{\alpha}}\xmath{x+C_3}\xtolabel{strB}\\[5pt]
\xfromlabel{strA,strB}\xmath{x/y}\xtolabel{\omega}
\end{tabular}
}
where $C_3\eqd C_2/2.$

\item Signal similarity -- the structural similarity {\it ssim} (outputs 2,3,4 are similarity factors: luminance, contrast,structure):

\hspace*{-10mm}
\xunitdef{ssim}
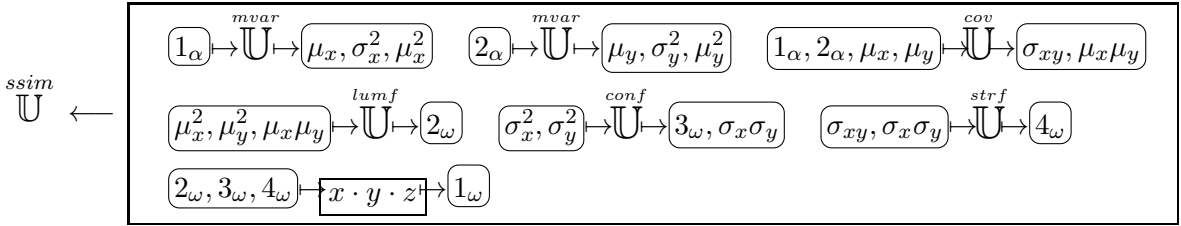
{
\begin{tabular}{l}
\xfromlabel{1_{\alpha}}\xunit{mvar}{}{}\xtolabel{\mu_x,\sigma^2_x,\mu_x^2}
\quad
\xfromlabel{2_{\alpha}}\xunit{mvar}{}{}\xtolabel{\mu_y,\sigma^2_y,\mu_y^2}
\quad
\xfromlabel{1_{\alpha},2_{\alpha},\mu_x,\mu_y}\xunit{cov}{}{}\xtolabel{\sigma_{xy},\mu_x\mu_y}
\\[5pt]
\xfromlabel{\mu_x^2,\mu_y^2,\mu_x\mu_y}\xunit{lumf}{}{}\xtolabel{2_{\omega}}
\quad
\xfromlabel{\sigma^2_x,\sigma^2_y}\xunit{conf}{}{}\xtolabel{3_{\omega},\sigma_x\sigma_y}
\quad
\xfromlabel{\sigma_{xy},\sigma_x\sigma_y}\xunit{strf}{}{}\xtolabel{4_{\omega}}
\\[5pt]
\xfromlabel{2_{\omega},3_{\omega},4_{\omega}}\xmath{x\cdot y\cdot z}\xtolabel{1_{\omega}}
\end{tabular}
}
\item Signal similarity -- contrast factor for the next level of pyramid {\it lconf}:

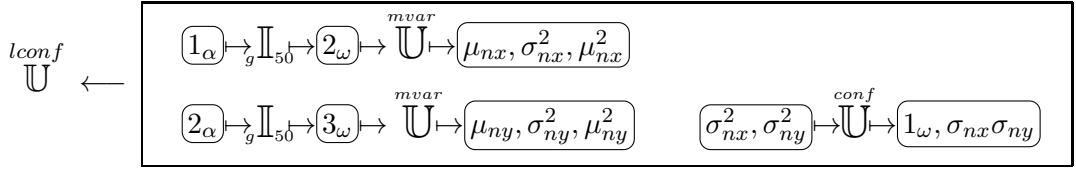
\xunitdef{lconf}{
\begin{tabular}{l}
\xfromlabel{1_{\alpha}}\xinterp{50}{g}\xtolabelto{2_{\omega}}\xunit{mvar}{}{}\xtolabel{\mu_{nx},\sigma^2_{nx},\mu_{nx}^2}
\\[5pt]
\xfromlabel{2_{\alpha}}\xinterp{50}{g}\xtolabelto{3_{\omega}}
\xunit{mvar}{}{}\xtolabel{\mu_{ny},\sigma^2_{ny},\mu_{ny}^2}
\quad\quad
\xfromlabel{\sigma^2_{nx},\sigma^2_{ny}}\xunit{conf}{}{}\xtolabel{1_{\omega},\sigma_{nx}\sigma_{ny}}
\end{tabular}
}

\item Signal similarity -- multi-scale (here 3-scale) similarity index {\it mssim} (third input is the contrast factor of  the first level in the image pyramid which is computed by {\it ssim}): 

\hspace*{-10mm}
\xunitdef{mssim}
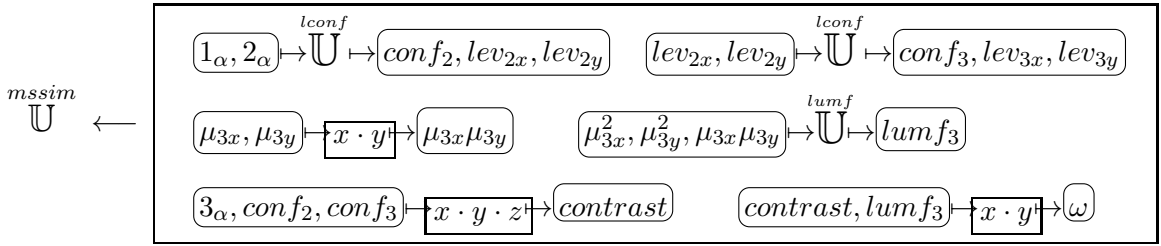
{
\begin{tabular}{l}
\xfromlabel{1_{\alpha},2_{\alpha}}\xunit{lconf}{}{}
\xtolabel{conf_2,lev_{2x},lev_{2y}}
\quad
\xfromlabel{lev_{2x},lev_{2y}}\xunit{lconf}{}{}
\xtolabel{conf_3,lev_{3x},lev_{3y}}
\\[5pt]
\xfromlabel{\mu_{3x},\mu_{3y}}\xmath{x\cdot y}\xtolabel{\mu_{3x}\mu_{3y}}
\quad\quad
\xfromlabel{\mu_{3x}^2,\mu_{3y}^2,\mu_{3x}\mu_{3y}}\xunit{lumf}{}{}\xtolabel{lumf_3}
\\[10pt]
\xfromlabel{3_{\alpha},conf_2,conf_3}\xmath{x\cdot y\cdot z}\xtolabel{contrast}
\quad\quad
\xfromlabel{contrast,lumf_3}\xmath{x\cdot y}\xtolabel{\omega}
\end{tabular}
}
\end{enumerate}

\paragraph{Loss function for 
{\it isEncDec} -- the symbolic diagram\\}

\begin{enumerate}
\item Loss term -- MSE for host and secret images:

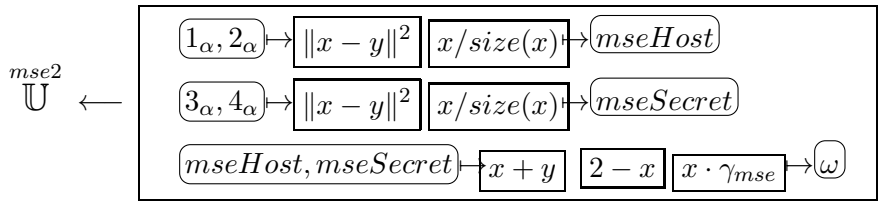
\xunitdef{mse2}{
\begin{tabular}{l}
\xfromlabel{1_{\alpha},2_{\alpha}}
\xmath{\|x-y\|^2}\xmath{x/size(x)}\xtolabel{mseHost}
\\[5pt]
\xfromlabel{3_{\alpha},4_{\alpha}}
\xmath{\|x-y\|^2}\xmath{x/size(x)}\xtolabel{mseSecret}
\\[5pt]
\xfromlabel{mseHost,mseSecret}\xmath{x+y}
\xmath{2-x}\xmath{x\cdot \gamma_{mse}}
\xtolabel{\omega}
\end{tabular}
}
\item Loss term -- SSIM for host and secret images (the second and the third output are contrast values, respectively):

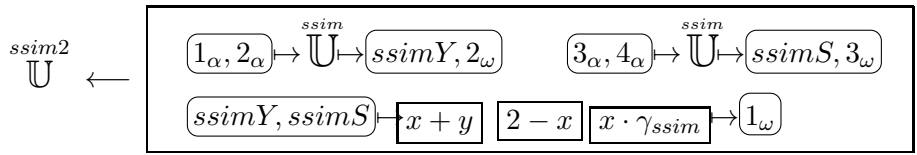
\xunitdef{ssim2}{
\begin{tabular}{l}
\xfromlabel{1_{\alpha},2_{\alpha}}
\xunit{ssim}{}{}\xtolabel{ssimY,2_{\omega}}
\quad\quad
\xfromlabel{3_{\alpha},4_{\alpha}}
\xunit{ssim}{}{}\xtolabel{ssimS,3_{\omega}}
\\[5pt]
\xfromlabel{ssimY,ssimS}\xmath{x+y}
\xmath{2-x}\xmath{x\cdot \gamma_{ssim}}
\xtolabel{1_{\omega}}
\end{tabular}
}
\item Loss term -- MSSIM for host and secret images (the last two inputs are contrast values for the first levels in the image pyramids, respectively):

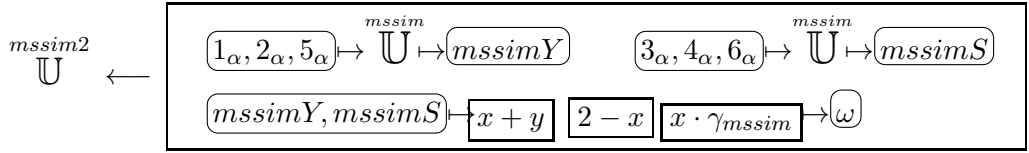
\xunitdef{mssim2}{
\begin{tabular}{l}
\xfromlabel{1_{\alpha},2_{\alpha},5_{\alpha}}
\xunit{mssim}{}{}\xtolabel{mssimY}
\quad\quad
\xfromlabel{3_{\alpha},4_{\alpha},6_{\alpha}}
\xunit{mssim}{}{}\xtolabel{mssimS}
\\[5pt]
\xfromlabel{mssimY,mssimS}\xmath{x+y}
\xmath{2-x}\xmath{x\cdot \gamma_{mssim}}
\xtolabel{\omega}
\end{tabular}
}
\item Aggregation of all loss terms:
\begin{enumerate}
\item GAN and MSE terms:

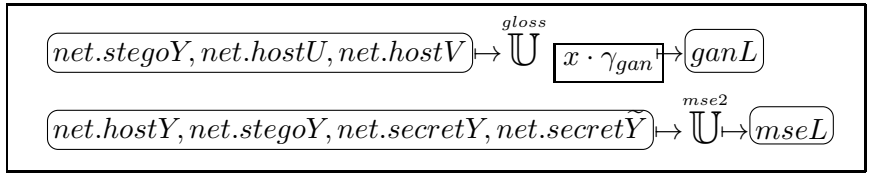
\framebox{
\begin{tabular}{l}
\xfromlabel{net.stegoY,net.hostU,net.hostV}
\xunit{gloss}{}{}\xmath{x\cdot \gamma_{gan}}\xtolabel{ganL}
\\[5pt]
\xfromlabel{net.hostY,net.stegoY,net.secretY,net.secret\widetilde{Y}}
\xunit{mse2}{}{}\xtolabel{mseL}
\end{tabular}
}

\item SSIM and MSSIM terms:

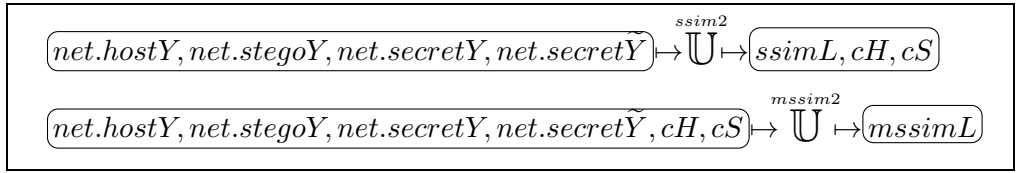
\framebox{
\begin{tabular}{l}
\xfromlabel{net.hostY,net.stegoY,net.secretY,net.secret\widetilde{Y}}\xunit{ssim2}{}{}\xtolabel{ssimL,cH,cS}
\\[5pt]
\xfromlabel{net.hostY,net.stegoY,net.secretY,net.secret\widetilde{Y},cH,cS}\xunit{mssim2}{}{}\xtolabel{mssimL}
\end{tabular}
}

\item {\it isEncDec} loss function -- structure and final summations:

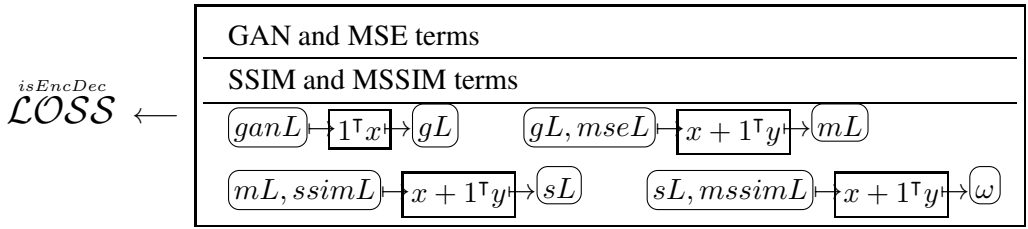
\xlossdef{isEncDec}{}{
\begin{tabular}{l}
GAN and MSE terms\\
\hline
SSIM and MSSIM terms\\
\hline
\xfromlabel{ganL}\xmath{\tp{1}x}\xtolabel{gL}
\quad\quad
\xfromlabel{gL,mseL}\xmath{x+\tp{1}y}\xtolabel{mL}
\\[5pt]
\xfromlabel{mL,ssimL}\xmath{x+\tp{1}y}\xtolabel{sL}
\quad\quad
\xfromlabel{sL,mssimL}\xmath{x+\tp{1}y}\xtolabel{\omega}
\end{tabular}
}
\end{enumerate}
\end{enumerate}

\paragraph{GAN discriminator {\it isGAN}\\}

\begin{enumerate}
\item Unit for spatial pyramid pooling:

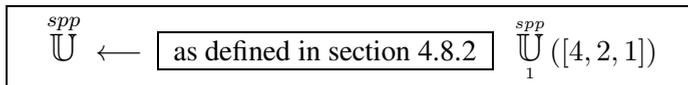
\framebox{
\begin{tabular}{l}
\xunitdef{spp}{
as defined in section \ref{sec:spp}
}
\xunitinstance{spp}{1}{[4,2,1]}
\end{tabular}
}
\item Unit for GAN scoring the input image :

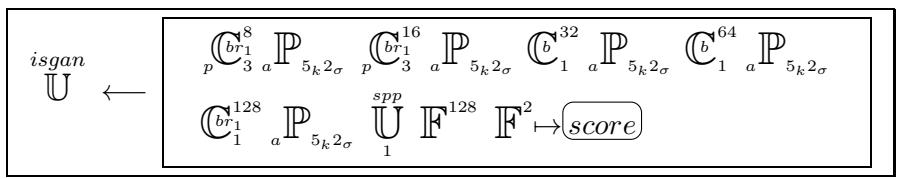
\framebox{
\xunitdef{isgan}{
\begin{tabular}{l}
\xconv{3}{8}{p}{}{br_1}\xpool{5_k2_{\sigma}}{}{a}{}{}
\xconv{3}{16}{p}{}{br_1}\xpool{5_k2_{\sigma}}{}{a}{}{}
\xconv{1}{32}{}{}{b}\xpool{5_k2_{\sigma}}{}{a}{}{}
\xconv{1}{64}{}{}{b}\xpool{5_k2_{\sigma}}{}{a}{}{}\\[5pt]
\xconv{1}{128}{}{}{br_1}\xpool{5_k2_{\sigma}}{}{a}{}{}
\xunit{spp}{1}{}
\xdense{}{128}{}{}{}
\xdense{}{2}{}{}{}
\xtolabel{score}
\end{tabular}
}
}

\item GAN discriminator for  invisible steganography {\it isGan}:

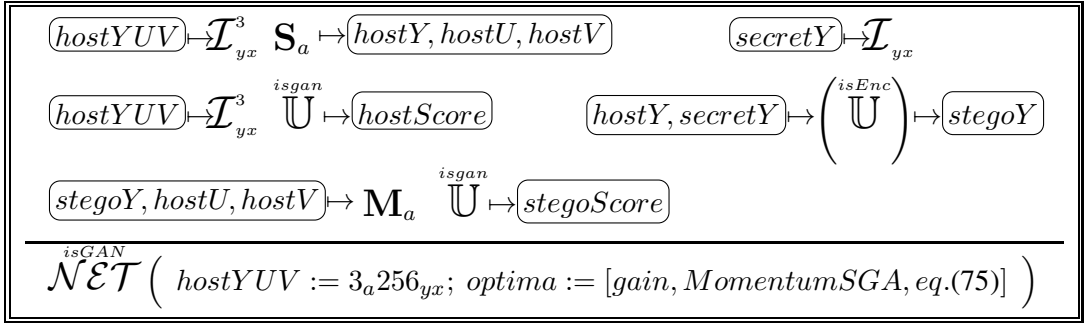
\doublebox{
\begin{tabular}{l}
\xin{yx}{3}{hostYUV}
\xsplit{a}{hostY,hostU,hostV}
\quad\quad\quad
\xin{yx}{}{secretY}
\\[5pt]
\xin{yx}{3}{hostYUV}
\xunit{isgan}{}{}
\xtolabel{hostScore}
\quad\quad\quad
\xfromlabel{hostY,secretY}
\xunitcall{isEnc}{}{}
\xtolabel{stegoY}
\\[5pt]
\xmerge{stegoY,hostU,hostV}{a}
\xunit{isgan}{}{}
\xtolabel{stegoScore}
\\[5pt]
\hline
\xbound{isGAN}{}{
\begin{array}{l}
hostYUV := 3_a256_{yx};\ 
optima := [gain, MomentumSGA, eq. \eqref{eq:stego-gan-gain}]
\end{array}
}
\end{tabular}
}
\end{enumerate}

\paragraph{Gain function  {\it isGAN} for invisible steganography\\}

Let $h\inv{2}$, $s\inv{2}$ be the host score, and stego scores produced by the analyzer. We assume that $h_0>h_1$  denotes higher probability of host image for the input. In order to apply standard GAN gain function $\log$ we convert them to probabilities by SoftMax formula: 
\begin{equation}\label{eq:stego-gan-gain}
p_0 \eqd \frac{e^{h_0}}{e^{h_0}+e^{h_1}},\ q_1\eqd \frac{e^{s_1}}{e^{s_0}+e^{s_1}} \quad\lra\quad
gain(h,s) \eqd \log p_0 + \log q_1
\end{equation}

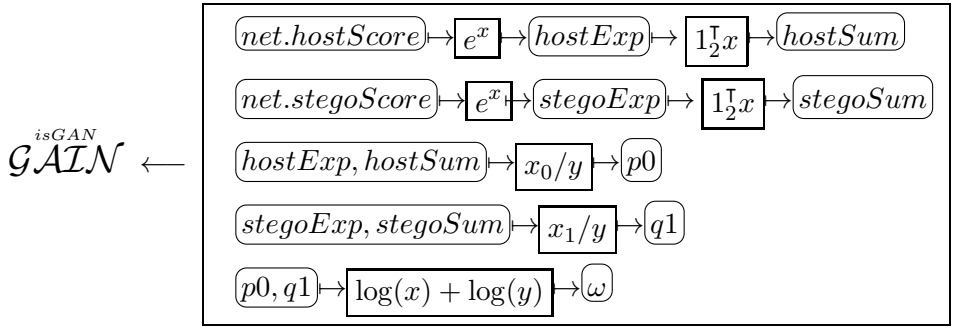
\xgaindef{isGAN}{}{
\begin{tabular}{l}
\xfromlabel{net.hostScore}
\xmath{e^x}
\xtolabelto{hostExp}
\xmath{\tp{1}_2x}
\xtolabel{hostSum}\\[5pt]
\xfromlabel{net.stegoScore}
\xmath{e^x}\xtolabelto{stegoExp}
\xmath{\tp{1}_2x}
\xtolabel{stegoSum}\\[5pt]
\xfromlabel{hostExp,hostSum}
\xmath{x_0/y}
\xtolabel{p0}\\[5pt]
\xfromlabel{stegoExp,stegoSum}
\xmath{x_1/y}
\xtolabel{q1}\\[5pt]
\xfromlabel{p0,q1}
\xmath{\log(x)+\log(y)}
\xtolabel{\omega}
\end{tabular}
}


%

\subsection{STNN for image content annotation}

Automatic image  annotation as the research task includes different approaches like:
\begin{itemize}
  \item {\it image indexing} when a descriptor and similarity/proximity function are designed to solve fast image searching problem\footnote{For instance using {\it k-NN} ($k$ nearest neighbors) methodology. In its original formulation, the indexing task does not assume any knowledge about the class the images belong to.}
  \item {\it image captioning} when a semantic description is given for the query image,
  \item {\it image labeling} when the image objects are masked-out in the pixel domain and the labels from a list of object categories are assigned to the masked areas.
\end{itemize}

Descriptor extraction for face recognition fits to the first above category. The solution proposed by Jacek Naruniec and Marek Kowalski of Warsaw University of Technology (KN-FR \cite{Kowalski-Naruniec18}) is a combination of VGG-like solution with loss function aggregating SoftMax loss with intra class variance, proposed by Wen, Zhang et al. in \cite{WenY16a}: {\it A Discriminative Feature Learning Approach for Deep Face Recognition}.

\begin{description}
\item[KN-FR as unstructured STNN:] $ $\\[5pt]
\hspace*{-10mm}
\doublebox
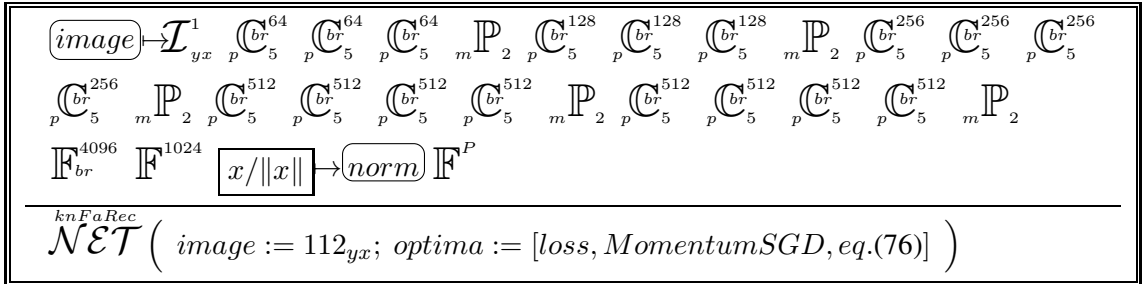
{
\begin{tabular}{l}
\xin{yx}{1}{image}
\xconv{5}{64}{p}{}{br}
\xconv{5}{64}{p}{}{br}
\xconv{5}{64}{p}{}{br}
\xpool{2}{}{m}{}{}
\xconv{5}{128}{p}{}{br}
\xconv{5}{128}{p}{}{br}
\xconv{5}{128}{p}{}{br}
\xpool{2}{}{m}{}{}
\xconv{5}{256}{p}{}{br}
\xconv{5}{256}{p}{}{br}
\xconv{5}{256}{p}{}{br}
\\[5pt]
\xconv{5}{256}{p}{}{br}
\xpool{2}{}{m}{}{}
\xconv{5}{512}{p}{}{br}
\xconv{5}{512}{p}{}{br}
\xconv{5}{512}{p}{}{br}
\xconv{5}{512}{p}{}{br}
\xpool{2}{}{m}{}{}
\xconv{5}{512}{p}{}{br}
\xconv{5}{512}{p}{}{br}
\xconv{5}{512}{p}{}{br}
\xconv{5}{512}{p}{}{br}
\xpool{2}{}{m}{}{}
\\[5pt]
\xdense{}{4096}{}{}{br}
\xdense{}{1024}{}{}{}
\xmath{x/\|x\|}
\xtolabel{norm}
\xdense{}{P}{}{}{}
\\[5pt]
\hline
\xbound{knFaRec}{}{
\begin{array}{l}
image := 112_{yx};\ optima := [loss,MomentumSGD,eq.\eqref{eq:kn-loss}]
\end{array}
}
\end{tabular}
}\\
$P$ is the number of persons taken for training.
\item[KN-FR as structured STNN:] User units definitions and their instances:\\[5pt]
\hspace*{-10mm} 
\framebox
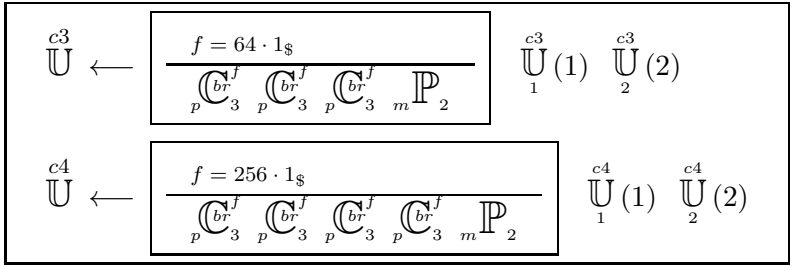
{
\begin{tabular}{l}
\xunitdef{c3}{
\begin{tabular}{l}
\xexpression{f = 64\cdot 1_{\$}}\\
\hline
\xconv{3}{f}{p}{}{br}
\xconv{3}{f}{p}{}{br}
\xconv{3}{f}{p}{}{br}
\xpool{2}{}{m}{}{}
\end{tabular}
}\ \ 
\xunitinstance{c3}{1}{1}\ \ 
\xunitinstance{c3}{2}{2}
\\[20pt]
\xunitdef{c4}{
\begin{tabular}{l}
\xexpression{f = 256\cdot 1_{\$}}\\
\hline
\xconv{3}{f}{p}{}{br}
\xconv{3}{f}{p}{}{br}
\xconv{3}{f}{p}{}{br}
\xconv{3}{f}{p}{}{br}
\xpool{2}{}{m}{}{}
\end{tabular}
}\ \ 
\xunitinstance{c4}{1}{1}\ \ 
\xunitinstance{c4}{2}{2}
\end{tabular}
}
\item[KN-FR as structured STNN:] The main architecture:\\[5pt]
\hspace*{-10mm}
\doublebox
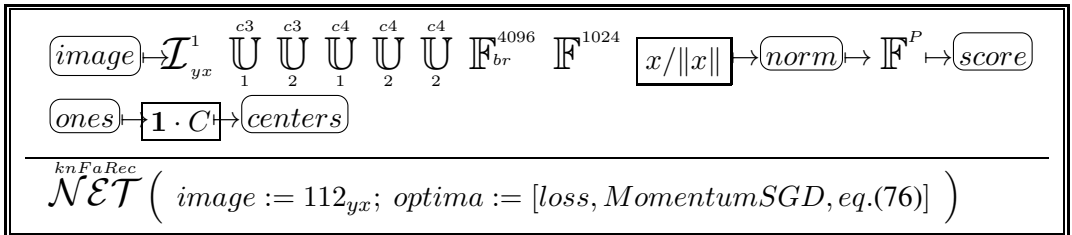
{
\begin{tabular}{l}
\xin{yx}{1}{image}
\xunit{c3}{1}{}
\xunit{c3}{2}{}
\xunit{c4}{1}{}
\xunit{c4}{2}{}
\xunit{c4}{2}{}
\xdense{}{4096}{}{}{br}
\xdense{}{1024}{}{}{}
\xmath{x/\|x\|}
\xtolabelto{norm}
\xdense{}{P}{}{}{}
\xtolabel{score}
\\[5pt]
\xfromlabel{ones}\xmath{\bm{1}\cdot C}\xtolabel{centers}
\\[5pt]
\hline
\xbound{knFaRec}{}{
\begin{array}{l}
image := 112_{yx};\ optima := [loss,MomentumSGD,eq.\eqref{eq:kn-loss}]
\end{array}
}\end{tabular}
}\\
$P$ is the number of persons taken for training.
\end{description}

Loss function is combination of the SoftMax loss and the total squared error for distances of mini-batch tensors to class centers $C[k]$, $k\in[K],$ where both are considered from the input of the last full connection layer (labeled by {\it norm}). Namely, for any $b\in[N_b]:$
\begin{equation}\label{eq:kn-loss}
\begin{array}{l}
loss_b\left(X_{norm}[b,:],X_{score}[b,:]\right) =
\ds-\ln\left(\frac
{e^{X_{score}[b,t_b]}}
{\ds\sum_{k\in[K]}e^{X_{score}[b,k]}}
\right) + \lambda\cdot\bigg\|X_{norm}[b,:]-C[t_b]\bigg\|^2\\[25pt]
\begin{array}{rcl}
loss\left(X_{norm},X_{score}\right) &=& \ds\sum_{b\in[N_b]} loss_b\left(X_{norm}[b,:],X_{score}[b,:]\right)\\[15pt]
&=&\ds\sum_{b\in[N_b]}\left[
\ds-\ln\left(\frac
{e^{X_{score}[b,t_b]}}
{\ds\sum_{k\in[K]}e^{X_{score}[b,k]}}
\right) + \lambda\cdot\bigg\|X_{norm}[b,:]-C[t_b]\bigg\|^2
\right]
\end{array}
\end{array}
\end{equation}

The above total $loss$ function value is back propagated to each of $N_b$ duals of loss functions $loss_b$. The symbolic diagrams of $loss_b$ are the same for all $b$:

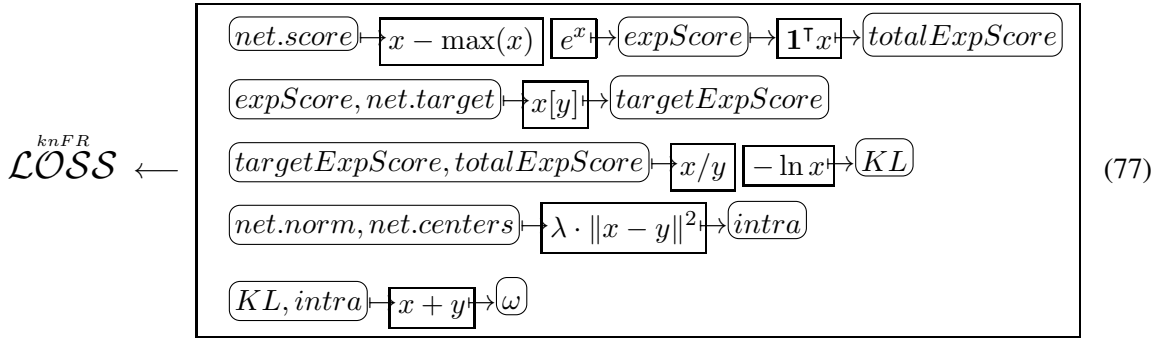
\begin{equation}\label{eq:kn-loss-stable}
\text{
\xlossdef{knFR}{}{
\begin{tabular}{l}
\xfromlabel{net.score}\xmath{x-\max(x)}\xmath{e^x}\xtolabelto{expScore}\xmath{\tp{\bm{1}}x}\xtolabel{totalExpScore}\\[5pt]
\xfromlabel{expScore,net.target}\xmath{x[y]}\xtolabel{targetExpScore}\\[5pt]
\xfromlabel{targetExpScore,totalExpScore}\xmath{x/y}\xmath{-\ln x}\xtolabel{KL}\\[5pt]
\xfromlabel{net.norm,net.centers}\xmath{\lambda\cdot\|x-y\|^2}\xtolabel{intra}\\[10pt]
\xfromlabel{KL,intra}\xmath{x+y}\xtolabel{\omega}
\end{tabular}
}}
\end{equation}

The descriptor for indexing is taken from the output labeled by $norm.$ The Euclidean distance without weights is applied as the distance function.

The authors of KN-FR informed that their estimation of class centers  is actually the exponential moving average with weights $0.3$ for the current example and $0.7$  for the previous approximation of the the center of the class, the current example belongs to. It differs from the method presented above, which in turn differs in the way of the centroid computing from the original method described in the paper \cite{WenY16a}.

\paragraph{Optimization details\\}

The authors of KN-FR net report the following parameters for optimization:
\begin{itemize}
\item the weight for intra term: $\lambda = 0.005,$
\item the number of person for the training: $P=10^5$,
  \item batch size: $N_b=64$,
  \item optimization method: {\it MomentumSGD},
  \item momentum coefficient: $\beta=0.09,$ with decay $5\cdot 10^{-4},$
  \item learning rate: $\alpha=10^{-1}$ with logarithmic decay to $10^{-3},$
  \item one epoch consists of $N_e=64000$ images selected from the training set of about $8\cdot 10^6$ images,
  \item the termination condition is: the number of epochs is equal to $200$ or the validation error stops to decrease,
  \item the model was tested on the {\it LFW} dataset and the  recognition rate exceeded $98,9\%$.
\end{itemize}

\subsection{STNN for image based human-computer interfacing via 3D models}

We show two simplified solutions for fp68 landmarks detection (FP68-PIL) and for the pose detection (POSE-PIL). The  architecture for both problems is the same in feature extraction part. They differ in the regression part and obviously in the loss functions. They were recently developed by Rafa\l\ Pilarczyk in \cite{PilarczykR18a}: {\it Tuning deep learning algorithms for face alignment and pose estimation}.

\paragraph{FP68-PIL and POSE-PIL as structured STNN -- common feature extraction unit\\}

\noindent
\framebox
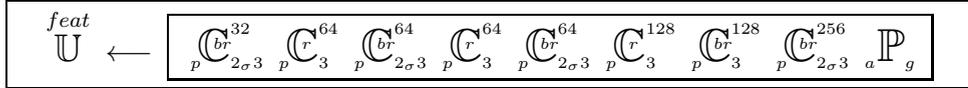
{
\begin{tabular}{l}
\xunitdef{feat}{
\xconv{2_{\sigma}3}{32}{p}{}{br}
\xconv{3}{64}{p}{}{r}
\xconv{2_{\sigma}3}{64}{p}{}{br}
\xconv{3}{64}{p}{}{r}
\xconv{2_{\sigma}3}{64}{p}{}{br}
\xconv{3}{128}{p}{}{r}
\xconv{3}{128}{p}{}{br}
\xconv{2_{\sigma}3}{256}{p}{}{br}
\xpool{g}{}{a}{}{}
}
\end{tabular}
}

\paragraph{FP68-PIL as structured STNN -- the main architecture\\[5pt]}

\noindent
\doublebox
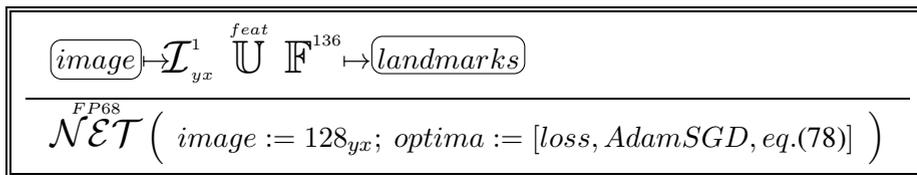
{
\begin{tabular}{l}
\xin{yx}{1}{image}
\xunit{feat}{}{}
\xdense{}{136}{}{}{}
\xtolabel{landmarks}
\\[5pt]
\hline
\xbound{FP68}{}{
\begin{array}{l}
image := 128_{yx};\ 
optima := [loss, AdamSGD, eq.\eqref{eq:fp68-loss}]
\end{array}
}
\end{tabular}
}

The loss function computes the average distance between the predicted facial landmarks and the ground truth landmarks. The measure is normalized by the interocular distance in the ground truth image (points with indexes $l=52, r=20$):
\begin{equation}\label{eq:fp68-loss}
loss(M_p,M_g) = \frac{\sum_{i\in[68]}\|M_p[i]-M_g[i]\|_2}{68\cdot\|M_g[20]-M_g[52]\|_2}
\end{equation}
where $M_p$ are the landmarks predicted and $M_g$ are ground truth landmarks.


\begin{enumerate}
\item Loss function for {\it  FP68} -- general structure:

\hspace*{-5mm}
\xlossdef{FP68}{}
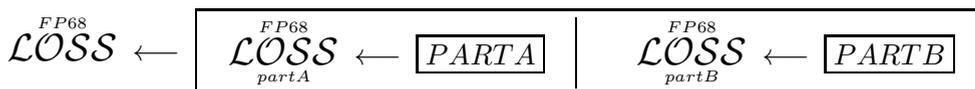
{
\begin{tabular}{l|l}
\xlossdef{FP68}{partA}{PART A} \quad &\quad
\xlossdef{FP68}{partB}{PART B}
\end{tabular}
}
\item Loss function for {\it  FP68} (part A) -- (a) convenient views of data received from the network and the "teacher"; (b) interocular distance:

\hspace*{-5mm}
\xlossdef{FP68}{partA}
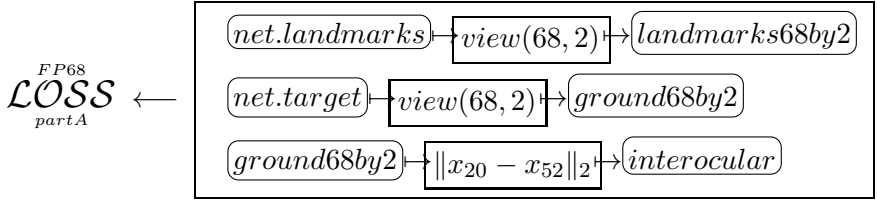
{
\begin{tabular}{l}
\xfromlabel{net.landmarks}\xmath{view(68,2)}\xtolabel{landmarks68by2}\\[5pt]
\xfromlabel{net.target}\xmath{view(68,2)}\xtolabel{ground68by2}
\\[5pt]
\xfromlabel{ground68by2}\xmath{\|x_{20}-x_{52}\|_2}\xtolabel{interocular}
\end{tabular}
}
\item Loss function for {\it  pose} (part B) -- (a) computing average distance between predicted and ground truth landmarks; (b) normalizing the error by interocular distance:

\hspace*{-5mm}
\xlossdef{FP68}{partB}
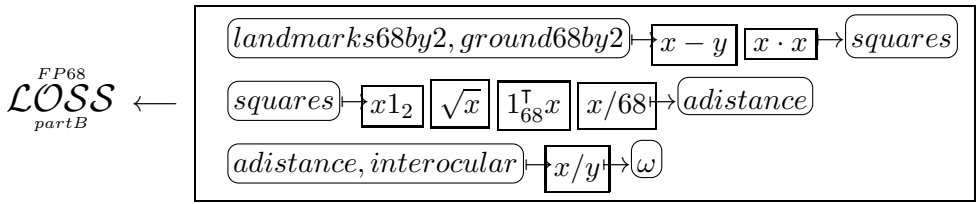
{
\begin{tabular}{l}
\xfromlabel{landmarks68by2,ground68by2}\xmath{x-y}\xmath{x\cdot x}
\xtolabel{squares}\\[5pt]
\xfromlabel{squares}\xmath{x1_2}\xmath{\sqrt{x}}\xmath{\tp{1}_{68}x}\xmath{x/68}\xtolabel{adistance}
\\[5pt]
\xfromlabel{adistance,interocular}\xmath{x/y}\xtolabel{\omega}
\end{tabular}
}

\end{enumerate}

\paragraph{POSE-PIL as structured STNN -- the main architecture\\[5pt]}

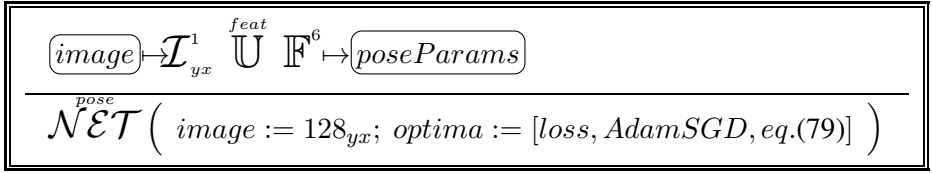
\doublebox{
\begin{tabular}{l}
\xin{yx}{1}{image}
\xunit{feat}{}{}
\xdense{}{6}{}{}{}
\xtolabel{poseParams}
\\[5pt]
\hline
\xbound{pose}{}{
\begin{array}{l}
image := 128_{yx};\ 
optima := [loss, AdamSGD, eq.\eqref{eq:pose-loss}]
\end{array}
}
\end{tabular}
}

\paragraph{Loss function via Candide-3 model}

The pose is represented by six free parameters produced by the network: $s'\in\bb{R},r'\inv{3},t\inv{2}$ where:
$s\eqd (s')^2$ -- the scaling coefficient, $r'$ -- the twist vector which is changed to the rotation matrix $R$ by the Rodriguez formula, and $t$ -- translation vector in the normalized to $[0,1]$ pixel coordinates:
\[
m = \max(n_x,n_y) \lra \left[x'\ass x/m;\ y'\ass y/m\right]
\]
where the image resolution is $n_x\times n_y.$

The loss function instead of ground truth pose,  exploits two data tensors: 
\begin{itemize}
  \item {\it net.target.candide} denoted as $X_{cand}\inm{3}{42}$: 3D points $C_i$ of standard Candide-3 model -- the subset of {\it still points}\footnote{The still landmark is the neutral point wrt to possible facial expressions - its share in creating the expression is small.} is transformed by the affine transformation (determined from the pose) and by the orthographic projection (which simply drops the $z$ coordinate) onto normalized pixels in the image plane,
  \item {\it net.target.landmarks} denoted as $X_{land}\inm{2}{42}$: 2D ground truth normalized facial landmarks which are compared with the above pixels. 
\end{itemize}

\begin{equation}\label{eq:pose-loss}
loss(X_{pose},X_{cand},X_{land}) = 
\left\|s(RX_{cand})_{xy}+t\tp{1}_{42}-X_{land}\right\|_2^2
\end{equation}


The above affine transformation and the orthographic projection joined together can be explained as follows:
\begin{equation}
\begin{array}{l}
y= s\cdot (Rx)_{xy}+t\\[5pt]
x\inv{3},\ y\inv{2},\ R\inm{3}{3},\ s\eqd (s')^2,\ s'\in\bb{R},\ t\inv{2}
\end{array}
\end{equation}
where 
\begin{itemize}
  \item $x$ -- 3D point vector in the Candide-3 coordinates, 
  \item $y$ -- 2D point vector in the normalized pixel coordinates, 
  \item $s$ -- scaling from Candide-3 to the normalized pixel coordinates, 
  \item $t$ -- translation in the normalized pixel coordinates, 
  \item $R$ -- rotation matrix in Candide-3 coordinates. 
\end{itemize}

The Rodriguez formula enables transfer from the twist vector $r'\inv{3}$ of free parameters into the rotation matrix $R\inm{3}{3}$ of nine constrained parameters.
The rotation matrix $R$ is found from its rotation axis $u\inv{3},\|u\|=1$ and the rotation angle $\theta$  which in turn are established from the twist vector $r':$
\begin{equation}\label{eq:rodrigues}
\begin{array}{l}
\theta = \|r'\|,\ u \eqd \frac{r'}{\theta}\\[5pt]
R = \cos\theta\cdot I_3+(1-\cos\theta)u\tp{u}+\sin\theta
\left[
\begin{array}{ccc}
0 & -u_z & u_y\\
u_z & 0 & -u_x\\
-u_y & u_x & 0
\end{array}
\right]
\end{array}
\end{equation}

\begin{enumerate}
\item Loss function for {\it  pose} -- general structure:

\hspace*{-5mm}
\xlossdef{pose}{}
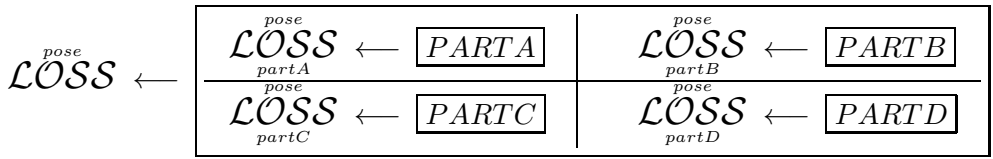
{
\begin{tabular}{l|l}
\xlossdef{pose}{partA}{PART A} \quad &\quad
\xlossdef{pose}{partB}{PART B}
\\
\hline
\xlossdef{pose}{partC}{PART C} \quad &\quad
\xlossdef{pose}{partD}{PART D}
\end{tabular}
}
\item Loss function for {\it  pose} (part A) -- convenient views of data received from {\it the network and the teacher}:

\hspace*{-5mm}
\xlossdef{pose}{partA}
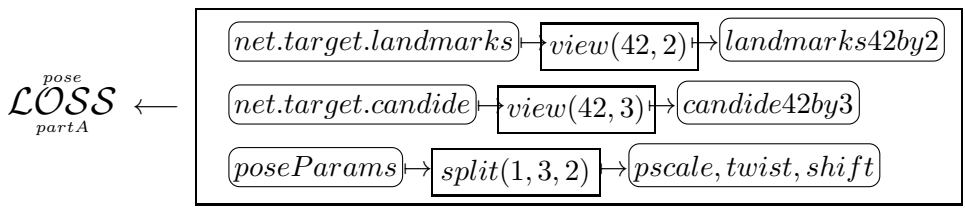
{
\begin{tabular}{l}
\xfromlabel{net.target.landmarks}\xmath{view(42,2)}\xtolabel{landmarks42by2}\\[5pt]
\xfromlabel{net.target.candide}\xmath{view(42,3)}\xtolabel{candide42by3}
\\[5pt]
\xfromlabel{poseParams}
\xmath{split(1,3,2)}
\xtolabel{pscale,twist,shift}
\end{tabular}
}
\item Loss function for {\it  pose} (part B) -- recovering the scale coefficient $s$, the angle $\theta$, and the rotation axis $u$:

\hspace*{-5mm}
\xlossdef{pose}{partB}
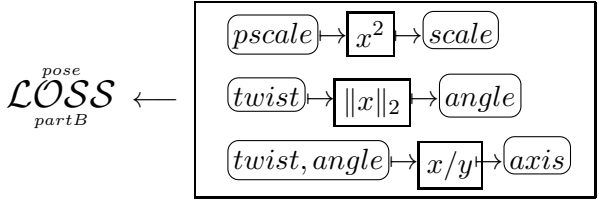
{
\begin{tabular}{l}
\xfromlabel{pscale}
\xmath{x^2}
\xtolabel{scale}\\[5pt]
\xfromlabel{twist}
\xmath{\|x\|_2}
\xtolabel{angle}\\[5pt]
\xfromlabel{twist,angle}
\xmath{x/y}\xtolabel{axis}
\end{tabular}
}
\item Loss function for {\it  pose} (part C) -- computing the rotation matrix $R$:

\hspace*{-5mm}
\xlossdef{pose}{partC}
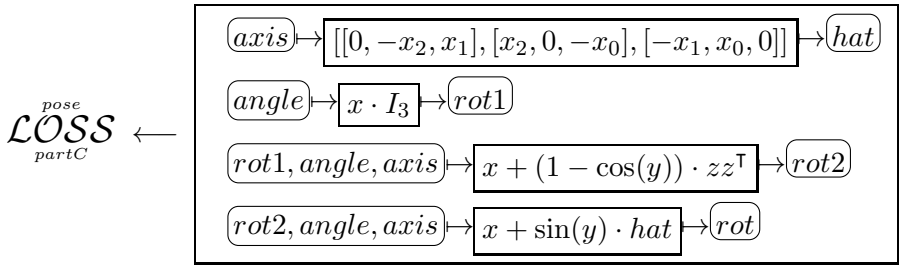
{
\begin{tabular}{l}
\xfromlabel{axis}
\xmath{[[0,-x_2,x_1],[x_2,0,-x_0],[-x_1,x_0,0]]}
\xtolabel{hat}\\[5pt]
\xfromlabel{angle}
\xmath{x\cdot I_3}
\xtolabel{rot1}\\[5pt]
\xfromlabel{rot1,angle,axis}
\xmath{x+(1-\cos(y))\cdot z\tp{z}}
\xtolabel{rot2}\\[5pt]
\xfromlabel{rot2,angle,axis}
\xmath{x+\sin(y)\cdot hat}
\xtolabel{rot}
\end{tabular}
}
\item Loss function for {\it  pose} (part D) -- (a) performing the affine transformation (rotating, scaling, translating, and orthographic projection) for the selected $42$ points of Candide-3 model; (b) evaluating the squared error of predicted landmarks wrt to ground truth landmarks:

\hspace*{-5mm}
\xlossdef{pose}{partD}
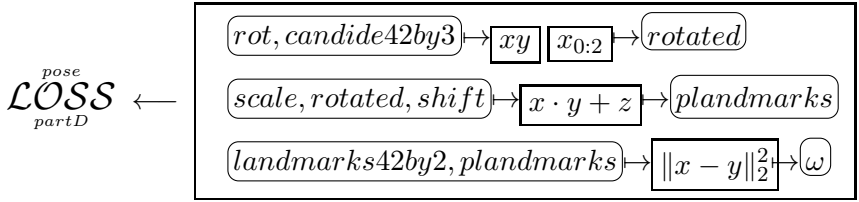
{
\begin{tabular}{l}
\xfromlabel{rot,candide42by3}
\xmath{xy}\xmath{x_{0:2}}
\xtolabel{rotated}\\[5pt]
\xfromlabel{scale,rotated,shift}
\xmath{x\cdot y+z}
\xtolabel{plandmarks}\\[5pt]
\xfromlabel{landmarks42by2,plandmarks}
\xmath{\|x-y\|_2^2}\xtolabel{\omega}
\end{tabular}
}
\end{enumerate}

\paragraph{Optimization details\\}

For optimization AdamSGD optimizer is used with the learning rate $\alpha=10^{-4}$. The mini-batch size equals to $N_b=64.$ Number of optimizer steps for each model is bounded by $3\cdot 10^5.$ However, the training is terminated if there is no progress in reducing of the validation error.

\subsection{STNN for image based data/code security}

The applications of CNN for data or code security are still sparse. However, the simple classifier of possibly infected data bit streams into several malware software like  trojans or viruses is selected from MSc thesis of David Gibert: {\it Convolutional Neural Networks for Malware Classification}, Department of Computer Science, University of Barcelona . 

In {\bf 3C-2D} scheme, the bit streams is divided into packets of $2^{10}$ bits and their 2D views of shape $32\times 32$ are delivered to the CNN classifier. The author claims the  detector accuracy for the nine  types of malware software on the level of about $98\%.$

\noindent
\doublebox
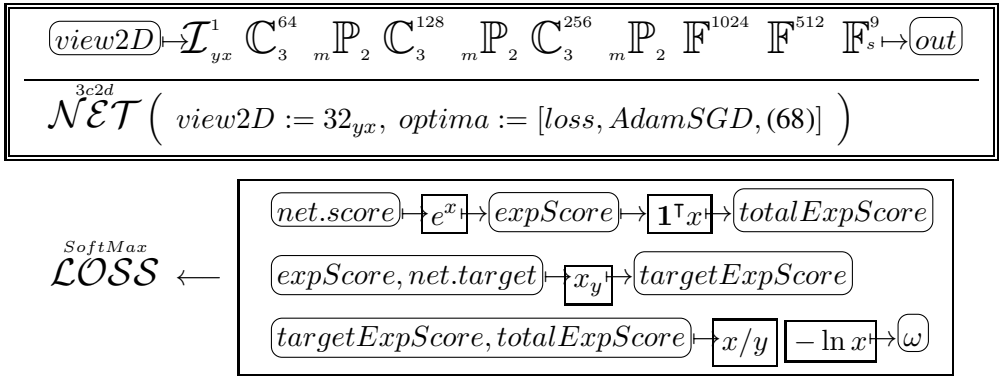
{
\begin{tabular}{l}
\xin{yx}{1}{view2D}
\xconv{3}{64}{}{}{}
\xpool{2}{}{m}{}{}
\xconv{3}{128}{}{}{}
\xpool{2}{}{m}{}{}
\xconv{3}{256}{}{}{}
\xpool{2}{}{m}{}{}
\xdense{}{1024}{}{}{}
\xdense{}{512}{}{}{}
\xdense{}{9}{}{}{s}
\xtolabel{out}
\\[5pt]
\hline
\xbound{3c2d}{}{
\begin{array}{l}
view2D := 32_{yx},\ 
optima := [loss, AdamSGD, \eqref{eq:soft-max-loss}]
\end{array}
}
\end{tabular}
}\\[5pt]

\xlossdef{SoftMax}{}{
\begin{tabular}{l}
\xfromlabel{net.score}\xmath{e^x}\xtolabelto{expScore}\xmath{\tp{\bm{1}}x}\xtolabel{totalExpScore}\\[5pt]
\xfromlabel{expScore,net.target}\xmath{x_y}\xtolabel{targetExpScore}\\[5pt]
\xfromlabel{targetExpScore,totalExpScore}\xmath{x/y}\xmath{-\ln x}\xtolabel{\omega}
\end{tabular}
}

\section{Comments, conclusions, and next steps for STNN}

\paragraph{Comments\\}

The presented tutorial by no means is the survey of neural solutions for CREAMS applications neither it is the selection of the "best" representatives examples for each of application areas. However, if we interpret the concept of "best" via a single criteria then the author's decisions were following wrt CREAMS areas: 
\begin{enumerate}
  \item the best bit performance versus subjective quality -- the areas C,E,
  \item the most compact STNN design complexity: the areas A,M,S,
  \item the audio modern application: the area R.
\end{enumerate}

\paragraph{Conclusions\\}

The presented tutorial, according to the author's intention was 
\begin{enumerate}
  \item to present a review of novel CNN solutions for applications representing all CREAMS categories,
  \item to describe the CNN/CREAMS solutions in a symbolic notation,
  \item to give the complete outline of theory and algorithms which are behind the architectures used in the above CNN/CREAMS solutions,
  \item to make the CNN description and its gradient flow equations complete for all units exploited, 
  \item to give for all presented applications, the symbolic representations of mathematically defined loss/gain functions,
  \item to answer the question how those magnificent models work and how they are designed.
\end{enumerate}

It seems that STNN  notation independent of programming paradigms and particular interfaces offered by software CNN libraries, has helped to achieve the above goals. Obviously, to a large extend -- the readers are the last oracle.

Concluding, the author beliefs that STNN is not only a convenient symbolic notation for public presentations of CNN based solutions for not only CREAMS problems but also  that it is a design blueprint with a potential for automatic generation of application source code.

\paragraph{Next steps for STNN\\}

Certainly there are drawbacks of such complex graphical presentation via \LaTeX\ environment. The basic problem which touched me directly arises at \LaTeX\ edition of STNN diagrams. In order to make them, several dozen of \LaTeX\ commands were prepared and using them in textual editor is rather a tedious work. A GUI which is well designed and next implemented to make their selection and manipulation more convenient is of merit. A WEB GUI could be really interesting.

Having GUI for \LaTeX\ commands assembling, the next steps for STNN development can be outlined:
\begin{enumerate}
  \item STNN \LaTeX\ sequence parsing for its syntactic and semantic correctness,
  \item STNN analysis for the net complexity analysis,
  \item STNN translation to JSON representation for STNN indirect implementation,
  \item and finally the direct code generators based on STNN representation.
\end{enumerate}

\nocite{*}
\bibliographystyle{fundam}
\bibliography{citations}

\begin{thebibliography}{10}
\providecommand{\url}[1]{\texttt{#1}}
\providecommand{\urlprefix}{URL }
\expandafter\ifx\csname urlstyle\endcsname\relax
  \providecommand{\doi}[1]{doi:\discretionary{}{}{}#1}\else
  \providecommand{\doi}{doi:\discretionary{}{}{}\begingroup
  \urlstyle{rm}\Url}\fi
\providecommand{\eprint}[2][]{\url{#2}}

\bibitem{Rosenblatt57}
Rosenblatt F.
\newblock The Perceptron--a perceiving and recognizing automaton. Report
  85-460-1.
\newblock Technical report, Cornell Aeronautical Laboratory, 1957.

\bibitem{WerbosP82a}
Werbos PJ.
\newblock Applications of advances in nonlinear sensitivity analysis.
\newblock In: Drenick RF, Kozin F (eds.), System Modeling and Optimization.
  Springer Berlin Heidelberg, Berlin, Heidelberg.
\newblock ISBN 978-3-540-39459-4, 1982 pp. 762--770.

\bibitem{RumelhartD88a}
Rumelhart DE, Hinton GE, Williams RJ.
\newblock Learning Representations by Back-propagating Errors.
\newblock In: Anderson JA, Rosenfeld E (eds.), Neurocomputing: Foundations of
  Research, pp. 696--699. MIT Press, Cambridge, MA, USA.
\newblock ISBN 0-262-01097-6, 1988.
\newblock \urlprefix\url{http://dl.acm.org/citation.cfm?id=65669.104451}.

\bibitem{Schmidhuber14a}
Schmidhuber J.
\newblock Deep Learning in Neural Networks: An Overview.
\newblock \emph{CoRR}, 2014.
\newblock \textbf{abs/1404.7828}.
\newblock \eprint{1404.7828}, \urlprefix\url{http://arxiv.org/abs/1404.7828}.

\bibitem{KnuthD64a}
Knuth DE.
\newblock Backus Normal Form vs. Backus Naur Form.
\newblock \emph{Commun. ACM}, 1964.
\newblock \textbf{7}(12):735--736.
\newblock \doi{10.1145/355588.365140}.
\newblock \urlprefix\url{http://doi.acm.org/10.1145/355588.365140}.

\bibitem{Paszke17a}
Paszke A, Gross S, Chintala S, Chanan G, Yang E, DeVito Z, Lin Z, Desmaison A,
  Antiga L, Lerer A.
\newblock Automatic differentiation in PyTorch.
\newblock In: 31st Conference on Neural Information Processing Systems (NIPS
  2017). 2017 .

\bibitem{Ruder17a}
Ruder S.
\newblock An overview of gradient descent optimization algorithms.
\newblock \emph{CoRR}, 2016.
\newblock \textbf{abs/1609.04747}.
\newblock \eprint{1609.04747}, \urlprefix\url{http://arxiv.org/abs/1609.04747}.

\bibitem{Kingma18a}
Kingma DP, Ba J.
\newblock Adam: {A} Method for Stochastic Optimization.
\newblock \emph{CoRR}, 2014.
\newblock \textbf{abs/1412.6980}.
\newblock \eprint{1412.6980}, \urlprefix\url{http://arxiv.org/abs/1412.6980}.

\bibitem{Nesterov83a}
NESTEROV YE.
\newblock A method for solving the convex programming problem with convergence
  rate $O(1/k^2)$.
\newblock \emph{Dokl. Akad. Nauk SSSR}, 1983.
\newblock \textbf{269}:543--547.
\newblock \urlprefix\url{https://ci.nii.ac.jp/naid/10029946121/en/}.

\bibitem{Goodfellow14a}
Goodfellow IJ, Pouget-Abadie J, Mirza M, Xu B, Warde-Farley D, Ozair S,
  Courville A, Bengio Y.
\newblock Generative Adversarial Nets.
\newblock In: Proceedings of the 27th International Conference on Neural
  Information Processing Systems - Volume 2, NIPS'14. MIT Press, Cambridge, MA,
  USA, 2014 pp. 2672--2680.
\newblock \urlprefix\url{http://dl.acm.org/citation.cfm?id=2969033.2969125}.

\bibitem{Ulyanov16a}
Ulyanov D, Vedaldi A, Lempitsky VS.
\newblock Instance Normalization: The Missing Ingredient for Fast Stylization.
\newblock \emph{CoRR}, 2016.
\newblock \textbf{abs/1607.08022}.
\newblock \eprint{1607.08022}, \urlprefix\url{http://arxiv.org/abs/1607.08022}.

\bibitem{SimonyanZ14a}
Simonyan K, Zisserman A.
\newblock Very Deep Convolutional Networks for Large-Scale Image Recognition.
\newblock \emph{CoRR}, 2014.
\newblock \textbf{abs/1409.1556}.
\newblock \eprint{1409.1556}, \urlprefix\url{http://arxiv.org/abs/1409.1556}.

\bibitem{SzegedyC14a}
Szegedy C, Liu W, Jia Y, Sermanet P, Reed SE, Anguelov D, Erhan D, Vanhoucke V,
  Rabinovich A.
\newblock Going Deeper with Convolutions.
\newblock \emph{CoRR}, 2014.
\newblock \textbf{abs/1409.4842}.
\newblock \eprint{1409.4842}, \urlprefix\url{http://arxiv.org/abs/1409.4842}.

\bibitem{HeK14a}
He K, Zhang X, Ren S, Sun J.
\newblock Spatial Pyramid Pooling in Deep Convolutional Networks for Visual
  Recognition.
\newblock \emph{CoRR}, 2014.
\newblock \textbf{abs/1406.4729}.
\newblock \eprint{1406.4729}, \urlprefix\url{http://arxiv.org/abs/1406.4729}.

\bibitem{Agustsson18a}
Agustsson E, Tschannen M, Mentzer F, Timofte R, Gool LV.
\newblock Generative Adversarial Networks for Extreme Learned Image
  Compression.
\newblock \emph{CoRR}, 2018.
\newblock \textbf{abs/1804.02958}.
\newblock \eprint{1804.02958}, \urlprefix\url{http://arxiv.org/abs/1804.02958}.

\bibitem{Mentzer18a}
Mentzer F, Agustsson E, Tschannen M, Timofte R, Van~Gool L.
\newblock Conditional Probability Models for Deep Image Compression.
\newblock In: The IEEE Conference on Computer Vision and Pattern Recognition
  (CVPR). 2018 .

\bibitem{Wang17a}
Wang T, Liu M, Zhu J, Tao A, Kautz J, Catanzaro B.
\newblock High-Resolution Image Synthesis and Semantic Manipulation with
  Conditional GANs.
\newblock \emph{CoRR}, 2017.
\newblock \textbf{abs/1711.11585}.
\newblock \eprint{1711.11585}, \urlprefix\url{http://arxiv.org/abs/1711.11585}.

\bibitem{Isola16a}
Isola P, Zhu J, Zhou T, Efros AA.
\newblock Image-to-Image Translation with Conditional Adversarial Networks.
\newblock \emph{CoRR}, 2016.
\newblock \textbf{abs/1611.07004}.
\newblock \eprint{1611.07004}, \urlprefix\url{http://arxiv.org/abs/1611.07004}.

\bibitem{Johnson16a}
Johnson J, Alahi A, Li F.
\newblock Perceptual Losses for Real-Time Style Transfer and Super-Resolution.
\newblock \emph{CoRR}, 2016.
\newblock \textbf{abs/1603.08155}.
\newblock \eprint{1603.08155}, \urlprefix\url{http://arxiv.org/abs/1603.08155}.

\bibitem{Radford15a}
Radford A, Metz L, Chintala S.
\newblock Unsupervised Representation Learning with Deep Convolutional
  Generative Adversarial Networks.
\newblock \emph{CoRR}, 2015.
\newblock \textbf{abs/1511.06434}.
\newblock \eprint{1511.06434}, \urlprefix\url{http://arxiv.org/abs/1511.06434}.

\bibitem{Zhu17a}
Zhu J, Park T, Isola P, Efros AA.
\newblock Unpaired Image-to-Image Translation using Cycle-Consistent
  Adversarial Networks.
\newblock \emph{CoRR}, 2017.
\newblock \textbf{abs/1703.10593}.
\newblock \eprint{1703.10593}, \urlprefix\url{http://arxiv.org/abs/1703.10593}.

\bibitem{ChungJ18a}
Chung JS, Nagrani A, Zisserman A.
\newblock {VoxCeleb2}: Deep Speaker Recognition.
\newblock In: INTERSPEECH. 2018 .

\bibitem{HeK15a}
He K, Zhang X, Ren S, Sun J.
\newblock Deep Residual Learning for Image Recognition.
\newblock \emph{CoRR}, 2015.
\newblock \textbf{abs/1512.03385}.
\newblock \eprint{1512.03385}, \urlprefix\url{http://arxiv.org/abs/1512.03385}.

\bibitem{DongS18a}
{Dong} S, {Zhang} R, {Liu} J.
\newblock {Invisible Steganography via Generative Adversarial Network}.
\newblock \emph{ArXiv e-prints}, 2018.
\newblock \eprint{1807.08571}.

\bibitem{Wang03a}
Wang Z, Simoncelli EP, Bovik AC.
\newblock Multiscale structural similarity for image quality assessment.
\newblock In: The Thrity-Seventh Asilomar Conference on Signals, Systems
  Computers, 2003, volume~2. 2003 pp. 1398--1402 Vol.2.
\newblock \doi{10.1109/ACSSC.2003.1292216}.

\bibitem{Kowalski-Naruniec18}
Kowalski M, Naruniec J.
\newblock (personal communication).

\bibitem{WenY16a}
Wen Y, Zhang K, Li Z, Qiao Y.
\newblock A Discriminative Feature Learning Approach for Deep Face Recognition.
\newblock In: Leibe B, Matas J, Sebe N, Welling M (eds.), Computer Vision --
  ECCV 2016. Springer International Publishing, Cham.
\newblock ISBN 978-3-319-46478-7, 2016 pp. 499--515.

\bibitem{PilarczykR18a}
Pilarczyk R, Skarbek W.
\newblock Tuning deep learning algorithms for face alignment and pose
  estimation.
\newblock In: Proc.SPIE, volume 10808. 2018 pp. 10808 -- 10808 -- 8.
\newblock \doi{10.1117/12.2501682}.
\newblock \urlprefix\url{https://doi.org/10.1117/12.2501682}.

\bibitem{Gibert16a}
Gibert D.
\newblock Convolutional Neural Networks for Malware Classification.
\newblock Master's thesis, Department of Computer Science, Escola Politecnica
  de Catalunya, UPC, Barcelon, Spain, 2016.

\bibitem{Salimans16a}
Salimans T, Goodfellow IJ, Zaremba W, Cheung V, Radford A, Chen X.
\newblock Improved Techniques for Training GANs.
\newblock \emph{CoRR}, 2016.
\newblock \textbf{abs/1606.03498}.
\newblock \eprint{1606.03498}, \urlprefix\url{http://arxiv.org/abs/1606.03498}.

\end{thebibliography}

\end{document}